\documentclass[twoside,11pt]{article}

\makeatletter
\let\Ginclude@graphics\@org@Ginclude@graphics
\makeatother

\usepackage{jmlr2e}

\usepackage{lastpage}

\usepackage{xr}
\usepackage{hyperref}
\usepackage{algorithm,algorithmic}
\usepackage{makecell}
\usepackage{amssymb, multirow, paralist, color,amsmath}
\usepackage{enumitem}
\usepackage{textcase}
\usepackage[T1]{fontenc}
\usepackage{mathtools}
\usepackage{cleveref}
  
\usepackage{empheq}
\usepackage{color}
\definecolor{darkgreen}{rgb}{0.00,0.5,0.00}
\usepackage{fancybox}
\usepackage{tablefootnote}
\usepackage{natbib}
\newtheorem{asm}{Assumption}

\usepackage{subfigure}

\definecolor{myteal}{RGB}{27,158,119}
\definecolor{myorange}{RGB}{217,95,2}
\definecolor{myred}{RGB}{231,41,138}
\definecolor{mypurple}{RGB}{152,78,163}
\definecolor{myblue}{rgb}{.9, .9, 1}
\definecolor{mygreen}{RGB}{0,100,0}
\definecolor{mycyan}{rgb}{0.88,1,1}
\definecolor{mydarkred}{RGB}{192,47,25}

\def\inner#1#2{\langle #1, #2 \rangle}

\def \S {\mathbf{S}}
\def \A {\mathcal{A}}

\def \R {\mathbb{R}}

\def \w {\mathbf{w}}
\def \v {\mathbf{v}}

\def \E {\mathbb{E}}

\def \1 {\mathbf{1}}
\def \z {\mathbf{z}}

\def \u {\mathbf{u}}

\def \L {\mathcal{L}}

\newcommand{\Norm}[1]{\left\|#1\right\|}

\usepackage[skins]{tcolorbox}

\def \F {\mathcal{F}}
\def \I {\mathbf{I}}

\def \Deltah {\widehat\Delta}
\def \nablah {\widehat\nabla}

\usepackage{cancel}
\usepackage{colortbl}

\def \B {\mathcalB}
\def \C {\mathbf C}
\def \O {\mathcal O}

\usepackage{bbm}
\usepackage{booktabs}
\usepackage{tablefootnote}

\newtheorem{cor}{Corollary}

\newlength\mytemplen
\newsavebox\mytempbox

\makeatletter
\newcommand\mybluebox{%
	\@ifnextchar[
	{\@mybluebox}%
	{\@mybluebox[0pt]}}

\def\@mybluebox[#1]{%
	\@ifnextchar[
	{\@@mybluebox[#1]}%
	{\@@mybluebox[#1][0pt]}}

\def\@@mybluebox[#1][#2]#3{
	\sbox\mytempbox{#3}%
	\mytemplen\ht\mytempbox
	\advance\mytemplen #1\relax
	\ht\mytempbox\mytemplen
	\mytemplen\dp\mytempbox
	\advance\mytemplen #2\relax
	\dp\mytempbox\mytemplen
	\colorbox{myblue}{\hspace{1em}\usebox{\mytempbox}\hspace{1em}}}

\def \E {\mathbb{E}}
\def \bT {\mathbb{T}}

\def \D {\mathbf{D}}
\def \z {\mathbf{z}}
\def \u {\mathbf{u}}

\def \w {\mathbf{w}}

\def \R {\mathbb{R}}
\def \S {\mathcal{S}}

\def \A {\mathcal{A}}

\def \v {\mathbf{v}}

\def \vhat {\widehat{\bf v}}

\def \bg {\mathbf{g}}

\usepackage{amssymb}
\usepackage{pifont}
\newcommand{\cmark}{\text{\ding{51}}}
\newcommand{\xmark}{\text{\ding{55}}}

\usepackage{bm}

\def \B {\mathcal{B}}

\def \C {\mathcal{C}}

\def \T {\mathcal{T}}
\def \F {\mathcal{F}}

\def \momlvo {$\text{MOML}^{\text{v1}}$}
\def \momlvs {$\text{MOML}^{\text{v2}}$}
\usepackage{fontawesome}
\usepackage{threeparttable,adjustbox}

\usepackage[colorinlistoftodos,bordercolor=green,backgroundcolor=white,linecolor=green,textsize=scriptsize]{todonotes}

\newcommand{\CC}[1]{\cellcolor{gray!30}}

\jmlrheading{24}{2023}{1-46}{11/21; Revised
12/22}{4/23}{21-1301}{Bokun Wang, Zhuoning Yuan, Yiming Ying, Tianbao Yang}

\ShortHeadings{Memory-Based Optimization Methods for MAML and PFL}{Wang, Yuan, Ying, Yang}
\firstpageno{1}

\begin{document}
	
	\title{Memory-Based Optimization Methods for Model-Agnostic Meta-Learning and Personalized Federated Learning}
	
	\author{\name Bokun Wang \email bokun-wang@tamu.edu\\
		\addr Department of Computer Science and Engineering\\
		Texas A\&M University\\
		College Station, TX 77843, USA
		\AND
		\name Zhuoning Yuan \email zhuoning-yuan@uiowa.edu \\
		\addr Department of Computer Science\\
		The University of Iowa\\
		Iowa City, IA 52242, USA
	\AND 
\name Yiming Ying \email yying@albany.edu\\
\addr Department of Mathematics and Statistics\\
University at Albany\\
Albany, NY 12222, USA
\AND 
\name Tianbao Yang \email tianbao-yang@tamu.edu\\
\addr Department of Computer Science and Engineering\\
Texas A\&M University\\
College Station, TX 77843, USA}
	
\editor{Zaid Harchaoui}
	
	\maketitle
	
	\begin{abstract}
In recent years, model-agnostic meta-learning (MAML) has become a popular research area. However, the stochastic optimization of MAML is still underdeveloped. Existing MAML algorithms rely on the ``episode'' idea by sampling a few tasks and data points to update the meta-model at each iteration. Nonetheless, these algorithms either fail to guarantee convergence with a constant mini-batch size or require processing a large number of tasks at every iteration, which is unsuitable for continual learning or cross-device federated learning where only a small number of tasks are available per iteration or per round. To address these issues, this paper proposes memory-based stochastic algorithms for MAML that converge with vanishing error. The proposed algorithms require sampling a constant number of tasks and data samples per iteration, making them suitable for the continual learning scenario. Moreover, we introduce a communication-efficient memory-based MAML algorithm for personalized federated learning in cross-device (with client sampling) and cross-silo (without client sampling) settings. Our theoretical analysis improves the optimization theory for MAML, and our empirical results corroborate our theoretical findings. Interested readers can access our code at \url{https://github.com/bokun-wang/moml}.
	\end{abstract}
	
	
	\section{Introduction}
	
	Despite the remarkable success of modern deep learning approaches, they are often criticized for their heavy reliance on large amounts of data~\citep{journals/corr/abs-1801-00631}. In contrast, humans can learn with relatively small amounts of data thanks to their ability to continuously learn from multiple tasks. Recently, meta-learning has garnered significant attention for its ability to perform well on new tasks using the adaptation and prior knowledge gained from previous tasks~\citep{schmidhuber1987evolutionary, thrun2012learning, hospedales2020meta}. Among meta-learning approaches, the model-agnostic meta-learning (MAML) technique based on gradient-based optimization~\citep{finn2017model} has proven to be successful across a broad range of problems that can be trained using gradient descent. Specifically, MAML proposes to solve the following optimization problem.
 \begin{align}\label{eq:maml}
		\min_{\w\in\R^d}F(\w) = \frac{1}{n} \sum_{i=1}^n \L_i(\w - \alpha\nabla \L_i(\w)),
	\end{align}
	where we use $\w\in\R^d$ to represent the meta-model, and $n$ to denote the number of tasks. The risk function for the $i$-th task is denoted by $\L_i$, and can be expressed as $\L_i(\w) = \E_{\z\sim \mathbf D_i}[\ell_i(\w, \z)]$. Here, $\mathbf D_i$ represents the data distribution for the $i$-th task, while $\ell_i(\cdot)$ denotes the loss function. The inner gradient step, $\w - \alpha\nabla \L_i(\w)$, represents an adaptation from the meta-model $\w$ to the $i$-th task.
	
	MAML has received considerable attention from researchers, with several studies investigating its applications and extensions~\citep{nichol2018first, antoniou2018train, behl2019alpha, yoon2018bayesian, raghu2019rapid, li2017meta, grant2018recasting}. However, the stochastic optimization algorithms used to solve the MAML problem (Eq. \eqref{eq:maml}) are still far from satisfactory. Two key quantities are present in each ``episode'' of the optimization algorithms: the number of sampled tasks denoted by $B$, and the number of sampled data points per task denoted by $K$. The episode is designed to mimic the few-shot task by sub-sampling both tasks and data points~\citep{Vinyals2016MatchingNF, Snell2017PrototypicalNF, Ravi2017OptimizationAA}. Unfortunately, the original MAML method based on the episode does not necessarily converge to a stationary point of the objective function $F$ in (Eq. \eqref{eq:maml}) unless $K$ is sufficiently large. Recently, \citet{fallah2020convergence} provided the first convergence analysis of the original MAML approach, which suggests that to find an $\epsilon$-stationary point $\w$ satisfying $\E[\Norm{\nabla F(\w)}]\leq\epsilon$, one needs to run MAML for $T=\O(1/\epsilon^2)$ iterations and sample $K=\O(1/\epsilon^2)$ data points for all $n$ tasks in each iteration. However, these batch sizes are impractical for driving the error level $\epsilon$ to be sufficiently small. 
 
    In recent work, \citet{hu2020biased} introduced two biased stochastic methods, namely BSGD and BSpiderBoost, which are modifications of the original MAML algorithm~\citep{finn2017model} for solving Eq.~\eqref{eq:maml}. These methods have convergence guarantees for finding an $\epsilon$-stationary point, but require impractical settings for the number of sampled tasks and data points per iteration ($K$ and $B$) as well as the number of iterations ($T$), as $K=\O(1/\epsilon^2)$, $B=\O(1)$, and $T=\O(1/\epsilon^4)$ for BSGD, and $K=\O(1/\epsilon^2)$, $B=\sqrt{n}$, and $T=\O(1/\epsilon^2)$ for BSpiderBoost. Neither setting is practical for driving the error level $\epsilon$ to be sufficiently small, not to mention the imposed additional assumptions (see Table~\ref{tab:finite_n_comparison} for a more thorough comparison). These findings suggest that the original MAML approach may not converge to an accurate solution if a small batch size $K$ is used. Other studies have approached the optimization of~(Eq. \eqref{eq:maml}) as a two-level compositional function~\citep{chen2020solving,DBLP:journals/corr/abs-2002-03755} or bilevel optimization problem~\citep{franceschi2018bilevel, Ji2020ProvablyFA, Chen2021ASS}, but these methods either require $K$ to be very large or involve passing through all $n$ tasks at each iteration. 
	
    Federated Learning (FL) is a framework for distributed learning over a federation of mobile devices~\citep{konevcny2016federated, mcmahan2017communication, kairouz2021advances,wang2021field}. In FL, the local data of each device cannot be shared with other devices or the central server. There are two important settings of FL: \emph{cross-silo} and \emph{cross-device}. The cross-silo FL is typically located at data centers, where the number of clients is limited (say dozens) and the clients are available at each iteration. In contrast, the cross-device FL is deployed on a network of mobile devices, where the number of clients is much larger than that of cross-silo FL, but few clients are available in each iteration. Personalized FL has drawn attention due to the challenges and questions posed by data heterogeneity. Recently, the connection between meta-learning and personalized federated learning (FL) has been noticed, since both tasks in meta-learning and clients in federated learning are heterogeneous~\citep{Jiang2019ImprovingFL}. The convergence theory of the federated variant of MAML has been developed in \citet{fallah2020personalized}.
	
	This paper aims to improve MAML optimization by addressing the following question: 
	
	\shadowbox{\begin{minipage}[t]{0.95\columnwidth}%
			\it Can we design efficient stochastic optimization algorithms for MAML, which can converge to a stationary point with only $K=\O(1)$ and $B=\O(1)$ to update the model? 
	\end{minipage}}
	
	\subsection{Contributions}
	
	We present the main contributions of our work below: 
	\begin{itemize}[leftmargin=*]
		\item We address the problem of stochastic optimization for MAML in both single-node and federated learning settings. In the single-node setting, the server has access to all tasks and their data, while in the federated learning setting, the central server has no access to the individual tasks and their data on distributed clients. We propose two memory-based stochastic algorithms, MOML and LocalMOML. MOML is designed for the centralized setting, while LocalMOML can be used in both centralized and federated learning settings. The proposed algorithms maintain and update individualized models, or memories, for each task using a \textsc{Momentum} update. This involves computing a moving average of historical stochastic updates of individual models.
	\item We provide the convergence guarantees of MOML and LocalMOML for finding a stationary point of the non-convex objective with only $K=\O(1)$ data samples per task and $B=\O(1)$ tasks per iteration. To the best of our knowledge, this is the first work to achieve such results. We also provide a comparison of our theoretical results and key features of the proposed algorithms with other existing results in Table~\ref{tab:finite_n_comparison}. Importantly, our LocalMOML algorithm consistently outperforms the existing Per-FedAvg algorithm~\citep{fallah2020personalized} in terms of sample complexity.
\item Our proposed methods MOML and LocalMOML support task/client sampling, which is a desirable property under both single-node learning and federated learning settings, unlike some methods listed in Table~\ref{tab:finite_n_comparison}. Task sampling is desired when the tasks are on the same machine (single-node learning), as backpropagating through all $n$ tasks requires much more GPU memory. Moreover, in the continual learning regime, only a small proportion of tasks might be available every iteration. In the cross-device federated learning regime, the server can only access available clients via the client sampling process because the number of clients is huge and direct connections cannot be easily established. 
	\end{itemize}
	
	\begin{table}[t]
		\caption{Comparison of proposed algorithms with existing approaches when the number of tasks $n$ is finite. $\epsilon$ denotes the accuracy for an $\epsilon$-stationary point $\E\left[|\nabla F(\w)|\right] \leq \epsilon$. Ticks and crosses in \textcolor{blue}{\bf blue} are pros while those in \textcolor{purple}{\bf purple} are cons.}
		\label{tab:finite_n_comparison}
		\setlength\tabcolsep{3.5pt} 
		\centering
		\scalebox{0.8}{
			\begin{threeparttable}[b]
				\centering
				\begin{tabular}{cccccc}\toprule[.1em]
					\multicolumn{6}{c}{Single-Node Learning}\\\midrule
					\multirow{2}{*}{Algorithm}   & \multirow{2}{*}{ \makecell{Task\\Sampling}}  & \multirow{2}{*}{\makecell{Sample\\Complexity}} & \multirow{2}{*}{\makecell{\#Data points ($K$)\\Per Iteration}} & \multicolumn{2}{c}{Strict Assumptions}  \\\cmidrule(lr){5-6}
					&& & & \makecell{Bounded\\Gradient} & \makecell{Stochastic\\Lipschitz}\tnote{\color{red}(1)}  \\\cmidrule(lr){1-6}
					MAML~\citep{fallah2020convergence} & \textcolor{purple}{\xmark} &    $\O\left(n\epsilon^{-4}\right)$     &    $\O(\epsilon^{-2})$   & \textcolor{blue}{\xmark} &\textcolor{blue}{\xmark} \\
					SCGD~\citep{wang2017stochastic} &\textcolor{purple}{\xmark} & $\O(n\epsilon^{-8})$ & $\O(1)$  &\textcolor{purple}{\cmark} & \textcolor{blue}{\xmark}   \\
					NASA~\citep{Ghadimi2020AST} & \textcolor{purple}{\xmark} & $\O(n\epsilon^{-4})$ & $\O(1)$  & \textcolor{purple}{\cmark} & \textcolor{blue}{\xmark}   \\
					BSGD~\citep{hu2020biased} & \textcolor{blue}{\cmark} & $\O(\epsilon^{-6})$ & $\O(\epsilon^{-2})$ & \textcolor{purple}{\cmark}\tnote{\color{red}(2)} & \textcolor{purple}{\cmark} \\
					BSpiderBoost~\citep{hu2020biased}& \textcolor{blue}{\cmark} & $\O\left(n\epsilon^{-2} + \sqrt{n}\epsilon^{-4}\right)$ & $\O(\epsilon^{-2})$  & \textcolor{purple}{\cmark} & \textcolor{purple}{\cmark} \\
					\CC{10} \textbf{\momlvo~(This work)} &\CC{10} \textcolor{blue}{\cmark}  &\CC{10} $\O\left(n\epsilon^{-5}\right)$ &\CC{10} $\bm{ \O(1)}$  &\CC{10} \textcolor{purple}{\cmark} &\CC{10}\textcolor{blue}{\xmark}  \\
					\CC{10} \textbf{\momlvs~(This work)} &\CC{10} \textcolor{blue}{\cmark} &\CC{10} $\O\left(n \epsilon^{-5}\right)$ &\CC{10} $ \bm{\O(1)}$ &\CC{10} \textcolor{blue}{\xmark} &\CC{10} \textcolor{blue}{\xmark}  \\
					\midrule[.1em]
					\multicolumn{6}{c}{Personalized Federated Learning}\\\midrule
	Algorithm & \makecell{Client\\Sampling} & \makecell{Sample\\Complexity} & \makecell{Communication\\Complexity} & \multicolumn{2}{c}{\makecell{Avg. \#Data points\\($K$) Per Iteration}}  \\\cmidrule(lr){1-6}
				Per-FedAvg \citep{fallah2020personalized}& \xmark & $\O\left(n\epsilon^{-6}\right)$  & $\O(\epsilon^{-3}) $  &\multicolumn{2}{c}{$\O(\epsilon^{-2})$}   \\
		Per-FedAvg \textbf{(This work)\tnote{\color{red}(3)} } & \cmark & $\O(\epsilon^{-7})$ & $\O(\epsilon^{-4}) $  & \multicolumn{2}{c}{$\O(\epsilon^{-2})$}   \\
		\rowcolor{gray!30} 			\textbf{LocalMOML (This work) }&\xmark &$\bm{\O\left(n\epsilon^{-5}\right)}$ & $\O(\epsilon^{-3})$&\multicolumn{2}{c}{$ \bm{\O(1)}$}  \\
					\rowcolor{gray!30}  \textbf{LocalMOML (This work)} & \cmark &$\bm{\O\left(\epsilon^{-6}\right)}$ & $\O(\epsilon^{-4})$&\multicolumn{2}{c}{$ \bm{\O(1)}$}  \\
					\bottomrule[.1em]
				\end{tabular}
				\begin{tablenotes}
					\item [\textcolor{red}{(1)}] {\scriptsize Stochastic Lipschitz: $\nabla \ell(\cdot;\z)$ is Lipschitz continuous for each $\z$, which is stronger than the Lipschitzness of $\L_i(\cdot)$ in Assumption~\ref{asm:smoothness}.}
					\item [\textcolor{red}{(2)}] {\scriptsize BSGD can obtain the same rate without the bounded gradient assumption if the weak convexity of $F(\cdot)$ is additionally assumed.} 
					\item [\textcolor{red}{(3)}] {\scriptsize The analysis of Per-FedAvg with client sampling in \cite{fallah2020personalized} seems to be problematic. See \Cref{sec:wrong_fed_maml} for details.}
				\end{tablenotes}
		\end{threeparttable}}

	\end{table}	

\begin{table}[t]
	\caption{Comparison of proposed algorithms with existing approaches when the number of tasks $n$ is infinite. For example, the tasks are online. }
	\label{tab:comparision_online}
	\setlength\tabcolsep{3.5pt} 
	\centering
	\scalebox{0.85}{
		\begin{threeparttable}[b]
			\centering
			\begin{tabular}{cccccc}\toprule[.1em]
				\multicolumn{6}{c}{Single-Node Learning}\\\midrule
				\multirow{2}{*}{Algorithm}    & \multirow{2}{*}{\makecell{Sample\\Complexity}} & \multirow{2}{*}{\makecell{\#Data points ($K$)\\Per Iteration}}& \multirow{2}{*}{\makecell{\#Tasks ($B$)\\Per Iteration}} & \multicolumn{2}{c}{Strict Assumptions}  \\\cmidrule(lr){5-6}
				&& & & \makecell{Bounded\\Gradient} & \makecell{Stochastic\\Lipschitz} \\\cmidrule(lr){1-6}
				MAML~\citep{fallah2020convergence} &    $\O\left(\epsilon^{-6}\right)$     &    $\O(\epsilon^{-2})$   & $\O(\epsilon^{-2})$ & \xmark &\xmark \\
				BSpiderBoost~\citep{hu2020biased} & $\O\left(\epsilon^{-5}\right)$ & $\O(\epsilon^{-2})$ & $\O(\epsilon^{-2})$ or $\O(\epsilon^{-1})$ & \cmark & \cmark \\
				BSGD~\citep{hu2020biased} & $\O(\epsilon^{-6})$ & $\O(\epsilon^{-2})$ & 	$\O(1)$  & \cmark  & \cmark  \\
				\CC{10} \textbf{LocalMOML (This work)} &\CC{10} $\O(\epsilon^{-6})$ & \CC{10}$\bm{\O(1)}$ &\CC{10} $\bm{\O(1)}$  &\CC{10} \cmark &\CC{10}  \xmark  \\
				\bottomrule[.1em]
			\end{tabular}
	\end{threeparttable}}

\end{table}	

\section{Related Works}

In this section, we discuss previous works related to ours in four categories.
	
\paragraph{Meta-Agnostic Meta-Learning (MAML) } Gradient-based MAML was introduced in \citet{finn2017model} and has since become a popular algorithm for learning from prior experience, with many applications in supervised learning, reinforcement learning, and more. Later on, several works have delved deeper into MAML to better understand its practical performance and provide some tricks of the trade for further improving its practicability \citep{antoniou2018train, raghu2019rapid, behl2019alpha}.  The vanilla MAML has been generalized from various perspectives. For example, probabilistic MAML is introduced in \citet{finn2018probabilistic} to model a distribution over prior model parameters. Other algorithms with multi-step gradient descent~\citep{ji2020multistep} and with partial parameters adaptation~\citep{ji2020convergence} have also been proposed. \citet{rajeswaran2019meta} proposed meta-learning with implicit gradients by formulating the problem as bilevel optimization. Besides, Hessian-free variants of MAML have been proposed to improve computational efficiency \citep{finn2017model, nichol2018first, Zhou2019EfficientML,Song2020ESMAMLSH, fallah2020convergence}.

\paragraph{Optimization theory of MAML} In recent years, researchers have focused on addressing the computational and optimization challenges of MAML and its variants. For instance, \citet{balcan2019provable} provided provable guarantees of a generalized framework of gradient-based MAML in the convex online learning scheme. When the loss function $\L_i$ is convex, global convergence of MAML has been established for meta-supervised learning and meta-reinforcement learning in \citet{wang2020global}. When $\L_i$ is nonconvex, the convergence to stationary points of MAML and its first-order and Hessian-free variants is proved in \citet{fallah2020convergence}. While the BSGD algorithm proposed in \citet{hu2020biased} has been shown to have theoretical and practical advantages over the results in \citep{fallah2020convergence}, it relies on stronger assumptions. As previously noted, these results are still not entirely satisfactory for MAML. Finally, the convergence of iMAML to stationary points in the nonconvex setting has been demonstrated, but the theory requires processing all $n$ tasks at each iteration. Besides, SCGD~\citep{wang2017stochastic} and NASA~\citep{Ghadimi2020AST} can also be used to optimize the MAML objective since problem~(Eq. \eqref{eq:maml}) can be viewed as an instance of stochastic two-level compositional problems in the form of $\E_{\xi}[f_{\xi}(\E_{\xi'}[\mathbf g(\w; \xi')]; \xi)]$, where $\xi' = (\z'_1,\ldots, \z'_N)\sim\mathbf D_1\times\ldots\times \mathbf D_n$, $\xi = (\z_1,\ldots, \z_n)\sim\mathbf D_1\times\ldots\times \mathbf D_n$, $\mathbf g(\w; \xi') = [\w - \alpha \nabla \ell_1(\w; \z'_1); \cdots; \w-\alpha\nabla \ell_n(\w; \z'_n)]\in\R^{nd}$,  and $f_{\xi}(\mathbf g) = \frac{1}{n}\sum_{i=1}^n \ell_i([\mathbf g]_i, \z_i)$. However, a limitation of using SCGD or NASA to solve MAML is that they require passing through all $n$ tasks at each iteration~\citep{wang2017stochastic,Ghadimi2020AST}. Additionally, these works assume that both $\nabla f$ and $\nabla \bg$ are bounded.

\paragraph{Federated learning related to MAML} Our work builds upon previous research that has explored the relationship between MAML and classical federated averaging (FedAvg) in federated learning (FL). In \citet{Jiang2019ImprovingFL}, the authors show how FedAvg can be connected to MAML and derive a heuristic-based algorithm that alternates between running FedAvg for several iterations and using a meta-learning approach for fine-tuning. This results in a good initial model for any client and improves personalized performance even when the local data is limited. \citet{fallah2020personalized} show that the MAML-based Per-FedAvg leads to superior personalized federated learning performance compared to FedAvg on some numerical experiments. Personalized federated learning similar to but not exactly the same as MAML has also been considered in several recent works~\citep{Hanzely2020FederatedLO,NEURIPS2020_f4f1f13c}. Our work specifically focuses on federated learning with the vanilla MAML formulation, similar to \citet{fallah2020personalized}. 

\paragraph{Continual learning} Finally, it is worth noting that using a memory buffer to track each task has been explored in other continual learning paradigms to address the problem of catastrophic forgetting in a learning agent, such as memory-based lifelong learning \citep{10.5555/3295222.3295393,kirkpatrick2016overcoming,DBLP:conf/nips/GuoLYR20}. However, it is important to emphasize that the memory used in our MAML algorithms and that used in lifelong learning are distinct. In MAML, the memory is used to track individual models of different tasks, while in lifelong learning, it is used to store some training data for different tasks. 

Finally, we note that the proposed techniques can be employed for solving other problems with similar structures to MAML, e.g., the meta-tailoring problem~\citep{alet2021tailoring}.

\section{Preliminaries}

In this section, we present the notation, assumptions, and key challenges in solving Eq.~\eqref{eq:maml}. 

\subsection{Notation}
The Euclidean norm of a vector and the spectral norm of a matrix are denoted by $\|\cdot\|$. Calligraphic and capital letters, such as $\B$ and $\S$, denote sets. For a data distribution $\mathbf D$, we use $\S\sim \mathbf D$ to denote a set of i.i.d. samples following the distribution $\mathbf D$. We use $\mathbb I()$ to denote the indicator function. The unbiased stochastic gradient and stochastic Hessian of the risk function $\L_i$ based on a random set $\S\sim\D_i$ of size $K$ are denoted by $\nablah_{\S} \L_i(\w)= \frac{1}{K}\sum_{\z_i\in\S} \nabla \ell_i(\w;\z_{i})$, and $\nablah_{\S}^2 \L_i(\w) = \frac{1}{K}\sum_{\z_i\in \S}\nabla^2\ell_i(\w; \z_{i})$, respectively. Refer to Table~\ref{tab:notation} for a complete list of notations used in this paper.

\subsection{Assumptions}
Throughout the paper, we assume that \Cref{asm:smoothness},~\ref{asm:bounded_var},~and \ref{asm:bounded_below} are satisfied, which are standard in the literature~\citep{fallah2020convergence,ji2020multistep,rajeswaran2019meta}.
\begin{asm}\label{asm:smoothness}
	$\L_i(\cdot)$ has $L$-Lipschitz continuous gradient and $\rho$-Lipschitz continuous Hessian, that is, $\|\nabla\L_i(\w) - \nabla \L_i(\w')\|\leq L\|\w - \w'\|$, and $\|\nabla^2\L_i(\w) - \nabla^2 \L_i(\w')\|\leq \rho\|\w - \w'\|$ for any $\w, \w'\in\R^d$. 
\end{asm}
\begin{asm}\label{asm:bounded_var}
	The variance of stochastic gradient $\nabla\ell(\w,\z)$ and stochastic Hessian $\nabla^2 \ell(\w,\z)$ are upper bounded:
	\begin{align*}
		\E_{\z\sim\mathbf D_i}\bigg[{\Norm{\nabla \ell(\w;\z) - \nabla \L_i(\w)}^2}\bigg]\leq \sigma_G^2,\quad \E_{\z\sim\mathbf D_i}\bigg[{\Norm{\nabla^2 \ell(\w;\z) - \nabla^2\L_i(\w)}^2}\bigg] \leq \sigma_H^2.
	\end{align*}
\end{asm}
\begin{asm}\label{asm:bounded_below}
	$F$ is bounded below, $\inf_{\w\in\R^d} F(\w) > -\infty$. 
\end{asm}
Most of the existing results to solve Eq.~\eqref{eq:maml} in the literature~\citep{finn2017model,rajeswaran2019meta,wang2017stochastic,Ghadimi2020AST,chen2020solving,fallah2020personalized} are under the bounded gradient assumption (Assumption~\ref{asm:bounded_grad}).
\begin{asm}
	There exists $G>0$, $\Norm{\nabla \L_i(\w)}\leq G$ for any $\w\in\R^d$.
	\label{asm:bounded_grad}
\end{asm}
Instead of \Cref{asm:bounded_grad}, \citet{fallah2020convergence} establish the convergence theory of MAML based on \Cref{asm:bound_grad_dis}, where the gradients are not necessarily bounded.
\begin{asm} There exists $\gamma_G\geq 0$, $\frac{1}{n}\sum_{i=1}^n \Norm{\nabla \L_i(\w) - \nabla \L(\w)}^2\leq \gamma_G^2$ for all $\w\in\R^d$ and $\L(\w)\coloneqq \frac{1}{n}\sum_{i=1}^n \L_i(\w)$.
	\label{asm:bound_grad_dis}
\end{asm}
\subsection{Main Challenges}\label{sec:challenges}
A key to the design of stochastic optimization of~(Eq. \eqref{eq:maml}) is to estimate the gradient of the objective $\nabla F(\w) = 	\frac{1}{n}\sum_{i=1}^n(I - \alpha\nabla^2 \L_i(\w) )\nabla \L_i(\w - \alpha\nabla \L_i(\w))$ based on random samples. Existing algorithms, such as those proposed in~\citep{fallah2020convergence,hu2020biased}, typically estimate the gradient of the objective via mini-batch averaging:
\begin{align}\label{eq:est}
	\Deltah_{\B} = \frac{1}{B}\sum_{i\in \B} (I - \alpha\nablah_{\S_2^i}^2 \L_i(\w)) \nablah_{\S_3^i} \L_i(\w- \alpha\nablah_{\S_1^i}  \L_i(\w)),
\end{align}
where $\B$ denotes the set of $B$ sampled tasks, $\S_1^i, \S_2^i, \S_3^i$ denote three independent sample sets of size $K$ for each sampled task $\T_i$. However, this na\"ive approach could lead to a large optimization error when $K$ is not large enough.

Thus, the first challenge is to design an algorithm that provably converges with sub-sampled tasks and a constant number of data points. To tackle this challenge, we borrow the idea from \cite{wang2017stochastic} that keeps track of the sequence $\v_i(\w_t) = \w_t - \alpha \nabla \L_i(\w_t)$ with an estimator $\u^i$ for each task $\T_i$ (also known as the personalized model). The first novelty of our work, compared to prior work~\citep{wang2017stochastic}, lies in the task sampling approach, where the algorithm only needs to sample data and compute the stochastic gradients for a subset of $B$ tasks, instead of all $n$ tasks. 

Moreover, it is even more challenging to establish similar convergence guarantees without the bounded gradient assumption (\Cref{asm:bounded_grad}). As shown in Lemma~\ref{lem:smoothness}, the gradient-Lipschitz parameter of the meta-objective $F(\w)$ is $L(\w)\coloneqq 4L + \frac{2\rho \alpha }{n}\sum_{i=1}^n\Norm{\nabla \L_i(\w)}$.
To handle the unbounded gradients, \citet{fallah2020convergence} estimate the gradient-Lipschitz parameter $L(\w_t)$ by a stochastic estimator $\widehat{L}(\w_t)\coloneqq 4L + \frac{2\rho \alpha}{|\B_{L_t}|}\sum_{i\in \B_{L_t}}\Norm{\nablah_{\S_{L_t}^i}\L_i(\w_t)}$ and set the stepsize $\eta_t$ to be inversely proportional to $\widehat{L}(\w_t)$ .  However, MAML with that stepsize $\eta_t$ still requires $K=\O(1/\epsilon^2)$ data samples per task in each iteration to ensure convergence. Thus, it was still an open problem whether the proposed technique can be extended to this setting for getting rid of the unrealistic requirement of large batch size.

\section{Memory-Based MAML (MOML) in the Single-Node Learning}
We tackle the challenges mentioned in \Cref{sec:challenges} by proposing the MOML algorithm. 

\subsection{Algorithm Outline}
The proposed \momlvo~(\Cref{alg:moml})~updates the personalized model of a sampled task $\T_i$, $i\in \B_t$ by a momentum step while those of the other tasks $i\not\in\B_t$ are untouched, that is, 
\begin{align}\label{eq:moml_v1}
	\u_{t+1}^i =\begin{cases}
		(1-\beta_t)\u_t^i + \beta_t \vhat_t^i & i\in\B_t\\
		\u_t^i & i\not\in\B_t
	\end{cases},\tag{v1}
\end{align}
where $\beta\in(0,1]$ is the momentum factor. It is worth noting that \momlvo~with $\beta=1$ recovers the original MAML algorithm. Based on the updated personalized models $\u_{t+1}^i$, we can compute the stochastic gradient by
\begin{align}\label{eq:moml_est}
	\Deltah_{\B_t} = \frac{1}{B}\sum_{i\in \B_t} (I - \alpha\nablah_{\S_2^i}^2 \L_i(\w_t)) \nablah_{\S_3^i} \L_i(\u_{t+1}^i).
\end{align}

\begin{algorithm}[t]
	\caption{$\text{MOML}^{\text{v1}}$}\label{alg:moml}
	\begin{algorithmic}[1]
	\STATE Hyperparameters: $\beta$ (suggested value 0.5), $\eta$ (to be tuned in practice)
		\FOR{$t=0,1,\ldots,T-1$}
		\STATE Select a batch of $B$ tasks $\B_t$ from $n$ tasks
		\FOR{each task $\T_i$, $i\in\B_t$} 
		\STATE Select $K$ samples $\S_1^i \sim \mathbf D_i$ 
		and compute $\vhat_t^i = \w_t - \alpha \widehat{\nabla}_{\S_1^i}\L_i(\w_t)$
		\STATE Update the personalized model by $\u_{t+1}^i = (1-\beta)\u_t^i + \beta \vhat_t^i$.
		\ENDFOR
		\STATE Select $K$ samples $\S_2^i$ and $\S_3^i$ from $\mathbf D_i$ of task $\T_i$, $i\in\B_t$ and compute $\Deltah_{\B_t}$ by (\ref{eq:moml_est})
		\STATE Update the meta-model by $\w_{t+1} = \w_t - \eta  \Deltah_{\B_t}$
		\ENDFOR
	\end{algorithmic}
\end{algorithm}

What is the intuition behind \eqref{eq:moml_v1} and \eqref{eq:moml_est}? When the batch size $|\S_1^i|$ is not large enough, the estimator $\widehat{\v}_t^i \coloneqq \w-\alpha \widehat{\nabla}_{\S_1^i}\L_i(\w)$ to compute $\widehat{\Delta}_\B$ in \eqref{eq:est} might lead to large error to estimate $\w - \alpha \nabla \L_i(\w)$ and impede the convergence. 
Instead, we design a new personalized model estimator $\u_t^i$ which is an exponential moving average of many ``historical estimators'' $\widehat{\v}_{t'}^i$ in the past iterations $t'<t$, which covers much more data points and its estimation error is provably small. Please refer to the discussion after Lemma 2 for a formal justification. 	

\subsection{Convergence Analysis}

We establish the convergence guarantees of \momlvo~based on Assumption~\ref{asm:smoothness},~\ref{asm:bounded_var},~\ref{asm:bounded_below} and \ref{asm:bounded_grad}. With these assumptions, the meta-objective $F(\cdot)$ is $L_F$-smooth (see Lemma~\ref{lem:smoothness}). Based on this fact, we can derive the lemma below.
\begin{lemma}\label{lem:moml_all_steps}	
	If $\alpha \in(0,1/L]$, the iterates $\{\w_t\}_{t=0}^{T-1}$ of \emph{\momlvo}~satisfy that 
\begin{align*}
\frac{1}{T}\sum_{t=0}^{T-1}	\E\left[\Norm{\nabla F(\w_t)}^2\right] \leq \frac{2F(\w_0)}{\eta T}  + \frac{\eta L_F}{T}  \sum_{t=0}^{T-1}\E\left[\Norm{\widehat{\Delta}_{\B_t}}^2\right] + \frac{8L^2}{BT} \E\left[\sum_{t=0}^{T-1}\sum_{i\in \B_t} \Norm{\v_i(\w_t) - \u_{t+1}^i}^2\right].
	\end{align*}
\end{lemma}	
Next, we need to show the error of tracking $\v_i(\w_t)$ (the last term in Lemma~\ref{lem:moml_all_steps}) is vanishing when $\beta\in(0,1)$. Different from existing analysis of stochastic compositional optimization~\citep{wang2017stochastic}, the estimators $\u^i$ for tracking the inner functions $\v(\w) \coloneqq ([\v_i(\w)]^\top,\dotsc, [\v_n(\w)]^\top)^\top$ are only partially updated due to the task sampling. Hence, we need a different technique to bound the error. Given the total number of iterations $T$, we define $n$ totally ordered sets $\bT_1,\dotsc,\bT_n$, where $\bT_i \subseteq [T]$ contains the iteration indices that the $i$-th task is sampled, that is to say, $\bT_i=\{t^i_1, \ldots, t^i_k,\ldots,\}$.  If task $\T_i$ is sampled in iteration $t$ and the index of $t$ in $\bT_i$ is $k$, then $t^i_k=t$. Hence, we can define a mapping from $t$ to $k$ for $i\in\B_t$. Based on this definition, we can obtain that
\begin{align}\label{eq:equiv}
	\sum_{t=0}^{T-1}\sum_{i\in \B_t} \Norm{\v_i(\w_t) - \u_{t+1}^i}^2 = \sum_{i=1}^n\sum_{k=0}^{T_i-1} \Norm{\v_i(\w_{t_k^i}) - \u_{t_k^i+1}^i}^2.
\end{align}
Here $T_i$ is the cardinality of set $\bT_i$, which is random and depends on the task sampling $\{\B_t\}_{t=0}^{T-1}$. Lemma~\ref{lem:error_estimate} upper bounds the right-hand side of \eqref{eq:equiv}.
\begin{lemma}\label{lem:error_estimate}
	Suppose that the batch of tasks $\B_t$ is sampled uniformly at random. The error of \emph{\momlvo}~with $|\S_1^i| = K$ to keep track of $\v_i(\w_t)$ can be upper bounded as
	\begin{align}\label{eq:error_estimate}
		\E\left[\frac{1}{T}\sum_{t=0}^{T-1}\sum_{i\in \B_t} \Norm{\v_i(\w_t) - \u_{t+1}^i}^2\right]	\leq \left(\frac{n\sigma_G^2}{\beta K T} + \frac{16\eta^2 n^2 C_\Delta}{\beta^2 B^2}\right)\mathbb{I}[\beta\in(0,1)] + \frac{\beta \alpha^2 \sigma_G^2}{K},
	\end{align}
where $C_\Delta \coloneqq (\alpha^2\sigma_H^2/K + (1+\alpha L)^2)(\sigma_G^2/K+G^2)$.
\end{lemma}	

Lemmas~\ref{lem:moml_all_steps} and~\ref{lem:error_estimate} explain why our MOML algorithm converges with $B=\O(1)$ tasks and $K=\O(1)$ samples while previous works on MAML do not. As shown in Lemma~\ref{lem:error_estimate}, MOML's estimation error $	\E\left[\frac{1}{B}\sum_{i\in\B_t}\Norm{\v_i(\w_t)-\widehat{\v}_t^i}^2\right]\leq \O(\epsilon^2)$ even with $K=\O(1)$ by setting $\eta=\O(\epsilon^3)$, $\beta=\O(\epsilon^2)$. For the convergence of MAML, both a) the product of stepsize $\eta$ and the second moment of the stochastic meta-gradient $\hat{\Delta}_{\B_t}$ and b) the estimation error of personal models $\widehat{\v}_t^i = \w_t - \alpha \widehat{\nabla}_{\S_1^i}\L_i(\w_t)$ should be small. To be specific, it needs 
\begin{align}\label{eq:converge_cond}
\eta \E\left[\Norm{\hat{\Delta}_{\B_t}}^2\right]\leq \O(\epsilon^2) \quad \text{and}\quad \E\left[\frac{1}{B}\sum_{i\in\B_t}\Norm{\v_i(\w_t)-\widehat{\v}_t^i}^2\right]\leq \O(\epsilon^2)\tag{$\star$}	
\end{align}
for an $\epsilon$-stationary point. \citet{fallah2020convergence} on MAML sets $\eta=\O(1)$ and bounds the second moment as follows.
\begin{align*}
	\E\left[\Norm{\widehat{\Delta}_{\B_t}}^2\right]\leq  \begin{cases} \underbrace{2\left(1+\frac{20}{B}\right)\E\left[\Norm{\nabla F(\w_t)}^2\right]}_{\coloneqq \ddagger} + \frac{\sigma_G^2}{K} +\frac{1}{B}\left(14 G^2 + \frac{3\sigma_G^2}{K}\right) & \makecell{\text{total~\#tasks}~$n$\\~\text{is infinite}}\\
	\underbrace{2\left(1+\frac{20}{B}\right)\E\left[\Norm{\nabla F(\w_t)}^2\right]}_{\coloneqq \ddagger} + \frac{\sigma_G^2}{K} +\frac{(n-B)}{B(n-1)}\left(14 G^2 + \frac{3\sigma_G^2}{K}\right) & \makecell{\text{total~\#tasks}~$n$\\~\text{is finite.}}
		\end{cases}
\end{align*}
The $\ddagger$ term can be canceled out with the L.H.S. of Lemma 1. Besides, MAML (i.e. MOML with $\beta=1$) also satisfies
\begin{align}\label{eq:fval_approx_maml}
	\E\left[\frac{1}{B}\sum_{i\in\B_t}\Norm{\v_i(\w_t)-\widehat{\v}_t^i}^2\right]\leq \frac{\alpha^2\sigma_G^2}{K}.\tag{$\diamond$}
\end{align}
To make \eqref{eq:converge_cond} hold, \citet{fallah2020convergence} need $B = \O(\epsilon^{-2})$, $K=\O(\epsilon^{-2})$ for infinite $n$ case while $B = n$, $K=\O(\epsilon^{-2})$ for finite $n$ case. More recent work \cite{hu2020biased} instead use a small stepsize $\eta=\O(\epsilon^2)$ for MAML. Thus, making \eqref{eq:converge_cond} hold only needs $B=\O(1)$ tasks. However, it still needs $K=\O(\epsilon^{-2})$ data points due to \eqref{eq:converge_cond} and \eqref{eq:fval_approx_maml}. It is not clear how to remove the $K=\O(\epsilon^{-2})$ requirement for vanilla MAML. 
Next, we are ready to present the main convergence theorem of $\text{MOML}^{\text{v1}}$.
\begin{theorem}[Informal]
	Under Assumptions~\ref{asm:smoothness},~\ref{asm:bounded_var}~\ref{asm:bounded_below},~\ref{asm:bounded_grad}, \emph{\momlvo}~with stepsizes $\eta_t =\eta = \O(\epsilon^{-3})$, $\beta_t =\beta= \O(\epsilon^{-2})$ and constant batch sizes $|\S_1^i|=|\S_2^i|=|\S_3^i| = K = \O(1)$, $|\B_t| = B = \O(1)$ can find a stationary point $\w_\tau$ in $T=\O(n\epsilon^{-5})$ iterations.
	\label{thm:moml_v1_informal}
\end{theorem}

Compared to previous works that aim to solve Eq. \eqref{eq:maml} under the same assumptions, \momlvo~can ensure convergence without the need for an extremely large batch $K=\O(1/\epsilon^2)$, as required in \citep{hu2020biased}, or processing all $n$ tasks~\citep{Ghadimi2020AST}.

\subsection{Handling Unbounded Gradients}

We propose another variant of MOML --- \momlvs~in \Cref{alg:moml_v2} that is provably convergent with possibly unbounded gradients and only requires $K=\O(1)$ data samples for each sampled task in one iteration. In \momlvs, the personalized model is updated as $\u^i$
\begin{align}\label{eq:moml_v2}
	\u_{t+1}^i =\begin{cases}
		(1-\beta_t)\u_t^i + \beta_t \w_t + \frac{\beta_t }{p_i} (\vhat_t^i - \w_t)& i\in\B'_t\\
		(1-\beta_t)\u_t^i + \beta_t \w_t & i\not\in\B'_t
	\end{cases},\tag{v2}	
\end{align}
where $\B'_t$ is independent of $\B_t$ and $p_i$ is the probability of selecting task $i$, that is, $p_i = \mathrm{Prob}(i\in\B'_t)$. Besides, we set $\eta_t = \frac{\eta_0}{\widehat{L}(\w_t)}$ and $\beta_t = 6L^2\eta_0^{-1/3}\eta_{t-1}$ for \momlvs. Under the same set of assumptions as \citet{fallah2020convergence}, we can show the error of tracking the inner function $\v_i(\w_t)$ is diminishing for \momlvs~when $\eta_0$ is properly chosen. 
\begin{lemma}\label{lem:fval_recursion}
	If $\eta_0 \leq \min\{(1/3L)^{3/2}, (3L^2/C_3 C_6)^{3/2}\}$, we have
		\begin{align*}
			& \E\left[\Upsilon_{t+1}\mid \F_t\right] \leq \left(1-3L^2 \eta_0^{-\frac{1}{3}}\E\left[\eta_t\mid \F_t\right]\right) \E\left[\Upsilon_t\mid \F_t\right] + \eta_0^{\frac{1}{3}}\E\left[\eta_t\mid \F_t\right]C_9\Norm{\nabla F(\w_t)}^2 + \eta_0^{\frac{4}{3}}C_{10} ,
		\end{align*}
	where $\Upsilon_t\coloneqq \frac{1}{n}\sum_{i=1}^n\Norm{\u_{t+1}^i - \v_i(\w_t)}^2$ and $C_9$, $C_{10}$ are $\O(1)$ constants w.r.t. $\epsilon$ and $n$~\footnote{Proof of Lemma~\ref{lem:fval_recursion} in the appendix specifies the detailed expressions of $C_9$ and  $C_{10}$.}.
\end{lemma} 

\begin{theorem}[Informal]
	Under Assumptions~\ref{asm:smoothness}, \ref{asm:bounded_var}, \ref{asm:bounded_below}, and \ref{asm:bound_grad_dis}, it is guaranteed that \emph{$\text{MOML}^{\text{v2}}$} with $\eta_t = \frac{\eta_0}{\widehat{L}(\w_t)}$, $\beta_t = 6L^2\eta_0^{-1/3}\eta_{t-1}$, $\eta_0 = \O(\epsilon^{-3})$ and constant batch sizes $|\S_1^i|=|\S_2^i|=|\S_3^i| = K = \O(1)$, $|\B_t|=|\B_t'| = B = \O(1)$ can find an $\epsilon$-stationary point $\w_\tau$ in $T=O(C_p\epsilon^{-5})$ iterations, where $C_p = \max_i 1/p_i - 1$.
\label{thm:moml_v2_informal}
\end{theorem}

\begin{algorithm}[t]
	\caption{$\text{MOML}^{\text{v2}}$}\label{alg:moml_v2}
	\begin{algorithmic}[1]
\STATE	Hyperparameters: $\beta$ (suggested value 0.5), $\eta$ (to be tuned in practice)
		\FOR{$t=0,1,\ldots,T-1$}
		\STATE Select two mutually independent batches of tasks $\B_t$, $\B_t'$ from $n$ tasks
		\FOR{each task $\T_i$, $i\in\B_t'$} 
		\STATE Select $K$ samples $\S_1^i \sim \mathbf D_i$ 
		and compute $\vhat_t^i = \w_t - \alpha \widehat{\nabla}_{\S_1^i}\L_i(\w_t)$
		\ENDFOR
		\STATE Update the personalized model $\u^i$ by
		\begin{align*}
				\u_{t+1}^i =\begin{cases}
				(1-\beta_t)\u_t^i + \beta_t \w_t + \frac{\beta_t }{p_i} (\vhat_t^i - \w_t)& i\in\B_t'\\
				(1-\beta_t)\u_t^i + \beta_t \w_t & i\not\in\B_t'
			\end{cases}.
		\end{align*}
		\STATE Select $K$ samples $\S_2^i$ and $\S_3^i$ from $\mathbf D_i$ of task $\T_i$, $i\in\B_t$ and compute $\Deltah_{\B_t}$ by (\ref{eq:moml_est})
		\STATE Update the meta-model by $\w_{t+1} = \w_t - \eta_t  \Deltah_{\B_t}$
		\ENDFOR
	\end{algorithmic}
\end{algorithm}

\begin{remark}
Theorem~\ref{thm:moml_v2_informal} demonstrates that \emph{\momlvs} improves upon the theory of MAML~\citep{fallah2020convergence} by eliminating the need to sample $K=\O(1/\epsilon^2)$ data samples in each iteration. However, it should be noted that both the \emph{MAML} variant in \citet{fallah2020convergence} and \emph{\momlvs} are more of theoretical interest because they require additional tasks/samples $\B_{\L_t}$ and $\S_{L_t}$ to estimate the gradient-Lipschitz parameter and then calculate the step size $\eta_t$ in each iteration. Previous research~\citep{ji2020multistep} and our experiments have demonstrated that the gradients $\nabla \L_i(\w_t)$ are well-bounded during the meta-training process, making \emph{\momlvo} with a constant step size a more practical variant.
\end{remark}	
 
 \section{LocalMOML for Personalized Federated Learning}
This section presents a stochastic algorithm for solving MAML in the federated learning setting, assuming that there are $n$ clients, and the $i$-th task and its corresponding data are only accessible at the $i$-th client. Note that this assumption can be relaxed to a setting where the $n$ tasks are allocated to $m<n$ clients with non-overlapping. These clients are only permitted to aggregate the models and not exchange any data. However, a naive implementation of MOML in the federated learning setting would require aggregating the local gradient estimator $\Deltah_{\B}^i$ at every iteration, or equivalently, aggregating the local copies of the meta-model at every iteration. Consequently, the communication complexity would be as high as the iteration complexity $T = \O(n/\epsilon^5)$. Our objective is to reduce communication complexity by proposing communication-efficient federated learning algorithms.
 
 \subsection{Algorithm Outline}
  Our algorithm, called LocalMOML, is presented in Algorithm~\ref{alg:fed_moml}. This algorithm is partially motivated by the numerous algorithms in federated learning that use local computations to trade-off communications~\citep{DBLP:conf/nips/Yang13,mcmahan2017communication,deng2020distributionally,karimireddy2019scaffold,stich2018local,DBLP:journals/corr/abs-1808-07217,woodworth2020local,khaled2020tighter}. Each client not only maintains and updates its personalized model $\u^i$ but also maintains and updates its local copy of the meta-model denoted by $\w^i$. A key feature of LocalMOML is that $H$ local steps are run on each sampled client before these clients communicate to aggregate the local meta-models.
 \begin{algorithm}[t]
 	\caption{LocalMOML}\label{alg:fed_moml}
 	\begin{algorithmic}[1]
 	\STATE	Hyperparameters: $\beta$ (suggested value 0.5), $\eta$ (to be tuned in practice)
 		\FOR{$r=1,\ldots,R$} 
 		\STATE Select a batch of $B$ clients $\B_r$. 
 		\FOR{each client $\C_i$, $i\in\B_r$} 
 		\IF{client sampling} 
 		\STATE Select $K_0$ samples $\S_0^i\sim \D_i$, reset $\u_{r,1}^i = \w_{r,1}^i - \alpha \nablah_{\S_0^i}\L_i(\w_{r,1}^i)$
 		\ELSE
 		\STATE Set $\u_{r,1}^i = \u_{r-1,H}^i$ 		
 		\ENDIF
 		\FOR{$h=1,\dotsc,H$} 
 		\STATE Sample $\S_1^i$ to update the personalized model $\u_{r,h}^i$ by (\ref{eq:fed_update_u})
 		\STATE Select two sets $\S_2^i$ and $\S_3^i$ of from $\mathbf D_i$ to compute $\Deltah_{r,h}^i$ by (\ref{eq:fed_deltah}), and update the local model by $\w_{r,h+1}^i = \w_{r,h}^i - \eta \Deltah_{r,h}^i$
 		\ENDFOR
 		\STATE Client $\C_i$ sends $ \w_{r,H+1}^i$ to the server.
 		\ENDFOR
 		\STATE The server aggregates and broadcasts $\w_{r+1} =\frac{1}{B}\sum_{i\in \B_r}\w_{r,H+1}^i$ 
 		\ENDFOR
 	\end{algorithmic}
 \end{algorithm}

We consider both the cross-silo setting and the cross-device setting in the literature on federated learning.  In the cross-device setting only a partial set of $B<n$ clients are sampled to update their local models at each round, while in the  cross-silo setting all $n$ clients will participate in updating the model at each round, that is,  $B=n$. In these two settings, LocalMOML needs different ways to initialize the personalized model $\u_{r,h}^i$ at the beginning of each round. In the cross-silo setting ($B=n$), personalized models are directly copied from the end of the previous round, that is, $\u_{r,1}^i = \u_{r-1, H}^i$. In the cross-device setting, ($B<n$), the personalized models for the sampled tasks are restarted as $\u_{r,1}^i = \w_{r,1}^i - \alpha \nablah_{\S_0^i} \L_i(\w_{r,1}^i)$ for $i \in\B_r$. Then, the personalized model $\u_{r,h}^i$ for a sampled client $i$ is updated as:
\begin{align}\label{eq:fed_update_u}
	\u_{r,h}^i =  
	(1-\beta) \u_{r,h-1}^i + \beta \left(\w_{r,h}^i - \alpha \nablah_{\S_1^i} \L_i(\w_{r,h}^i)\right).
\end{align}
Based on the local personalized model $\u^i$ and meta model $\w^i$, the stochastic gradient estimator is computed as:
\begin{align}\label{eq:fed_deltah}
	\Deltah_{r,h}^i = (I-\alpha \nablah^2_{\S_2^i}\L_i(\w_{r,h}^i))\nablah_{\S_3^i}\L_i(\u_{r,h}^i).
\end{align}	
Once all $H$ iterations have been completed in each round, the local copies of the meta-model at each client are aggregated for synchronization. We note that the Per-FedAvg algorithm~\citep{fallah2020personalized} is a special case of LocalMOML (\Cref{alg:fed_moml}) with $\beta = 1$, and that the original MAML algorithm corresponds to LocalMOML with $\beta = 1$ and $H=1$.

It is worth noting that LocalMOML can also be implemented in the single-node learning setting, for example by parallelizing on multiple CPU cores of a single machine. While MOML needs to maintain individualized models for all $n$ tasks during the entire training process, which limits its scalability to a large number of tasks, LocalMOML only needs to maintain individualized models for a small subset of $B$ sampled tasks within each round (epoch) $r$, where $B$ can be small enough to avoid critical memory issues. After completing the iterations of round $r$, the individualized model $\u_{r,H}^i$ for the sampled task in round $r$ can be erased from memory and will not be used in any later round $r'>r$, even if that task is selected again. This means that LocalMOML can handle an infinite number of tasks without encountering memory constraints.

\subsection{Convergence results of LocalMOML}

We analyze LocalMOML under the same assumptions as the Per-FedAvg~\citep{fallah2020personalized}, that is, Assumptions~\ref{asm:smoothness},~\ref{asm:bounded_var},~\ref{asm:bounded_below},~\ref{asm:bounded_grad},~\ref{asm:bound_grad_dis}, and the assumption below. 
\begin{asm}\label{asm:bounded_hes_dis}
	There exists $\gamma_H\geq 0$ that: $\frac{1}{n}\sum_{i=1}^n \Norm{\nabla^2 \L_i(\w) - \frac{1}{n}\sum_{i=1}^n\nabla^2 \L_i(\w)}^2 \leq \gamma_H^2$.
\end{asm}
Note that Assumptions~\ref{asm:smoothness} and \ref{asm:bounded_grad} implies Assumptions~\ref{asm:bound_grad_dis} and \ref{asm:bounded_hes_dis}. However, directly utilizing Assumptions~\ref{asm:bound_grad_dis} and \ref{asm:bounded_hes_dis} in the analysis can explicitly show the impact of the dissimilarity of data distribution on the final performance~\citep{fallah2020personalized}.

\begin{theorem}\label{thm:fed_moml}
	$R$ rounds of \emph{LocalMOML} with the stepsize $	\eta \leq \min\left\{\frac{C_4}{H},\frac{C_5\beta}{\mathbb{I}[\beta\in(0,1)]}\right\}$ leads to
	\begin{align*}
		& \frac{1}{R}\sum_{r=1}^R \E\left[\Norm{\nabla F(\w_r)}^2\right] \\
		& \leq \frac{4F(\w_1)}{\eta T} + 32\eta^2 H(H-1)C_1(\hat{\sigma}^2 + 2\gamma_F^2) + \frac{4\eta}{B}\left(\hat{\sigma}^2 + \frac{4(n-B)}{(n-1)}H\gamma_F^2\right) + \frac{64C_3\beta\alpha^2\sigma_G^2}{|\S_1^i|}\\
		& + \frac{64C_3\sigma_G^2}{\beta \left(\mathbb{I}[B<n]H S_0 + \mathbb{I}[B=n]HR \right)} \mathbb{I}[\beta\in(0,1)] + \frac{3072C_3\eta^2(\hat{\sigma}^2 + 2\gamma_F^2)}{\beta^2} \mathbb{I}[\beta\in(0,1)],
	\end{align*}
	where $T=RH$, $\hat{\sigma}^2 \coloneqq \frac{2\sigma_G^2}{|\S_3^i|} + \frac{2\alpha^2\sigma_G^2}{|\S_3^i|}\left(\frac{\sigma_H^2}{|\S_2^i|} + L^2\right)+ \frac{\alpha^2G^2\sigma_H^2}{|\S_2^i|} $, and $C_1, C_3, C_4, C_5$ are $\O(1)$ constants w.r.t. $\epsilon$ and $n$. 
\end{theorem}

We consider the special case $\alpha = 0$, $H = 1$, $\beta=1$, $B =n$ where LocalMOML becomes SGD and Theorem~\ref{thm:fed_moml} recovers the standard result of SGD: $\frac{1}{T}\sum_{t=1}^T \E\left[\Norm{\nabla F(\w_t)}^2\right]\leq \mathcal O\left(\frac{F(\mathbf w_1)}{\eta T} + \eta \sigma_G^2\right)$. Besides, the term $\gamma_F^2\coloneqq 3G^2 \alpha^2 \gamma_H^2 + 192 \gamma_G^2$ in Theorem~\ref{thm:fed_moml} explicitly shows the impact of the dissimilarity of data distribution on the final performance. To better understand the complexities of LocalMOML, we present two corollaries below corresponding to the cross-silo $B=n$ and cross-device $B=\O(1)$ settings. 
\begin{cor}[Cross-device Setting]
	In this setting, we set $\beta = O(\epsilon^2), \eta=\O(\epsilon^4), H= \O(1/\epsilon^2)$, $|\S_0^i| = K_0, |\S_1^i| = |\S_2^i| = |\S_3^i|= K$, $K=O(1), K_0= H$. Then, we can conclude that $T = \O(1/\epsilon^6)$ and $R= \O(1/\epsilon^4)$. Hence, the average number of data points per-iteration is $(K_0 + 3HK)/H = \O(1)$, the total sample complexity is $KBT + RK_0 = \O(1/\epsilon^6)$ and the communication complexity is $R= \O(1/\epsilon^4)$. 
\end{cor}
\begin{cor}[Cross-silo Setting]
	In this setting, we set $\beta = \O(\epsilon^2), \eta=\O(\epsilon^3),  H= \O(1/\epsilon^2)$, $|\S_0^i| = 0 , |\S_1^i| = |\S_2^i|= |\S_3^i| = K$, and $K=\O(1)$. Then, we can conclude that  $T = \O(1/\epsilon^5)$ and $R= \O(1/\epsilon^3)$. Hence, the total sample complexity is $KNT = \O(n/\epsilon^5)$ and the communication complexity is $R= \O(1/\epsilon^3)$. 
\end{cor}
Thus, in both cross-silo and cross-device settings, our LocalMOML not only removes the unrealistic requirement of the large batch size $\O(1/\epsilon^2)$ in every iteration for the convergence of Per-FedAvg but also improves the sample complexity. 

\section{Experiments}

We evaluate the performance of our proposed algorithms, \momlvo, \momlvs~and LocalMOML, on sinewave regression and one-shot classification tasks in the single-node setting. Furthermore, we demonstrate the effectiveness of LocalMOML in the simulated federated learning setting for the image classification task.

\subsection{Sinewave Regression in the Single-Node Setting} 
First, we compare our proposed \momlvo, \momlvs~and LocalMOML with baselines NASA, MAML (BSGD), BSpiderBoost, and Reptile~\citep{nichol2018first} on the sinewave regression problem~\citep{finn2017model} in the single-node setting. We generate 25 tasks for training in total, each of which is to fit the function $f(x) = A \sin(\phi + x)$, where $A = \{1,2,3,4,5\}$, $\phi = \frac{i\pi}{5}$, $i = 1,2,3,4,5$. Similar to \cite{finn2017model}, we choose the feedforward neural network with ReLU nonlinearities and 2 hidden layers of size 40 as the model and the mean-square error as the loss function. The training and validation data of each task are randomly sampled in an online manner $x\sim\text{Unif}(-5,5)$. NASA uses all $n$ tasks while the others sample $B=3$ out of $n$ tasks. We consider two possible minibatch sizes of data points: $K=1$ and $K=3$ in one iteration. The meta-learned model is adapted to 5 randomly sampled unseen tasks  $A\sim\text{Unif}(1,5)$, $\phi \sim\text{Unif}(\frac{\pi}{5},\pi)$, $x\in[-5, 5]$ after 10 steps of gradient descent with learning rate 0.01 and 10 test data points (in other words, 10-shots). After the adaptation, we evaluate the final test error on another 100 data points from the unseen task. The inner step size $\alpha$ is set to 0.01 for all algorithms. The outer step size $\eta$ is decayed 10 times at 75\% of the total iterations\footnote{except for BSpiderBoost, which requires a $\O(1)$ step size in its theory.} and its initial value is tuned for the algorithms separately by grid search in $\{0.1, 0.05, 0.01, 0.005, 0.001\}$. We also tune $\beta$ for \momlvo, \momlvs~and LocalMOML. It turns out that $\beta = 0.3$ and $\beta = 0.5$ work reasonably well for \momlvo~and LocalMOML while $\beta=0.1$ is good for \momlvs. For LocalMOML, we set the size of the initial number of samples $K_0$ of each round to be 2 times $K$ and $H=5$. The results are averaged over 5 trials with different random seeds. 

\begin{figure*}[t]
	\minipage{0.49\textwidth}
	\includegraphics[width=\linewidth]{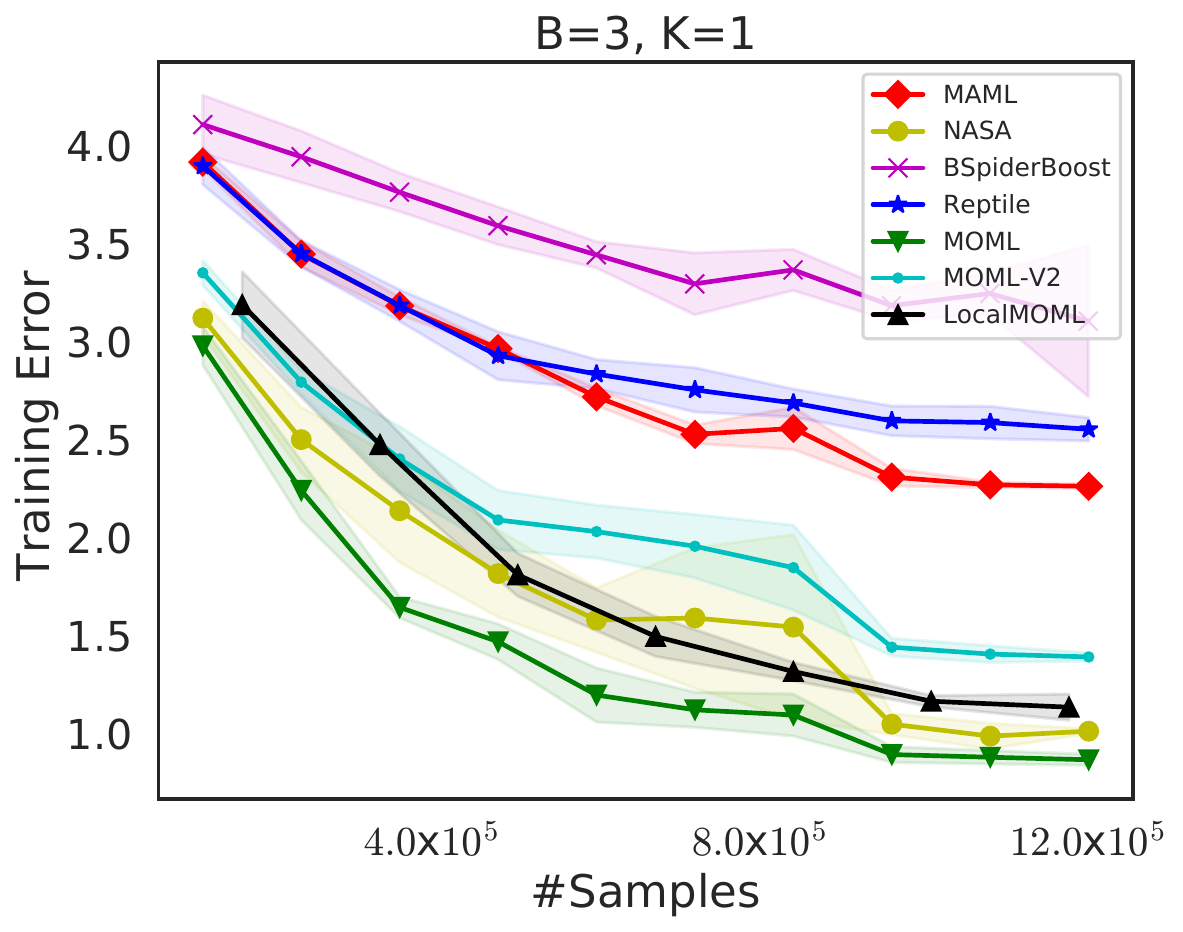}
	\endminipage\hfill  
	\minipage{0.49\textwidth}
	\includegraphics[width=\linewidth]{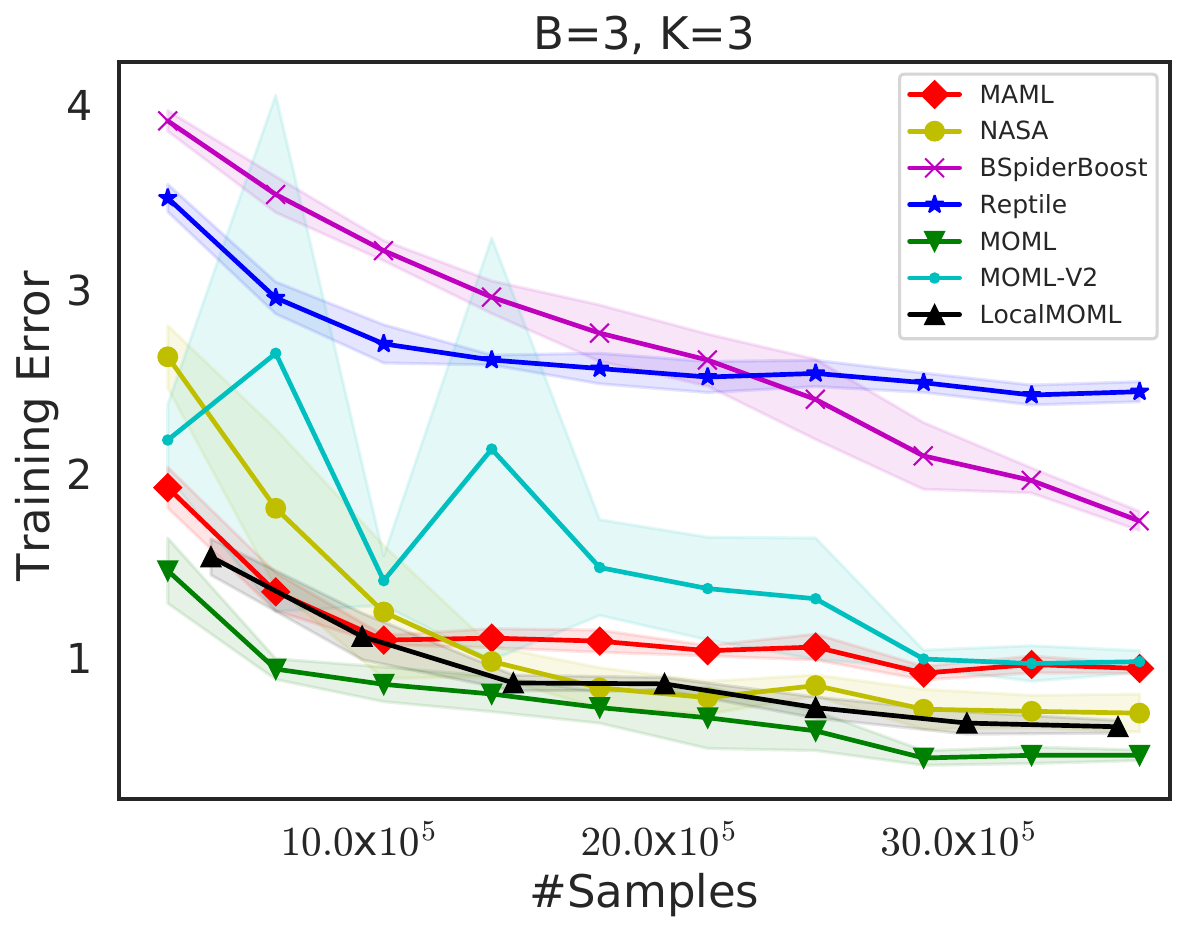}
	\endminipage\hfill  
	\caption{Convergence comparison in terms of the number of samples.}
	\label{fig:sinwave}
\end{figure*}

We compare the convergence of our proposed algorithms and the baselines in terms of the number of samples. The training error Eq.~\eqref{eq:maml} is approximated by $100n$ data points sampled from $n$ tasks. As seen in Figure~\ref{fig:sinwave}, \momlvo~converges the fastest among the algorithms. We also report the final test errors and wall-clock time per iteration in Table~\ref{tab:sinewave}. The differences between MAML (BSGD) and \momlvo~in test error seem to be significant. The reason might be that MAML (BSGD) has a large optimization error due to a small batch size $K=1,3$. The generalization error might also contribute to the differences in test error but they are out of the scope of our paper.

Besides, it seems that \momlvs~performs worse than \momlvo~in practice. Moreover, \momlvs~also takes longer time per iteration than \momlvs~because \momlvs~needs to update all $n$ personalized models while \momlvo~only needs to update personalized models for $B$ sampled tasks. In our experiment, we find that the computed stochastic gradients are well bounded across the iterations such that \momlvo~is more appropriate.

Fitted sinusoids on five unseen tasks can be found in Appendix~\ref{sec:fitted_curves}.

\begin{table*}[t]
\caption{Comparison of final test error for the sinewave regression task in the single-node setting. For each metric, the best one is highlighted in black and the second-best one is highlighted in gray.}
	\centering 
	\scalebox{0.64}{
		\begin{tabular}{cccccccc}
			\toprule
			\multicolumn{8}{c}{$K=1$}\\
			\midrule[.1em]
			Metrics & \makecell{MAML \\(BSGD)} & Reptile & NASA & BSpiderBoost  & \momlvo &\momlvs  & LocalMOML\\\midrule[.05em]
			Avg. Test Error & 0.889 $\pm$ 0.021 & 1.212 $\pm$ 0.054 &\textbf{\textcolor{gray}{0.361 $\pm$ 0.010}} &1.300 $\pm$ 0.104 & \textbf{0.291 $\pm$ 0.012}	&0.448 $\pm$ 0.012 & 0.462 $\pm$ 0.043  \\\cmidrule(lr){1-1}
			\makecell{Avg. Time Per\\ Iteration (ms)} &\textbf{\textcolor{gray}{1.713 $\pm$ 0.030}} & \textbf{1.380 $\pm$ 0.071} &17.524 $\pm$ 1.306&3.960 $\pm$ 0.090 & 2.180 $\pm$ 0.031& 7.863 $\pm$ 0.302 & 2.556 $\pm$ 0.100 \\
			\midrule[.1em]
			\multicolumn{8}{c}{$K=3$}\\
			\midrule[.1em]
			Metrics &\makecell{MAML \\(BSGD)} & Reptile & NASA & BSpiderBoost  & \momlvo &\momlvs  & LocalMOML \\\midrule[.05em]
			Avg. Test Error & 0.321 $\pm$ 0.015 & 1.132 $\pm$ 0.028  &0.214 $\pm$ 0.050 &0.650 $\pm$ 0.042 & \textbf{\textcolor{gray}{0.196 $\pm$ 0.068}}	&0.268 $\pm$ 0.026& \textbf{0.170 $\pm$ 0.020} \\\cmidrule(lr){1-1}
			\makecell{Avg. Time Per\\Iteration (ms)}  &\textbf{\textcolor{gray}{1.729 $\pm$ 0.059}} & \textbf{1.430 $\pm$ 0.068} & 17.608 $\pm$ 0.458&4.159 $\pm$ 0.023 &2.294 $\pm$ 0.123& 8.016 $\pm$ 0.121&2.624 $\pm$ 0.035 \\
			\bottomrule
	\end{tabular}}
	\label{tab:sinewave}
\end{table*}

\subsection{One-Shot Classification in the Single-Node Setting}
We also compare our proposed \momlvo, \momlvs~and LocalMOML with baselines NASA, MAML (BSGD), Reptile~\citep{nichol2018first}, and ProtoNet~\citep{Snell2017PrototypicalNF} on the Omniglot and CIFAR-100 datasets. Among the algorithms,  ProtoNet is a metric-learning-based meta-learning algorithm while the others are based on the optimization of objective Eq.~\eqref{eq:maml}. For the Omniglot dataset, we randomly select 25 tasks for training and 10 tasks for testing while each task has 20 classes (20 ways). For the CIFAR-100 dataset, we randomly select 17 tasks for training and 3 tasks for testing while each task has 5 classes (5 ways). There are no shared classes among the tasks. We only report the test accuracy since the loss value of ProtoNet is not directly comparable to the others. We choose the feedforward neural network with ReLU nonlinearities and 2 hidden layers of size 40 as the model. The results are averaged over 3 trials with different random seeds. As shown in Table~\ref{tab:few_shot}, MOML variants improve the performance of MAML on average.

\begin{table*}[t]
\caption{Test accuracy for one-shot image classification in the single-node setting. }
	\centering 
	\scalebox{0.7}{
		\begin{tabular}{cccccccc}
			\toprule
			\multicolumn{8}{c}{{\Large Omniglot (20-way)}}\\
			\midrule[.1em]
			Metrics & \makecell{MAML \\(BSGD)} & Reptile & ProtoNet & NASA   & \momlvo &\momlvs  & LocalMOML \\\midrule[.05em]
			Test Accuracy (\%) & 44.31 $\pm$ 1.29 & 46.02 $\pm$ 1.08 & 45.67 $\pm$ 3.30 &46.15 $\pm$ 0.82& \textbf{46.35 $\pm$ 1.38} & 45.81 $\pm$ 0.32 & 46.24 $\pm$ 1.70\\
			\midrule[.1em]
			\multicolumn{8}{c}{{\Large CIFAR-100 (5-way)}}\\
			\midrule[.1em]
			Metrics &\makecell{MAML \\(BSGD)} & Reptile & ProtoNet & NASA & \momlvo &\momlvs  & LocalMOML  \\\midrule[.05em]
			Test Accuracy (\%) & 40.18 $\pm$ 0.49 & 40.09 $\pm$ 0.11 &31.11 $\pm$ 6.29 & 40.60 $\pm$ 0.33& 40.49 $\pm$ 0.21 & \textbf{40.82 $\pm$ 0.25} & 40.38 $\pm$ 0.74\\
			\bottomrule
	\end{tabular}}
	\label{tab:few_shot}
\end{table*}

\subsection{Image Classification in a (Simulated) Federated Learning Setting} 
Second, we compare our LocalMOML with baselines Per-FedAvg~\citep{fallah2020personalized} and pFedMe~\citep{NEURIPS2020_f4f1f13c} on the image classification problem.  We consider three data sets, MNIST, CIFAR-10, and CIFAR-100. In order to create heterogeneous data distribution, we follow the setup in~\cite{fallah2020personalized}. In particular, 
For MNIST and CIFAR-10 (10 classes), we distribute the training data between $N=50$ clients (tasks) as follows: (i) half
of the clients, each has $a$ images of each of the first five classes; (ii) The rest half clients, each has $a/2$
images from only one of the first five classes and $2a$ images from only one of the other five classes.  For CIFAR-100, we consider the 20 super-classes and distribute the data similarly by dividing them into the first 10 classes and the other ten classes to run the same procedure. Similarly, we divide the test data among the clients with the same distribution as the one for the training data. We set $a=68$ for constructing the distributed training sets of MNIST, CIFAR-10, and CIFAR-100, and set $a=34$ for constructing the test sets of MNIST and CIFAR-10 and $a=15$ for constructing the test sets of CIFAR-100. 

We conduct experiments on four GPUs to mimic the cross-device federated learning setting, where all 50 tasks are distributed to the four GPUs roughly evenly.  At each round, all four GPUs participate in the learning but each GPU only samples a batch of $B'$ tasks from its owned tasks, and we consider two settings $B'=1, B'=5$. It means every round a total of $B=4B'$ tasks are sampled for updating. We train  a neural network with 2 hidden layers with each layer having 40 neurons and use the ReLU activation function. We use $\alpha=0.001$ and the step size for the considered algorithms is tuned in a range similar to before. For all algorithms, we consider two settings of $H=4$ and $H=10$. The mini-batch size at every iteration (including the initial one at each round) is set to $5$, that is, $K=5, K_0=5$. We tune the $\beta$ in a range $[0.1, 0.9]$, and run a total of $10000$ iterations. For pFedMe, we tune its hyperparameter $\lambda = 100$ and set the number of steps to be $50$ to solve the sub-problem accurately enough. We evaluate the test accuracy after $5$-shot learning with 10 steps of fine-tuning on the test set. The final results are reported in Table~\ref{tab:fed}. We can see that the proposed LocalMOML outperforms Per-FedAvg and pFedMe in almost all settings with substantial improvements on the most difficult CIFAR-100 data \footnote{Note that pFedMe has a tunable hyper-parameter $R$: \#steps to solve the inner sub-problem. Larger $R$ leads to better performance but higher computation costs (longer running time). We choose $R$ to make the running time of pFedMe comparable to that of Per-FedAvg/LocalMOML for a fair comparison.}.

\begin{table*}[t]
\caption{Comparison of final test accuracy (percentage) for 5-shot learning on three image classification data sets in a cross-device federated learning setting. The results as averaged over three runs with different random seeds. 
}
	\scalebox{0.73}{
		\begin{tabular}{cc|ccc|ccc}
			\toprule
			&  & \multicolumn{3}{c|}{$H=4$} & \multicolumn{3}{c}{$H=10$} \\ \hline
			&  \makecell{\#Tasks ($B'$) \\ per-GPU} & Per-FedAvg & LocalMOML & pFedMe & Per-FedAvg & LocalMOML & pFedMe \\ \midrule
			\multirow{2}{*}{MNIST} & 1 & \textbf{91.63 $\pm$ 0.80} & 91.62 $\pm$ 0.82 & 91.52 $\pm$ 0.94& 94.05 $\pm$ 0.37&\textbf{94.17 $\pm$ 0.44}& 94.10 $\pm$ 0.21\\
			& 5 & 92.19 $\pm$ 0.72 & \textbf{92.21 $\pm$ 0.73}& 92.08 $\pm$ 0.67&94.81 $\pm$ 0.28 & \textbf{94.83 $\pm$ 0.28}&94.36 $\pm$ 0.28\\
			\hline
			\multirow{2}{*}{CIFAR-10} & 1 & 63.81 $\pm$ 0.93& \textbf{66.16 $\pm$ 1.03}&63.22 $\pm$ 1.26& 66.25 $\pm$ 1.55  & \textbf{68.33 $\pm$ 1.60} & 64.80 $\pm$ 2.92\\
			& 5 & 64.44 $\pm$ 1.00 & \textbf{66.87 $\pm$ 1.08} & 62.68 $\pm$ 1.25 &67.10 $\pm$ 2.72 & \textbf{68.06 $\pm$ 1.77} &  65.66 $\pm$ 2.27\\
			\hline
			\multirow{2}{*}{CIFAR-100} & 1 & 50.65 $\pm$ 1.31 & \textbf{54.00 $\pm$ 1.10} & 49.64 $\pm$ 0.83& 52.55 $\pm$ 1.38 & \textbf{56.35 $\pm$ 0.92} & 51.61 $\pm$ 1.32 \\
			& 5 & 50.92 $\pm$ 1.31 & \textbf{54.39 $\pm$ 0.78} & 49.25 $\pm$ 1.06 & 53.37 $\pm$ 1.30 & \textbf{56.45 $\pm$ 1.49} & 50.95 $\pm$ 1.19 \\ \bottomrule
	\end{tabular}}
	\centering
	\label{tab:fed}
\end{table*}

\section{Conclusions and Discussion}
In this paper, we have focused on stochastic optimization for model-agnostic meta-learning and presented two novel algorithms for both the single-node and federated learning settings. Our MOML and LocalMOML algorithms outperform existing meta-learning algorithms in several aspects. Specifically, our convergence analysis ensures that our algorithms converge to a stationary point by sampling a fixed number of tasks and a fixed number of samples per iteration. Moreover, our LocalMOML algorithm not only reduces the computational complexity but also minimizes the communication complexity compared to existing federated learning algorithms that tackle the same problem.

One limitation of the proposed MOML algorithm is that they need to maintain an individualized model for each task during the whole training process, which makes it not applicable to the problem Eq. \eqref{eq:maml} with a large/infinite number of tasks or embedded systems with small memory for learning large models. When implemented in the single-node learning setting, LocalMOML does not suffer from the same problem.
It remains an interesting open problem to further improve the convergence rate of LocalMOML. 

\acks{This work is partially supported by NSF awards 2147253, 2110545, 1844403, and Amazon research
award.}	

\bibliography{all,draft}

	
	\newpage
	
	\appendix
		\appendix
	\clearpage 
	\begin{table}[htp]
		\centering
  		\caption{Notations we use throughout the paper.}
		\begin{tabular}{|c|c|c|}
			\hline 
			\multicolumn{3}{|c|}{{\bf Basic} } \\
			\hline 
			$n$ & total number of tasks/clients & \\
			\hline 
			$\T_i$ & the $i$-th task & \\
			\hline 
			$\w$ & the trainable parameter & \\
			\hline
			$\ell_i(\cdot,\z)$ & loss function with the data sample $\z$ of $i$-th task & \\
			\hline
			$\L_i(\cdot)$ & risk function for the $i$-th task & \\
			\hline
			$\D_i$ & distribution of data for the $i$-th task &\\
			\hline
			$F_i(\cdot)$ & the meta-objective of task $i$& Lemma~\ref{lem:ind_grad_to_total}\\
			\hline
			$F(\cdot)$ & the total meta-objective & Eq.~\eqref{eq:maml}\\
			\hline
			$\epsilon$ & target accuracy for an approximate stationary point & Table~\ref{tab:finite_n_comparison}\\
			\hline
			$L_i$, $\rho_i$ & Lipschitz constants of gradient $\nabla \L_i(\cdot)$ and Hessian $\nabla^2 \L_i(\cdot)$& \Cref{asm:smoothness}\\
			\hline
			$\sigma_G^2$, $\sigma_H^2$ & \makecell{variance upper bounds of stochastic gradient $\nabla \ell(\cdot,\z)$\\ and stochastic Hessian $\nabla^2 \ell(\cdot,\z)$} & \Cref{asm:bounded_var}\\
			\hline
			$\gamma_G^2$, $\gamma_H^2$ & \makecell{heterogeneity constants of gradients $\nabla \L_i(\cdot)$\\ and Hessian $\nabla^2 \L_i(\cdot)$ among the tasks/clients} & \makecell{\Cref{asm:bound_grad_dis}\\\Cref{asm:bounded_hes_dis}}\\
			\hline 
			$G$ & upper bound of the norm of gradient $\nabla\L_i(\cdot)$ & \Cref{asm:bounded_grad}\\
			\hline 
			\multicolumn{3}{|c|}{{\bf single-node Learning} } \\
			\hline %
			$B,K$ & batch size of tasks and batch size of data samples per batch & \\\hline
			$\B$ & size-$B$ batch of tasks & \\\hline
			$\T_i$ & the $i$-th task & \\
			\hline
			$\S_1^i,\S_2^i,\S_3^i$ & independent batches of data points  of task $i$& \\
			\hline 
			$\widehat{\nabla}_{\S} \L_i(\cdot)$ & stochastic gradient of $\L_i(\cdot)$ estimated on batch $\S$ & \\
			\hline 
			$\Deltah_{\B}$ & \makecell{stochastic estimator of meta-gradient \\estimated on a batch of tasks $\B$} & \eqref{eq:est}\\
			\hline
			$\u^i$ & personalized model for task $i$ &\\
			\hline 
			$\v_i(\w)$ & \makecell{updated model by one step of \\gradient descent $\v_i(\w) \coloneqq \w - \alpha \nabla \L_i(\w)$} & \\
			\hline 
			$\vhat^i$ & \makecell{updated model by one step of stochastic \\gradient descent $\vhat^i\coloneqq \w - \alpha \widehat{\nabla} \L_i(\w)$} & \\
			\hline
			$\eta_t$ & the step size in iteration $t$ & \\
			\hline 
			$\beta_t$ & the momentum factor in iteration $t$ & \\
			\hline
			$p_i$ & the probability of sampling task $i$ in $\B_t'$&\\
			\hline
			$C_p$ & the constant $\max_i 1/p_i - 1$ & \\
			\hline 
			\multicolumn{3}{|c|}{{\bf Federated Learning} } \\
			\hline 
			$\C_i$ & the $i$-th client & \\\hline
			$R,H$ & \makecell{total number of communication rounds\\ and total number of iterations in each round} & \\
			\hline 
			$\Delta\w_r$ & \makecell{the average of stochastic meta-gradients over \\the sampled clients $i\in\B_r$ and local steps of round $r$}&\\
			\hline 
			$\tilde{\eta}$ & \makecell{the effective stepsize per round, i.e., $\tilde{\eta}\coloneqq \eta H$}&\\
			\hline
		\end{tabular}
		\label{tab:notation}
	\end{table}
	\clearpage

	\begin{table}[t]
		\caption{Comparison of iteration complexities and required step sizes of algorithms under the single-node learning setting (supplementary to Table~\ref{tab:finite_n_comparison}). }
		\setlength\tabcolsep{3.5pt} 
		\centering
		\scalebox{0.85}{
			\begin{threeparttable}[b]
				\centering
				\begin{tabular}{cccccc}\toprule[.1em]
										\multicolumn{6}{c}{single-node Learning}\\\midrule
Algorithm & Stepsize   & \makecell{Iteration\\Complexity}  & \makecell{\#Task (B)\\Per Iteration}  & \makecell{\#Data points ($K$)\\Per Iteration} & \makecell{Sample\\Complexity}  \\\cmidrule(lr){1-6}
					MAML~\citep{fallah2020convergence} &  $\O(1)$  & $\O(\epsilon^{-2})$  & $n$ &    $\O(\epsilon^{-2})$ &    $\O\left(n\epsilon^{-4}\right)$        \\
					SCGD~\citep{wang2017stochastic}& $\O(\epsilon^{6})$ & $\O(\epsilon^{-8})$ & $n$ & $\O(1)$ & $\O(n\epsilon^{-8})$     \\
					NASA~\citep{Ghadimi2020AST} & $\O(\epsilon^{2})$   & $\O(\epsilon^{-4})$ & $n$ & $\O(1)$ & $\O(n\epsilon^{-4})$    \\
					BSGD~\citep{hu2020biased} & $\O(\epsilon^{2})$  & $\O(\epsilon^{-4})$ &	$\O(1)$ & $\O(\epsilon^{-2})$ & $\O(\epsilon^{-6})$   \\
					BSpiderBoost~\citep{hu2020biased}& $\O(1)$&$\O(\epsilon^{-2})$ & $\sqrt{n}$  & $\O(\epsilon^{-2})$  & $\O\left(n\epsilon^{-2} + \sqrt{n}\epsilon^{-4}\right)$   \\
					\CC{10} \textbf{\momlvo~(This work)} &\CC{10}$\O(\epsilon^{3})$ &\CC{10} $\O(n\epsilon^{-5})$&\CC{10}$ \bm{\O(1)}$  &\CC{10} $\bm{ \O(1)}$ &\CC{10} $\O\left(n\epsilon^{-5}\right)$    \\
					\CC{10} \textbf{\momlvs~(This work)} &\CC{10}$\O(\epsilon^{3})$ &\CC{10}$\O(n\epsilon^{-5})$&\CC{10} $ \bm{\O(1)}$ &\CC{10} $ \bm{\O(1)}$ &\CC{10} $\O\left(n \epsilon^{-5}\right)$   \\
					\bottomrule[.1em]
				\end{tabular}
		\end{threeparttable}}
		\label{tab:rate_and_lr}
	\end{table}

	\section{Preliminary Lemmas}
	\begin{lemma}[Lemma A.3 in \cite{fallah2020convergence}]
		If $\alpha \in[0,\frac{\sqrt{2}-1}{L})$ and Assumptions~\ref{asm:smoothness}, \ref{asm:bounded_var}, and \ref{asm:bound_grad_dis} are satisfied. Then, for any $\w\in\R^d$, we have
		\begin{align}\label{eq:grad_to_meta}
			& \Norm{\nabla \L(\w)}^2 \leq C_1^2 \Norm{\nabla F(\w)} + C_2^2\gamma_G^2,\\\label{eq:local_meta_to_glob}
			& \frac{1}{n}\sum_{i=1}^n\Norm{\nabla F_i(\w)}^2 \leq 2(1+\alpha L)^2 C_1^2 \Norm{\nabla F(\w)}^2 + (1+\alpha L)^2(2C_2^2 + 1)\gamma_G^2,
		\end{align}
		where $L \coloneqq \max_i L_i$, $\nabla \L(\w) \coloneqq \frac{1}{n}\sum_{i=1}^n\L_i(\w)$, $F_i(\w) \coloneqq \L_i(\w - \alpha \nabla \L_i(\w))$, $F(\w) \coloneqq \frac{1}{n}\sum_{i=1}^n F_i(\w)$, $C_1\coloneqq \frac{1}{1-2\alpha L - \alpha^2 L^2}$, and $C_2\coloneqq \frac{2\alpha L + \alpha^2L^2}{1-2\alpha L - \alpha^2 L^2}$.
		\label{lem:ind_grad_to_total}	
	\end{lemma}
	
	\begin{lemma}\label{lem:tau_nice}	
		Assume that $X = \frac{1}{n}\sum_{i=1}^n X_i$. If we sample a size-$B$ minibatch $\B$ from $\{1,\dotsc,n\}$ uniformly at random, we have $\E\left[\frac{1}{B}\sum_{i\in \B}X_i\right]  = X$ and 
		\begin{align}\label{eq:tau_nice}
\E\left[\Norm{\frac{1}{B}\sum_{i\in \B}X_i}^2\right]\leq\frac{n-B}{B(n-1)} \frac{1}{n}\sum_{i=1}^n \Norm{X_i}^2 + \frac{n(B-1)}{B(n-1)}\Norm{X}^2.
		\end{align}
	\end{lemma}	
	\begin{proof}
		First, we define random variables $\xi_i$ and $\xi_{ii'}$ as
		\begin{align*}
			\xi_i = \begin{cases}
				1 & i\in \B\\
				0 & i\notin \B
			\end{cases},\quad \xi_{ii'} = \begin{cases}
				1 & i\in\B~\text{and}~i'\in\B\\
				0 & \text{otherwise}
			\end{cases}
		\end{align*}
		Note that $\xi_{ii'} = \xi_i\xi_{i'}$, $\E\left[\xi_i\right] = \Pr(i\in\B)=\frac{B}{n}$ and $\E\left[\xi_{ii'}\right] = \Pr(i\in\B,i'\in\B) = \frac{B(B-1)}{n(n-1)}$. Thus, it is clear that
		\begin{align*}
			\E\left[\frac{1}{B}\sum_{i\in \B}X_i\right] = \E\left[\frac{1}{B}\sum_{i=1}^n \xi_i X_i\right] = \frac{1}{B}\sum_{i=1}^n \E\left[\xi_i\right] X_i = \frac{1}{B}\sum_{i=1}^n \Pr(i\in\B) X_i = \frac{1}{B}\sum_{i=1}^n \frac{B}{n} X_i = X.
		\end{align*}
		Moreover,
		\begin{align*}
			\E\left[\Norm{\frac{1}{B}\sum_{i\in \B}X_i}^2\right]&=\E\left[\Norm{\frac{1}{B}\sum_{i=1}^n \xi_i X_i}^2\right]= \frac{1}{B^2}\sum_{i=1}^n\E\left[\Norm{\xi_i X_i}^2\right] + \frac{1}{B^2}\sum_{i\neq i'} \E[\xi_i\xi_{i'}]\inner{X_i}{X_{i'}}\\
			& = \frac{1}{B^2}\sum_{i=1}^n \frac{B}{n}\Norm{X_i}^2 + \frac{1}{B^2}\sum_{i\neq i'} \frac{B(B-1)}{n(n-1)} \inner{X_i}{X_i'}\\
			& = \frac{n-B}{B(n-1)} \frac{1}{n}\sum_{i=1}^n \Norm{X_i}^2 + \frac{n(B-1)}{B(n-1)}\Norm{X}^2.
		\end{align*}
	\end{proof}
	
\begin{lemma}[Corollary A.1 of~\citealt{fallah2020convergence}]
	If $\alpha\in[0,1/L]$, for any $\w,\w'\in\R^d$, 
	\begin{align}\label{eq:pseudo_smoothness}
		F(\w') - F(\w) - \inner{\nabla F(\w)}{\w'-\w}\leq \frac{L(\w)}{2}\Norm{\w'-\w}^2,
	\end{align}
	where $L(\w) \coloneqq 4L + \frac{2\rho \alpha }{n}\sum_{i=1}^n \Norm{\nabla \L_i(\w)}$. If $\L_i(\cdot)$ is $G$-Lipschitz-continuous, that is, $\Norm{\nabla \L_i(\w)}\leq G$ for any $\w\in\R^d$, then $F(\cdot)$ is $L_F$-smooth with $L_F\coloneqq 4L + 2\rho \alpha G$. 
	\label{lem:smoothness}	
\end{lemma}

\begin{lemma}[Lemma 4.4 of~\citealt{fallah2020personalized}]
	If $\alpha\in(0,1/L]$ and Assumptions~\ref{asm:bound_grad_dis}, \ref{asm:bounded_grad}, and \ref{asm:bounded_hes_dis} are satisfied, we have
	\begin{align}\label{eq:meta_dissimilarity}
		\frac{1}{n}\sum_{i=1}^n \Norm{\nabla F_i(\w) - \nabla F(\w)}^2 \leq \gamma_F^2\coloneqq 3G^2\alpha^2\gamma_H^2 + 192\gamma_G^2.
	\end{align}
	\label{lem:gamma_F}
\end{lemma}

\begin{lemma}
	For $\widehat{L}(\w) \coloneqq  4L + \frac{2\rho \alpha}{B} \sum_{i\in\B}\Norm{\widehat{\nabla}_{\S}\L_i(\w)}$ and  $\frac{(n-1)B}{n-B}\geq \lceil (8\rho \alpha \gamma_G/L)^2\rceil$, $S\geq \lceil (8\rho \alpha\sigma_G/L)^2\rceil$, we have:
	\begin{align}\label{eq:stoc_lr_bounds}
		\frac{4}{5L(\w)}\leq\E\left[\frac{1}{\widehat{L}(\w) }\right] \leq \frac{11\eta_0}{8L(\w)} \leq \frac{\eta_0}{4L},\quad\quad  \frac{L(\w)}{2}\E\left[\frac{1}{\widehat{L}(\w)^2}\right] \leq \frac{2}{L(\w)}\leq \frac{1}{2L},
	\end{align} 
	where $L(\w)\coloneqq 4L + \frac{2\rho \alpha}{n}\sum_{i=1}^n \Norm{\nabla \L_i(\w)}$.
	\label{lem:step_size}	
\end{lemma}
\begin{proof}
	The results $\frac{4}{5L(\w)}\leq\E\left[\frac{1}{\widehat{L}(\w)}\right]$ and $\frac{L(\w)}{2}\E\left[\frac{1}{\widehat{L}(\w)^2}\right] \leq \frac{2}{L(\w)}\leq \frac{1}{2L}$ can be found in Lemma 5.9 of \cite{fallah2020convergence}. Now we prove the rest part. Utilize Theorem A.2 of \cite{fallah2020convergence} with $X=\frac{2\rho\alpha}{B'}\sum_{i\in\B'}\Norm{\widehat{\nabla}_{\S'} \L_i(\w)}$, $c=4L$, $k=1$:
	\begin{align*}
		L(\w)\E\left[\frac{1}{\widehat{L}(\w)}\right] \leq \frac{\sigma_X^2/4L + \mu_X^2 \frac{\mu_X}{\sigma_X^2 + \mu_X(\mu_X+c)}}{\sigma_X^2 + \mu_X^2}L(\w),
	\end{align*}
	where $\mu_X$ and $\sigma_X^2$ are the mean and variance of $X$. Consider that $\sigma_X^2 + \mu_X(\mu_X+4L)\geq \mu_X(\mu_X+4L)$ and $L(\w) \leq \mu_X + 5L$, which is shown in (60) of~\citet{fallah2020convergence}.
	\begin{align*}
		L(\w)\E\left[\frac{1}{\widehat{L}(\w)}\right] \leq \frac{\sigma_X^2/4L + \mu_X^2/(\mu_X+4L)}{\sigma_X^2 + \mu_X^2}(\mu_X+5L) = \frac{\frac{5}{4}\sigma_X^2 + \frac{\mu_X^2 (\mu_X+ 5L)}{\mu_X + 4L} + \frac{\sigma_X^2\mu_X}{4L}}{\sigma_X^2 + \mu_X^2}.	
	\end{align*}
	Note that $\frac{\mu_X+ 5L}{\mu_X+4L} = 1+ \frac{L}{\mu_X+4L}\leq \frac{5}{4}$ and $\sigma_X^2 \leq \frac{4\rho^2\alpha^2(n-B)}{B(n-1)}\left(\gamma_G^2 + \frac{\sigma_G^2}{S}\right)$
	\begin{align*}
		L(\w)\E\left[\frac{1}{\widehat{L}(\w)}\right] \leq \frac{5}{4} + \frac{\sigma_X^2\mu_X}{4L(\sigma_X^2 +\mu_X^2)} \leq \frac{5}{4} + \frac{\sigma_X^4/L^2 + \mu_X^2}{8(\sigma_X^2 +\mu_X^2)},
	\end{align*} 
	where the last inequality above uses Young's inequality. We only require $\sigma_X^2 \leq L^2$ to make $L(\w)\E\left[\frac{1}{\widehat{L}(\w)}\right] \leq \frac{11\eta_0}{8}$, which is satisfied if $\frac{(n-1)B}{n-B}\geq \lceil (8\rho \alpha \gamma_G/L)^2\rceil$, $S\geq \lceil (8\rho \alpha\sigma_G/L)^2\rceil$.	
\end{proof}		

\section{Convergence Analaysis of \momlvo}

\begin{lemma}\label{lem:bound_Delta}
	For the stochastic meta-gradient estimator $\Deltah_{\B_t}$ defined in \eqref{eq:moml_est}, it holds that
	\begin{align*}
		\E_{\S_2^i,\S_3^i}&\left[\Norm{\frac{1}{B}\sum_{i\in\B_t}(I-\alpha \nablah_{\S_2^i}^2 \L_i(\w_t))\nablah_{\S_3^i} \L_i(\u_{t+1}^i)}^2\right] \leq C_\Delta,
	\end{align*}
where $C_\Delta \coloneqq (\alpha^2\sigma_H^2/K + (1+\alpha L)^2)(\sigma_G^2/K+G^2)$.
\end{lemma}		
\begin{proof}
	Based on Assumption~\ref{asm:smoothness} and \ref{asm:bounded_grad}, we have
	\begin{small}
		\begin{align*}
			& \E_{\S_2^i}\left[\Norm{I-\alpha\nablah_{\S_2^i}^2\L_i(\w_t)}^2\right] = \alpha^2\E_{\S_2^i}\left[\Norm{\nablah_{\S_2^i}^2\L_i(\w_t) - \nabla^2 \L_i(\w_t)}^2\right] + (1+\alpha L)^2 \leq \frac{\alpha^2\sigma_H^2}{K} + (1+\alpha L)^2,\\			
			& \E_{\S_3^i}\left[ \Norm{\nablah_{\S_3^i}\L_i(\u_{t+1}^i)}^2\right]  = \E_{\S_3^i}\left[\Norm{\nablah_{\S_3^i}\L_i(\u_{t+1}^i) -\nabla \L_i(\u_{t+1}^i) }^2\right] + \Norm{\nabla \L_i(\u_{t+1}^i)}^2 \leq \frac{\sigma_G^2}{K}+G^2.
		\end{align*}
	\end{small}
	Consider that $\S_2^i$, $\S_3^i$ are mutually independent.
	\begin{small}
		\begin{align*}
			&\E_{\S_2^i,\S_3^i}\left[\Norm{\frac{1}{B}\sum_{i\in\B_t}(I-\alpha\nablah_{\S_2^i}^2\L_i(\w_t))\nablah_{\S_3^i}\L_i(\u_{t+1}^i)}^2\right] \\ &\leq  \frac{1}{B}\sum_{i\in\B_t}\E_{\S_2^i}\left[\Norm{I-\alpha\nablah_{\S_2^i}^2\L_i(\w_t)}^2\right]\E_{\S_3^i}\left[ \Norm{\nablah_{\S_3^i}\L_i(\u_{t+1}^i)}^2\right] \leq \left(\frac{\alpha^2\sigma_H^2}{S_2} + (1+\alpha L)^2\right)\left( \frac{\sigma_G^2}{S_3}+G^2\right).
		\end{align*}
	\end{small}
\end{proof}

\begin{proof}[Proof of Lemma~\ref{lem:moml_all_steps}]	Based on the $L_F$-smoothness of $F$, we have
	\begin{align*}
		F(\w_{t+1}) &\leq F(\w_t) + \inner{\nabla F(\w_t)}{\w_{t+1} - \w_t} + \frac{L_F}{2}\Norm{\w_{t+1} - \w_t}^2\\
		& = F(\w_t) - \eta \inner{F(\w_t)}{\frac{1}{B}\sum_{i\in\B_t} (I-\alpha \nablah_{\S_2^i}^2\L_i(\w_t))\nablah_{\S_3^i}\L_i(\u_{t+1}^i)} + \frac{\eta^2 L_F}{2}\Norm{\Deltah_{\B_t}}^2\\
		& = F(\w_t) - \eta\Norm{\nabla F(\w_t)}^2 + \eta\inner{\nabla F(\w_t)}{\nabla F(\w_t) - \Deltah_{\B_t}} +  \frac{\eta^2 L_F}{2}\Norm{\Deltah_{\B_t}}^2.
	\end{align*}
	Take expectation on both sides conditioned on $\F_t$, where $\F_t$ denotes all the randomness before the $t$-th iteration. 
	\begin{align}\nonumber
		\E\left[F(\w_{t+1})\mid \F_t\right]&\leq  F(\w_t) - \eta\Norm{\nabla F(\w_t)}^2 \\\label{eq:main_one_step}
		&  + \eta \inner{\nabla F(\w_t)}{\E\left[\nabla F(\w_t) - \Deltah_{\B_t}\mid \F_t\right]} + \frac{\eta^2 L_F}{2}\E\left[\Norm{\Deltah_{\B_t}}^2\mid \F_t \right].
	\end{align}
	Consider that $\E\left[\frac{1}{B}\sum_{i\in \B_t} (I-\alpha\nabla^2\L_i(\w_t))\nabla\L_i(\v_i(\w_t))\mid \F_t \right]  = \nabla F(\w_t)$.
	\begin{small}
		\begin{align*}
			& \E\left[\nabla F(\w_t) - \Deltah_{\B_t}\mid \F_t\right]\\
			& = \E\left[\frac{1}{B}\sum_{i\in\B_t} (I-\alpha\nabla^2\L_i(\w_t))\nabla\L_i(\v_i(\w_t)) - \frac{1}{B}\sum_{i\in\B_t}(I-\alpha \nablah_{\S_2^i}^2\L_i(\w_t))\nablah_{\S_3^i}\L_i(\u_{t+1}^i)\mid \F_t\right]\\
			&= \E\left[\frac{1}{B}\sum_{i\in\B_t} \left((I-\alpha\nabla^2\L_i(\w_t))\nabla\L_i(\v_i(\w_t)) -\E\left[ (I-\alpha\nablah_{\S_2^i}\L_i(\w_t))\nablah_{\S_3^i}\L_i(\u_{t+1}^i)\mid \F_t,\B_t\right]\right)\mid \F_t\right]\\
			&= \E\left[\frac{1}{B}\sum_{i\in\B_t} \left(I-\alpha\nabla^2\L_i(\w_t)\right)\left(\nabla \L_i (\v_i(\w_t)) - \nabla \L_i(\u_{t+1}^i)\right)\mid \F_t\right]. 
		\end{align*}
	\end{small}
	Then, Young's inequality implies that
	\begin{align}\nonumber
		& \inner{\nabla F(\w_t)}{\E\left[\nabla F(\w_t) - \Deltah_{\B_t}\mid \F_t\right]} \\\nonumber
		& = \E\left[\inner{\nabla F(\w_t)}{\frac{1}{B}\sum_{i\in\B_t} \left(I-\alpha\nabla^2\L_i(\w_t)\right)\left(\nabla \L_i (\v_i(\w_t)) - \nabla \L_i(\u_{t+1}^i)\right)}\mid \F_t \right]\\\label{eq:inner_prod}
		& \leq \frac{\Norm{\nabla F(\w_t)}^2}{2} + \frac{(1+\alpha L)^2L^2}{2}\E\left[\frac{1}{B}\sum_{i\in \B_t} \Norm{\v_i(\w_t) - \u_{t+1}^i}^2\mid \F_t\right].
	\end{align}
	Considering \eqref{eq:inner_prod} and setting $\alpha \leq 1/L$, the R.H.S. of \eqref{eq:main_one_step} can be upper bounded as
	\begin{small}
	\begin{align*}
		& \E\left[F(\w_{t+1})\mid \F_t\right]\\
		&\leq  F(\w_t) - \frac{\eta}{2}\Norm{\nabla F(\w_t)}^2 + \frac{\eta^2 L_F}{2}\E\left[\Norm{\widehat{\Delta}_{\B_t}}^2\mid \F_t\right] +   \frac{4 \eta L^2}{2}\E\left[\frac{1}{B}\sum_{i\in \B_t} \Norm{\v_i(\w_t) - \u_{t+1}^i}^2\mid \F_t\right].
	\end{align*}
\end{small}
	Use the tower property of conditional expectation, re-arrange the terms, and unwrap the recursion above from iteration $0$ to $T-1$
	\begin{align*}
		\E\left[\sum_{t=0}^{T-1}\Norm{\nabla F(\w_t)}^2\right] &\leq \frac{2F(\w_0)}{\eta} + \eta L_F  \sum_{t=0}^{T-1}\E\left[\Norm{\widehat{\Delta}_{\B_t}}^2\right]  + \frac{8L^2}{B} \E\left[\sum_{t=0}^{T-1}\sum_{i\in \B_t} \Norm{\v_i(\w_t) - \u_{t+1}^i}^2\right].
	\end{align*}
\end{proof}	

\begin{proof}[Proof of Lemma~\ref{lem:error_estimate}] Based on \eqref{eq:moml_v1}, the following equation holds
	\begin{align*}
		& \Norm{\v_i(\w_{t_k^i}) - \u_{t_{k_i}^i+1}^i}^2\\
		& =\Norm{\v_i(\w_{t_{k_i}^i}) - (1-\beta)\u_{t_{k-1}^i+1}^i - \beta\vhat_{t_k^i}^i}^2 \\
		& = \E\left[\Norm{(1-\beta)(\v_i(\w_{t_{k-1}^i})-\u_{t_{k-1}^i+1}^i) +(1-\beta)(\v_i(\w_{t_k^i}) - \v_i(\w_{t_{k-1}^i})) + \beta(\v_i(\w_{t_k^i}) - \vhat_{t_k^i}^i)}^2\right]
	\end{align*}
	We take expectation on both sides condition on $\F_{t_{k-1}^i}$
	\begin{align*}
		& \E\left[\Norm{\v_i(\w_{t_k^i}) - \u_{t_k^i+1}^i}^2\mid\F_{t_{k-1}^i} \right] \\
		&  = \beta^2 \alpha^2\E\left[\Norm{\nablah_{\S_1^i}\L_i(\w_{t_k^i}) - \nabla\L_i(\w_{t_k^i})}^2\mid\F_{t_{k-1}^i}\right] \\
		& \quad\quad + (1-\beta)^2 \E\left[\Norm{(\v_i(\w_{t_{k-1}^i}) - \u_{t_{k-1}^i+1}^i)  + (\v_i(\w_{t_k^i}) - \v_i(\w_{t_{k-1}^i}))}^2\mid\F_{t_{k-1}^i}\right]\\
		& \leq \frac{\beta^2\alpha^2\sigma_G^2}{|\S_1^i|} + (1-\beta)\Norm{\v_i(\w_{t_{k-1}^i})-\u_{t_{k-1}^i+1}^i}^2 \\
		& \quad\quad + (1-\beta)^2(1+1/\beta)(1+\alpha L)^2\E\left[\Norm{\w_{t_k^i} - \w_{t_{k-1}^i}}^2\mid\F_{t_{k-1}^i}\right],\\
		& \leq \frac{\beta^2\alpha^2\sigma_G^2}{|\S_1^i|} + (1-\beta)\Norm{\v_i(\w_{t_{k-1}^i})-\u_{t_{k-1}^i+1}^i}^2 + 8\eta^2\E\left[\Norm{\sum_{\tau=t_{k-1}^i}^{t_k^i-1} \Deltah_{\B_\tau}}^2\mid\F_{t_{k-1}^i}\right]\mathbb{I}[\beta\in(0,1)],
	\end{align*}
	where $t_{k-1}^i$ is the latest iteration before $t_k^i$ that task $\T_i$ is also sampled, in other words, $t_{k-1}^i = \max \{\tau\mid \tau\in\T_i\land \tau<t_k^i\}$.  Lemma~\ref{lem:bound_Delta} implies that
	\begin{align*}
		\E\left[\Norm{\sum_{\tau=t_{k-1}^i}^{t_k^i-1} \Deltah_{\B_\tau}}^2\mid\F_{t_{k-1}^i}\right] & \leq \E\left[(t_k^i - t_{k-1}^i) \sum_{\tau=t_{k-1}^i}^{t_k^i-1} \E\left[\Norm{\Deltah_{\B_\tau}}^2\mid \B_\tau, \F_{t_{k-1}^i}\right]\mid \F_{t_{k-1}^i}\right]\\
		& \leq \E\left[(t_k^i - t_{k-1}^i)^2\mid \F_{t_{k-1}^i}\right]C_\Delta.
	\end{align*}
	It is worth noting that $t_k^i - t_{k-1}^i$ follows the geometric distribution. Thus, the second moment satisfies $\E\left[(t_k^i - t_{k-1}^i)^2\mid \F_{t_{k-1}^i}\right]\leq \frac{2n^2}{B^2}$. Then,
	\begin{small}
		\begin{align*}
			& \E\left[\Norm{\v_i(\w_{t_k^i}) - \u_{t_k^i+1}^i}^2\mid\F_{t_{k-1}^i} \right]\\
			&  \quad\quad\quad \leq (1-\beta)\Norm{\v_i(\w_{t_{k-1}^i})-\u_{t_{k-1}^i+1}^i}^2  + \frac{16\eta^2 n^2 C_\Delta}{\beta B^2}\mathbb{I}[\beta\in(0,1)] + \frac{\beta^2\alpha^2\sigma_G^2}{|\S_1^i|}.
		\end{align*}
	\end{small}
	Re-arranging the terms, summing over $k=1,\dotsc, T_i$, and using the tower property of conditional expectation leads to
	\begin{align*}
		& \E\left[\sum_{k=0}^{T_i-1} \Norm{\v_i(\w_k^i)-\u_{t_k^i + 1}^i}^2\right]	\\
		& \leq \left(\frac{\E\left[\Norm{\v_i(\w_{t_0^i}) - \u_{t_0^i+1}^i}^2\right]}{\beta} + \frac{16\eta^2 n^2 C_\Delta}{\beta^2 B^2} \E\left[T_i\right]\right)\mathbb{I}[\beta\in(0,1)] + \frac{\beta \alpha^2 \sigma_G^2}{|\S_1^i|}\E\left[T_i\right].
	\end{align*}
	Initializing the personalized model $\u^i$ as $\u_{t_0^i+1}^i = \w_{t_0^i} - \alpha \nablah_{\S_1^i}\L_i(\w_{t_0^i})$ and summing over $i=1,\dotsc, n$ results in
	\begin{align*}
		& \E\left[\sum_{i=1}^n\sum_{k=0}^{T_i-1} \Norm{\v_i(\w_k^i)-\u_{t_k^i + 1}^i}^2\right]	\\
		& \leq \frac{n\sigma_G^2}{\beta |\S_1^i|}\mathbb{I}[\beta\in(0,1)]  + \left(\frac{16\eta^2 n^2 C_\Delta}{\beta^2 B^2}\mathbb{I}[\beta\in(0,1)] + \frac{\beta \alpha^2 \sigma_G^2}{|\S_1^i|}\right)\E\left[\sum_{i=1}^nT_i\right].
	\end{align*}
	Note that $\sum_{i=1}^nT_i = T$ based on the definition. 
\end{proof}		

\begin{theorem}[Detailed Version of Theorem~\ref{thm:moml_v1_informal}]
Under Assumptions~\ref{asm:smoothness},~\ref{asm:bounded_var}~\ref{asm:bounded_below},~\ref{asm:bounded_grad}, \emph{\momlvo} with stepsizes $\eta_t =\frac{B^{2/5}}{n^{2/5}T^{3/5}}$, $\beta_t =\frac{n^{2/5}}{B^{2/5}T^{2/5}}<1$ and constant batch sizes $|\S_1^i|=|\S_2^i|=|\S_3^i| = K =1$, $|\B_t| = B = 1$ can find a stationary point $\w_\tau$ in $T=\O(n\epsilon^{-5})$ iterations.	
\end{theorem}

\begin{proof}
Based on Lemma~\ref{lem:moml_all_steps} and Lemma~\ref{lem:error_estimate}, we have 
\begin{align*}
	\frac{1}{T}\sum_{t=0}^{T-1} \E\left[\Norm{\nabla F(\w_t)}^2\right] & \leq \frac{2F(\w_0)}{\eta T} + \eta L_F C_\Delta + \frac{8L^2}{B}\left(\frac{n\sigma_G^2}{\beta KT} + \frac{16\eta^2 n^2 C_\Delta}{\beta^2B^2} + \frac{\beta \alpha^2 \sigma_G^2}{K}\right).
\end{align*}	
Choosing $B=1$, $K=1$, $\eta = \frac{B^{2/5}}{n^{2/5}T^{3/5}}$ and $\beta = \frac{n^{2/5}}{B^{2/5}T^{2/5}}<1$ leads to
\begin{small}
\begin{align*}
		\frac{1}{T}\sum_{t=0}^{T-1} \E\left[\Norm{\nabla F(\w_t)}^2\right] & \leq \frac{2n^{2/5}F(\w_0)}{T^{2/5}} + \frac{L_FC_\Delta}{n^{2/5}T^{3/5}} + \frac{8L^2\sigma_G^2n^{3/5}}{T^{3/5}} + \frac{144L^2C_\Delta n^{2/5}}{T^{2/5}} + \frac{8L^2\alpha^2 \sigma_G^2n^{2/5}}{T^{2/5}}.
\end{align*}
\end{small} 
For $\w_\tau$ and $\tau$ is sampled from $0,\dotsc, T-1$ uniformly at random, $\E\left[\Norm{\nabla F(\w_\tau)}^2\right]\leq \frac{1}{T}\sum_{t=0}^{T-1} \E\left[\Norm{\nabla F(\w_t)}^2\right]$. Making the R.H.S. of the upper bound of $\frac{1}{T}\sum_{t=0}^{T-1} \E\left[\Norm{\nabla F(\w_t)}^2\right]$ be equal to or smaller than $\epsilon^2$ finishes the proof.
\end{proof}		

\section{Convergence Analysis of \momlvs}

\begin{lemma}\label{lem:stoc_grad_norm}
	For the stochastic estimator $\widehat{\Delta}_{\B_t} =\frac{1}{B}\sum_{i\in\B_t}(I-\alpha\nabla_{\S_2^i}^2\L_i(\w_t))\nabla_{\S_3^i}\L_i(\u_{t+1}^i)$ of the meta-gradient, we have
	\begin{align}\label{eq:stoc_grad_norm}
		\E\left[\Norm{\widehat{\Delta}_{\B_t}}^2\mid\F_t\right] \leq C_3\frac{1}{n}\sum_{i=1}^n \E\left[\Norm{\u_{t+1}^i -\v_i(\w_t)}^2\mid \F_t\right] + C_4 \Norm{\nabla F(\w_t)}^2 + C_5,
	\end{align}
	where $\F_t$ denotes all randomness occurred before the $t$-th iteration, and the constants are defined as $C_3 \coloneqq 2L^2 \left((1+\alpha L)^2 + \frac{\alpha^2\sigma_H^2}{|\S_2^i|}\right)$, $C_4\coloneqq  4\left((1+\alpha L)^2 + \frac{\alpha^2\sigma_H^2}{|\S_2^i|}\right)(1+\alpha L)^2 C_1 $, and $C_5\coloneqq \left((1+\alpha L)^2 + \frac{\alpha^2\sigma_H^2}{|\S_2^i|}\right)\left(\frac{\sigma_G^2}{|\S_3^i|}+2(1+\alpha L)^2(2C_2^2 + 1)\gamma_G^2\right)$.
\end{lemma}
\begin{proof}
	The definition of the stochastic estimator $\widehat{\Delta}_{\B_t}$ implies that
	\begin{align*}
		& \E\left[\Norm{\widehat{\Delta}_{\B_t}}^2\mid\F_t\right]  =  \E\left[\Norm{\frac{1}{B}\sum_{i\in\B_t}(I-\alpha\nabla_{\S_2^i}^2\L_i(\w_t))\nabla_{\S_3^i}\L_i(\u_{t+1}^i) }^2\mid\F_t\right] \\
		&  = \E\left[\Norm{\frac{1}{B}\sum_{i\in\B_t }\left((I-\alpha\nabla_{\S_2^i}^2\L_i(\w_t))\nabla_{\S_3^i}\L_i(\u_{t+1}^i) - (I-\alpha\nabla_{\S_2^i}^2\L_i(\w_t))\nabla \L_i(\u_{t+1}^i)\right)}^2\mid\F_t\right] \\
		& \quad\quad\quad + \E\left[\Norm{\frac{1}{B}\sum_{i\in\B_t}\left((I-\alpha\nabla_{\S_2^i}^2\L_i(\w_t))\nabla\L_i(\u_{t+1}^i)- (I-\alpha\nabla^2\L_i(\w_t))\nabla\L_i(\u_{t+1}^i)\right)}^2\mid\F_t\right] \\
		& \quad\quad\quad + \E\left[\Norm{\frac{1}{B}\sum_{i\in\B_t}(I-\alpha\nabla^2\L_i(\w_t))\nabla\L_i(\u_{t+1}^i)}^2\mid\F_t\right]\\
		& \leq \left((1+\alpha L)^2 + \frac{\alpha^2\sigma_H^2}{|\S_2^i|}\right)\left(\frac{\sigma_G^2}{|\S_3^i|} +  \frac{1}{n}\sum_{i=1}^n\E\left[\Norm{\nabla\L_i(\u_{t+1}^i)}^2\mid\F_t\right] \right) \\
		& \leq \left((1+\alpha L)^2 + \frac{\alpha^2\sigma_H^2}{|\S_2^i|}\right)\left(\frac{\sigma_G^2}{|\S_3^i|} +  \frac{2}{n}\sum_{i=1}^n \Norm{\nabla F_i(\w_t))}^2+ \frac{2L^2 }{n}\sum_{i=1}^n \E\left[\Norm{\u_{t+1}^i - \v_i(\w_t)}^2\mid \F_t\right] \right), 
	\end{align*}
	where the last inequality uses the fact $\nabla \L_i(\v_i(\w_t)) = \nabla F_i(\w_t)$. Lemma~\ref{lem:ind_grad_to_total} shows that
	\begin{align*}
		\frac{1}{n}\sum_{i=1}^n\Norm{\nabla F_i(\w_t)}^2 \leq 2(1+\alpha L)^2 C_1 \Norm{\nabla F(\w_t)}^2 + (1+\alpha L)^2 (2C_2^2 + 1)\gamma_G^2. 
	\end{align*}
	Then, 
	\begin{align*}
		\E\left[\Norm{\widehat{\Delta}_{\B_t}}^2\mid\F_t\right] 
		& = \left((1+\alpha L)^2 + \frac{\alpha^2\sigma_H^2}{|\S_2^i|}\right)\left(\frac{\sigma_G^2}{|\S_3^i|}+2(1+\alpha L)^2(2C_2^2 + 1)\gamma_G^2\right)\\
		& + 4\left((1+\alpha L)^2 + \frac{\alpha^2\sigma_H^2}{|\S_2^i|}\right)(1+\alpha L)^2 C_1 \Norm{\nabla F(\w_t)}^2 \\
		& +2L^2 \left((1+\alpha L)^2 + \frac{\alpha^2\sigma_H^2}{|\S_2^i|}\right)\frac{1}{n}\sum_{i=1}^n \E\left[\Norm{\u_{t+1}^i - \v_i(\w_t)}^2\mid \F_t\right].
	\end{align*}
\end{proof}	 

Apart from the notations in Table~\ref{tab:notation}, we define that $\I_i \coloneqq \left(0_{d\times d},\dotsc, I_{d\times d},\dotsc, 0_{d\times d}\right)^\top\in\R^{nd\times d}$ (where the $i$-th block in $\I_i$ is an identity matrix while the others are zeros), $\bar{\w}_t \coloneqq \left(\w_t^\top,\dotsc,\w_t^\top\right)^\top\in\R^{nd}$, $\u_t = \left([\u_t^1]^\top,\dotsc,[\u_t^n]^\top\right)^\top\in\R^{nd}$, $\widehat{\bg}_t \coloneqq \sum_{i\in\B_t'} \frac{1}{\alpha p_i} \I_i \left(\w_t - \vhat_t^i)\right)$, $\bar{\bg}_t \coloneqq \sum_{i=1}^n \I_i(\w_t - \vhat_t^i))/\alpha$, $\bg_t \coloneqq \sum_{i=1}^n \I_i(\w_t - \v_i(\w_t))/\alpha$, $\widetilde{\w}_t \coloneqq \bar{\w}_t - \alpha \bg_t$. Then, we can re-write the update rule of the personalized models $\u_t^i$ for all tasks $i\in[n]$ in a more succinct expression $\u_{t+1} = (1-\beta_t)\u_t + \beta_t\left(\bar{\w}_t -\alpha \widehat{\bg}_t \right)$.

\begin{lemma}\label{lem:last_term}
	For $\widehat{\bg}_t \coloneqq \sum_{i\in\B_t'} \frac{1}{\alpha p_i} \I_i \left(\w_t - \vhat_t^i)\right)$ and $\bg_t \coloneqq \sum_{i=1}^n \I_i(\w_t - \v_i(\w_t))/\alpha$, we have
	\begin{align}\nonumber
		\E\left[\Norm{\bg_t - \widehat{\bg}_t}^2\mid \F_t \right] &\leq 2nC_p C_1^2\Norm{\nabla F(\w_{t-1})}^2  \\\label{eq:last_term}
		& + nC_p\left((2C_2^2+1)\gamma_G^2 + \frac{2\sigma_G^2}{|\S_1^i|}\right) + 2nLC_p \Norm{\w_t-\w_{t-1}}^2,
	\end{align}
	where $C_p\coloneqq  \max_i\left(\frac{1}{p_i}-1\right)$.
\end{lemma}
\begin{proof}
	Consider that $\widehat{\bg}_t \coloneqq \sum_{i\in\B_t'} \frac{1}{\alpha p_i} \I_i \left(\w_t - \vhat_t^i)\right)$, $\bar{\bg}_t \coloneqq \sum_{i=1}^n \I_i(\w_t - \vhat_t^i))/\alpha$, $\bg_t \coloneqq \sum_{i=1}^n \I_i(\w_t - \v_i(\w_t)))/\alpha$.
	\begin{align}\label{eq:g_decomp}
		\E\left[\Norm{\bg_t - \widehat{\bg}_t}^2\mid \F_t\right] =  \E\left[\Norm{\widehat{\bg}_t - \bar{\bg}_t}^2\mid \F_t\right] +  \E\left[\Norm{\bg_t - \bar{\bg}_t}^2\mid \F_t\right].
	\end{align}
The first term on the right hand side of \eqref{eq:g_decomp} can be upper bounded as
\begin{align}\label{eq:remove_rc}
	\E\left[\Norm{\widehat{\bg}_t - \bar{\bg}_t}^2\mid \F_t\right] & \leq C_p \E\left[\Norm{\bar{\bg}_t}^2\mid \F_t\right] = C_p \left(\Norm{\bg_t}^2 + \E\left[\Norm{\bg_t - \bar{\bg}_t}^2\mid \F_t\right]\right),
\end{align}	
where we define $C_p\coloneqq  \max_i\left(\frac{1}{p_i}-1\right)$. Note that $\Norm{\bg_t}^2 = \sum_{i=1}^n \Norm{\nabla \L_i(\w_t)}^2$. Then,
\begin{align}\nonumber
	\Norm{\bg_t}^2&= n \Norm{\nabla \L(\w_t)}^2 + n\gamma_G^2 \leq 2n\Norm{\nabla \L(\w_{t-1})}^2 + 2nL\Norm{\w_t - \w_{t-1}}^2 + n\gamma_G^2\\\label{eq:bound_g}
	& \leq 2nC_1^2\Norm{\nabla F(\w_{t-1})}^2 + n(2C_2^2+1)\gamma_G^2 + 2nL\Norm{\w_t - \w_{t-1}}^2.
\end{align}
The last inequality above utilizes Lemma~\ref{lem:ind_grad_to_total}. Besides,
\begin{align}\label{eq:g_var}
	\E\left[\Norm{\bg_t - \bar{\bg}_t}^2\mid \F_t\right] & =\sum_{i=1}^n\E\left[\Norm{\widehat{\nabla}_{\S_1^i}\L_i(\w_t) - \nabla \L_i(\w_t)}^2\mid \F_t \right] \leq \frac{n\sigma_G^2}{|\S_1^i|}.
\end{align}
According to (\ref{eq:g_decomp}), (\ref{eq:remove_rc}), (\ref{eq:bound_g}), and (\ref{eq:g_var}), we have
\begin{align*}
	\E\left[\Norm{\bg_t - \widehat{\bg}_t}^2\mid \F_t \right] &\leq 2nC_p C_1^2\Norm{\nabla F(\w_{t-1})}^2  \\
	& + nC_p\left((2C_2^2+1)\gamma_G^2 + \frac{2\sigma_G^2}{|\S_1^i|}\right) + 2nLC_p \Norm{\w_t-\w_{t-1}}^2.
\end{align*}
\end{proof}

\begin{proof}[Proof of Lemma~\ref{lem:fval_recursion}] Recall that the notations $\v_i(\w_t) \coloneqq \w_t - \alpha \nabla\L_i(\w_t)$, $\bar{\w}_t \coloneqq \left(\w_t^\top,\dotsc,\w_t^\top\right)^\top\in\R^{nd}$, $\u_t = \left([\u_t^1]^\top,\dotsc,[\u_t^n]^\top\right)^\top$, $\widehat{\bg}_t \coloneqq \sum_{i\in\B_t'} \frac{1}{\alpha p_i} \I_i \left(\w_t - \vhat_t^i)\right)$, $\bg_t \coloneqq \sum_{i=1}^n \I_i(\w_t - \v_i(\w_t))/\alpha$, $\widetilde{\w}_t \coloneqq \bar{\w}_t - \alpha \bg_t$. The update rule of the personalized models is $\u_{t+1} = (1-\beta_t)\u_t + \beta_t\left(\bar{\w}_t -\alpha \widehat{\bg}_t \right)$. We define that $\Upsilon_t\coloneqq \frac{1}{n}\sum_{i=1}^n\Norm{\u_{t+1}^i - \v_i(\w_t)}^2$.
\begin{small} 
	\begin{align*}
		&\E\left[\Upsilon_{t+1}\mid \F_{t+1} \right] = \frac{1}{n}\E\left[\Norm{\u_{t+2} - \widetilde{\w}_{t+1}}^2\mid \F_{t+1}\right]\\
		& =\frac{1}{n}\E\left[\Norm{\widetilde{\w}_{t+1} - (1-\beta_{t+1})\u_{t+1} + \beta_{t+1} \left(\bar{\w}_{t+1} - \alpha\widehat{\bg}_{t+1}\right)}^2\mid \F_{t+1} \right]\\
		& = \frac{1}{n}\E\left[\Norm{(1-\beta_{t+1})(\widetilde{\w}_t - \u_{t+1}) + (1-\beta_{t+1})\left(\widetilde{\w}_{t+1} - \widetilde{\w}_t\right) + \alpha \beta_{t+1}(\bg_{t+1}-\widehat{\bg}_{t+1})}^2\mid \F_{t+1} \right]\\
		& \leq (1-\beta_{t+1})\frac{1}{n}\Norm{\widetilde{\w}_t - \u_{t+1}}^2 + \frac{8(1+\alpha L)^2}{\beta_{t+1} n}\Norm{\bar{\w}_{t+1} - \bar{\w}_t}^2 + \frac{\alpha^2\beta_{t+1}^2}{n} \E\left[\Norm{\bg_{t+1} - \widehat{\bg}_{t+1}}^2\mid \F_{t+1} \right]\\
		& = (1-\beta_{t+1})\frac{1}{n}\sum_{i=1}^n\Norm{\u_{t+1}^i - \v_i(\w_t)}^2 + \frac{8(1+\alpha L)^2}{\beta_{t+1}}\Norm{\w_{t+1} - \w_t}^2 + \frac{\alpha^2\beta_{t+1}^2}{n} \E\left[\Norm{\bg_{t+1} - \widehat{\bg}_{t+1}}^2\mid \F_{t+1} \right].
	\end{align*}
\end{small} 
Based on Lemma~\ref{lem:last_term}, we have
\begin{align*}
	& \E\left[\Upsilon_{t+1}\mid \F_{t+1} \right] \leq (1-\beta_{t+1})\Upsilon_t + 2\left(\frac{4(1+\alpha L)^2}{\beta_{t+1}} + \beta_{t+1}^2\alpha^2LC_p\right)\eta_t^2\Norm{\Delta_{\B_t}}^2 \\
	&\quad\quad\quad\quad\quad\quad + 2\beta_{t+1}^2\alpha^2 C_pC_1^2\Norm{\nabla F(\w_t)}^2 + \beta_{t+1}^2\alpha^2C_p(2C_2^2 + 1)\gamma_G^2 + \frac{2\beta_{t+1}^2\alpha^2C_p\sigma_G^2}{|\S_1^i|}.
\end{align*}
We choose $\beta_{t+1} = 6L^2\eta_0^{-1/3}\eta_t$. Since we need to ensure $\beta_t\leq 1$ for any $t$, we only need to maintain $\eta_0\leq \left(\frac{2}{3L}\right)^{\frac{3}{2}}$. Then, $\beta_{t+1}^2 = 36L^4 \eta_0^{-2/3} \eta_t^2 \leq 3 L^2\eta_0^{\frac{4}{3}}$.
\begin{align*}
\E\left[\Upsilon_{t+1}\mid \F_{t+1} \right]  \leq (1-6L^2\eta_0^{-1/3}\eta_t)\Upsilon_t + C_6\eta_0^{1/3}\eta_t\Norm{\Delta_{\B_t}}^2 +\eta_0^{1/3}\eta_t C_7\Norm{\nabla F(\w_t)}^2 + \eta_0^{4/3}C_8,
\end{align*}	
where $C_6\coloneqq \frac{4(1+\alpha L)^2 + \alpha^2LC_p}{3L^2}$, $C_7\coloneqq 18L^3\alpha^2 C_pC_1^2$, $C_8\coloneqq 3L^2\left(\alpha^2C_p(2C_2^2 + 1)\gamma_G^2 + \frac{2\alpha^2C_p\sigma_G^2}{|\S_1^i|}\right)$. Then, in view of the tower property of conditional expectation, we have
\begin{align}\nonumber
	& \E\left[\Upsilon_{t+1}\mid \F_t\right]  \leq (1-6L^2\eta_0^{-1/3}\E\left[\eta_t\mid \F_t\right])\E\left[\Upsilon_t\mid \F_t\right] \\\nonumber
	& \quad\quad\quad +\eta_0^{1/3} \E\left[\eta_t\mid \F_t\right]C_7\Norm{\nabla F(\w_t)}^2 + \eta_0^{4/3}C_8 \\\nonumber
	& \quad\quad\quad + C_6\eta_0^{1/3}\E\left[\eta_t\mid \F_t\right]\left(C_3 \frac{1}{n}\sum_{i=1}^n \E\left[\Upsilon_t\mid \F_t\right] + C_4 \Norm{\nabla F(\w_t)}^2 + C_5\right)\\\nonumber
	& =(1-6L^2\eta_0^{-1/3}(1-C_3C_6\eta_0^{2/3}/6L^2)\E\left[\eta_t\mid \F_t\right])\E\left[\Upsilon_t\mid \F_t\right] \\\nonumber
	& \quad\quad\quad + \eta_0^{1/3}\E\left[\eta_t\mid \F_t\right]\left(C_7 + C_4C_6\right)\Norm{\nabla F(\w_t)}^2+ \eta_0^{4/3}C_8 + \eta_0^{1/3}\E\left[\eta_t\mid \F_t\right]C_5C_6\\\nonumber
	& \leq \left(1-3L^2 \eta_0^{-1/3}\E\left[\eta_t\mid \F_t\right]\right) \E\left[\Upsilon_t\mid \F_t\right]\\\label{eq:fval_recursion}
	& \quad\quad\quad + \eta_0^{1/3}\E\left[\eta_t\mid \F_t\right]\left(C_7 + C_4C_6\right)\Norm{\nabla F(\w_t)}^2 + \eta_0^{4/3}\left(C_8 + \frac{C_5C_6}{4L}\right),
\end{align}
where the last step holds when $\eta_0\leq \left(\frac{3L^2}{C_3C_6}\right)^{3/2}$. We define $C_9\coloneqq C_7 + C_4C_6$, $C_{10} \coloneqq C_8 + C_5C_6/(4L)$. 
\end{proof}	

\begin{lemma}\label{lem:main_bd}
	If we set $\eta_0\leq \min\left\{\frac{2L^2}{5C_3}, \frac{1}{8C_9^{3/2}}, \frac{1}{20C_4}\right\}$ and define the potential function $\Phi_t$ as $\Phi_t\coloneqq \eta_0^{1/3}\frac{1}{n}\sum_{i=1}^n\Norm{\u_{t+1}^i - \v_i(\w_t)}^2 + F(\w_t)$, we have
	\begin{align*}
		\E\left[\Phi_{t+1}\right] &\leq \E\left[\Phi_t\right] - \frac{\eta_0}{80}\min\left\{\frac{\E\left[\Norm{\nabla F(\w_t)}^2\right]}{L + \rho\alpha \sigma}, \frac{\E\left[\Norm{\nabla F(\w_t)}\right]}{\rho\alpha}\right\} + \eta_0^{5/3}\left(\frac{\eta_0^{1/3}C_5}{2} +C_{10}\right).
	\end{align*}	
\end{lemma}

\begin{proof}
	Based on Lemma~\ref{lem:smoothness}, we have
	\begin{align*}
		F(\w_{t+1}) & \leq F(\w_t) + \inner{\nabla F(\w_t)}{\w_{t+1} - \w_t} + \frac{L(\w_t)}{2}\Norm{\w_{t+1} - \w_t}^2\\
		& = F(\w_t) - \eta_t \inner{F(\w_t)}{\widehat{\Delta}_{\B_t}} + \frac{\eta_t^2 L(\w_t)}{2}\Norm{\widehat{\Delta}_{\B_t}}^2\\
		& = F(\w_t) - \eta_t \Norm{\nabla F(\w_t)}^2 + \eta_t \inner{\nabla F(\w_t)}{\nabla F(\w_t) - \widehat{\Delta}_{\B_t}} + \frac{\eta_t^2 L(\w_t)}{2}\Norm{\widehat{\Delta}_{\B_t}}^2.
	\end{align*}
Consider the step size $\eta_t$ and $\widehat{\Delta}_{\B_t}$ are independent. Take expectation on both sides conditioned on $\F_t$, where $\F_t$ denotes all randomness occurred before the $t$-th iteration.\begin{align*}
	& \E\left[F(\w_{t+1})\mid \F_t\right] \\
	&\leq F(\w_t) - \E\left[\eta_t\mid \F_t\right] \Norm{\nabla F(\w_t)}^2 \\
	& \quad\quad\quad + \E\left[\eta_t\mid \F_t\right]  \inner{\nabla F(\w_t)}{\E\left[\nabla F(\w_t) - \widehat{\Delta}_{\B_t}\mid \F_t \right]} + \frac{\E\left[\eta_t^2\mid \F_t\right]  L(\w_t)}{2}\E\left[\Norm{\widehat{\Delta}_{\B_t}}^2\mid \F_t \right]
\end{align*}
Consider the fact $\E\left[\widehat{\Delta}_{\B_t}\mid \F_t \right] = \frac{1}{n}\sum_{i=1}^n\E\left[(I-\alpha \nabla^2\L_i(\w_t))\nabla \L_i(\u_{t+1}^i)\mid \F_t \right]$.
\begin{align*}
	\E\left[\nabla F(\w_t) - \widehat{\Delta}_{\B_t}\mid \F_t \right] & = \frac{1}{n}\sum_{i=1}^n \E\left[(I-\alpha \nabla^2\L_i(\w_t))(\nabla \L_i(\v_i(\w_t)) - \nabla \L_i(\u_{t+1}^i))\mid \F_t\right]
\end{align*}
Since $\Norm{\cdot}$ is a convex function, we have the following equation based on the Jensen's and Cauchy-Schwarz inequalities
\begin{align*}
	& \inner{\nabla F(\w_t)}{\E\left[\nabla F(\w_t) - \widehat{\Delta}_{\B_t}\mid \F_t \right]}  \leq \Norm{\nabla F(\w_t)}\Norm{\E\left[\nabla F(\w_t) - \widehat{\Delta}_{\B_t}\mid \F_t \right]}\\
	& \leq \frac{\Norm{\nabla F(\w_t)}^2}{2} + \frac{\Norm{\E\left[\nabla F(\w_t) - \widehat{\Delta}_{\B_t}\mid \F_t \right]}^2}{2}\\
	& \leq \frac{\Norm{\nabla F(\w_t)}^2}{2} + \frac{(1+\alpha L)^2 L^2 \frac{1}{n}\sum_{i=1}^n \E\left[\Norm{\u_{t+1}^i - \v_i(\w_t)}^2\mid \F_t \right]}{2}\\
	& \leq \frac{\Norm{\nabla F(\w_t)}^2}{2} + \frac{4 L^2 \frac{1}{n}\sum_{i=1}^n \E\left[\Norm{\u_{t+1}^i - \v_i(\w_t)}^2\mid \F_t \right]}{2},
\end{align*}
where the last inequality holds when $\alpha \leq 1/L$. Thus,
\begin{align}\nonumber
	& \E\left[F(\w_{t+1})\mid \F_t \right] \\ \nonumber
	& \leq F(\w_t) - \frac{\E\left[\eta_t\mid \F_t\right]}{2}\Norm{\nabla F(\w_t)}^2 + \frac{4L^2 \E\left[\eta_t\mid \F_t\right]}{2}\frac{1}{n}\sum_{i=1}^n \E\left[\Norm{\u_{t+1}^i - \v_i(\w_t)}^2\mid \F_t \right] \\\nonumber
	& \quad\quad\quad + \frac{L(\w_t)\E\left[\eta_t^2\mid \F_t\right]}{2}\E\left[\Norm{\widehat{\Delta}_{\B_t}}^2\mid \F_t \right]\\\nonumber
	& \leq F(\w_t) - \frac{\left(\E\left[\eta_t\mid \F_t\right]-C_4L(\w_t)\E\left[\eta_t^2\mid \F_t\right]\right)}{2}\Norm{\nabla F(\w_t)}^2 + \frac{C_5 L(\w_t)\E\left[\eta_t^2\mid \F_t\right]}{2} \\\label{eq:starter}
	& \quad\quad\quad + \frac{\left(4L^2\E\left[\eta_t\mid \F_t\right]  + L(\w_t) \E\left[\eta_t^2\mid \F_t\right]C_3 \right)}{2}\frac{1}{n}\sum_{i=1}^n \E\left[\Norm{\u_{t+1}^i - \v_i(\w_t)}^2\mid \F_t \right].
\end{align}
Based on Lemma~\ref{lem:step_size}, we can derive that
\begin{align*}
	-3L^2\E\left[\eta_t\mid \F_t \right] + \frac{\left(4L^2\E\left[\eta_t\mid \F_t\right]  + L(\w_t) \E\left[\eta_t^2\mid \F_t\right]C_3 \right)}{2} = -\frac{4L^2\eta_0}{5L(\w_t)} + \frac{2C_3\eta_0^2}{L(\w_t)}\leq 0,
\end{align*}
where we need $\eta_0\leq \frac{2L^2}{5C_3}$. Besides, if $\eta_0 \leq \frac{1}{8C_9^{3/2}}$ and $\eta_0 \leq \frac{1}{20C_4}$, we have
\begin{align*}
	& C_9 \eta_0^{2/3} \E\left[\eta_t\mid \F_t\right] -\frac{\left(\E\left[\eta_t\mid \F_t\right]-C_4L(\w_t)\E\left[\eta_t^2\mid \F_t\right]\right)}{2} \leq - \frac{\eta_0}{5L(\w_t)} + \frac{2\eta_0^2C_4}{L(\w_t)}\leq - \frac{\eta_0}{10L(\w_t)}.
\end{align*}
Multiplying (\ref{eq:fval_recursion}) by $\eta_0^{1/3}$ and summing it to (\ref{eq:starter}) leads to
\begin{small}
	\begin{align*}
		& \eta_0^{1/3}\frac{1}{n}\sum_{i=1}^n\E\left[\Norm{\u_{t+2}^i - \v_i(\w_{t+1}) }^2\right] + \E\left[F(\w_{t+1})\right]\\
		& \leq \eta_0^{1/3}\frac{1}{n}\sum_{i=1}^n \E\left[\Norm{\u_{t+1}^i - \v_i(\w_t)}^2 \right] + \E\left[F(\w_t)\right] - \E\left[\frac{\eta_0}{10L(\w_t)} \Norm{\nabla F(\w_t)}^2\right] + \eta_0^{5/3}\left(\eta_0^{1/3}C_5/2 + C_{10}\right).
	\end{align*}
\end{small} 
Define that $\Phi_t\coloneqq \eta_0^{1/3}\frac{1}{n}\sum_{i=1}^n\Norm{\u_{t+1}^i -\v_i(\w_t)}^2 + F(\w_t)$. Besides, utilize (103) $\sim$ (106) in \cite{fallah2020convergence}:
\begin{align*}
	\E\left[\Phi_{t+1}\right] & \leq \E\left[\Phi_t\right] - \frac{\eta_0}{80}\min\left\{\frac{\E\left[\Norm{\nabla F(\w_t)}^2\right]}{L + \rho\alpha \sigma}, \frac{\E\left[\Norm{\nabla F(\w_t)}\right]}{\rho\alpha}\right\}+ \eta_0^{5/3}\left(\eta_0^{1/3}C_5/2 + C_{10}\right).
\end{align*}	
\end{proof}

\begin{theorem}[Detailed Version of Thoerem~\ref{thm:moml_v2_informal}]
	Under Assumptions~\ref{asm:smoothness}, \ref{asm:bounded_var}, and \ref{asm:bound_grad_dis}, it is guaranteed that \emph{\momlvs} can find an $\epsilon$-stationary point in $\frac{160(L+\rho\alpha(\sigma+\epsilon)) \Phi_0}{C_{11}\epsilon^5}$ iterations, where $C_{11} = \Omega(1/C_p)$, $C_p = \max_i 1/p_i - 1$.
\end{theorem}

\begin{proof}
	Based on Lemma~\ref{lem:main_bd}, we can derive that:
	\begin{align*}
		\E\left[\Phi_{t+1}\right] & \leq \E\left[\Phi_t\right] - \frac{\eta_0}{80}\min\left\{\frac{\E\left[\Norm{\nabla F(\w_t)}^2\right]}{L + \rho\alpha \sigma}, \frac{\E\left[\Norm{\nabla F(\w_t)}\right]}{\rho\alpha}\right\} + \eta_0^{5/3}\left(\eta_0^{1/3}C_5/2 +C_{10}\right).
	\end{align*}
	Suppose that $\E\left[\Norm{\nabla F(\w_t)}\right]\geq \epsilon$ and $\E\left[\Norm{\nabla F(\w_t)}^2\right]\geq \left(\E\left[\Norm{\nabla F(\w_t)}\right]\right)^2 \geq \epsilon^2$, $\forall t\in 1,\dotsc, T$. Otherwise, we can find an $\epsilon$-stationary point in the first $T$ iterations. Thus, choosing $\eta_0 = C_{11}\epsilon^3$ ($C_{11}>0$) and telescoping over the $T$ iterations leads to:
	\begin{align*}
		T \frac{\epsilon^2}{L+\rho\alpha (\sigma +\epsilon) }& \leq T\min\left\{\frac{\epsilon^2}{L+\rho\alpha \sigma}, \frac{\epsilon}{\rho\alpha}\right\} \leq \frac{80\Phi_0}{C_{11}\epsilon^3} + 80 T C_{11}^{2/3}\epsilon^2 \left(C_9^{1/3}\epsilon C_5/2 + C_{10}\right),
	\end{align*}
	where $C_{11} \coloneqq \min\left\{\frac{1}{160C_5(L+\rho \alpha (\sigma+\epsilon))}, \frac{1}{\left(320C_9(L+\rho\alpha(\sigma+\epsilon))\right)^{3/2}}\right\}$. Note that $\eta_0=C_{11}\epsilon^3$ can satisfy the requirements on $\eta_0$ in Lemma~\ref{lem:fval_recursion} and Lemma~\ref{lem:main_bd}. Thus, we can find at least one $\epsilon$-stationary point if $T\geq \frac{160(L+\rho\alpha(\sigma+\epsilon)) \Phi_0}{C_{11}\epsilon^5}$. 
\end{proof}

\section{Convergence Analysis of  LocalMOML}\label{sec:lems_lcmoml}

To tackle with partial client sampling, we follow the ideas of  \cite{lixiang2019convergence, karimireddy2020scaffold}: in each round, the global model $\w_r$ is sent to the sampled clients $i\in\B_r$ and the sampled clients run local step. Here we assume that $\w_r$ is also \emph{virtually} sent to the other clients $i\not\in\B_r$ and those clients also  \emph{virtually} run local step.  After $H$ iterations, only $\Delta \w_r^i$, $i\in\B_r$ are aggregated to compute $\w_{r+1}$. It is worth noting that the extra communication of $\w_r$ to clients $i\not\in \B_r$ and the local steps on clients $i\not\in \B_r$ are only used in the proof and not actually executed when running \Cref{alg:fed_moml}. The lemma below is key to our analysis.

\begin{lemma}\label{lem:one_round}
	After one round of \emph{LocalMOML}, it satisfies that:
	\begin{align}\nonumber
		& \E\left[F(\w_{r+1})\right]  \leq \E\left[F(\w_r)\right] - \frac{\tilde{\eta}}{2}\left(1-8\tilde{\eta} - \frac{8\tilde{\eta}(n-B)}{B(n-1)}\right)\E\left[\Norm{\nabla F(\w_r)}^2\right] \\\nonumber
		& \quad\quad + \frac{\tilde{\eta}}{2}\left(C_{\rho,G,L} + 8\tilde{\eta} L_F^2 + \frac{8\tilde{\eta}(n-B)L_F^2}{B(n-1)}\right) \frac{1}{nH}\sum_{i=1}^n \sum_{h=1}^H \E\left[\Norm{\w_{r,h}^i - \w_r}^2\right]\\\label{eq:one_round}
		& \quad\quad + 8\tilde{\eta} \left(1+  L^2 \tilde{\eta}\right)\frac{1}{nH}\sum_{i=1}^n\sum_{h=1}^H\E\left[\Norm{\u_{r,h}^i - \v_i(\w_{r,h}^i)}^2\right] + \frac{4(n-B)}{B(n-1)}\tilde{\eta}^2 \gamma_F^2 + \frac{\tilde{\eta}^2\hat{\sigma}^2}{BH}.
	\end{align}
	where $\hat{\sigma}^2\coloneqq \frac{2\sigma_G^2}{S_3} + \frac{2\alpha^2\sigma_G^2}{S_3}\left(\frac{\sigma_H^2}{S_2} + L^2\right) + \frac{\alpha^2 G^2\sigma_H^2}{S_2}$, $\tilde{\eta} = \eta H$, $C_{\rho,G,L}\coloneqq 2 G^2 \rho^2/L^2 +16 L^2 $.
\end{lemma}	
\begin{proof}
	Based on the smoothness of $F$ shown in  Lemma~\ref{lem:smoothness}, we have
	\begin{align*}
		F(\w_{r+1}) & \leq F(\w_r) -\inner{\nabla F(\w_r)}{\Delta \w_r} + \frac{L_F}{2}\Norm{\Delta\w_r}^2.
	\end{align*}
	Let $\tilde{\eta} \coloneqq H \eta$ and $\F_r$ denotes all randomness occurred before the communication round $r$. Note that $\Delta \w_r = \frac{\tilde{\eta}}{BH}\sum_{i\in\B_r}\sum_{h=1}^H\Deltah_{r,h}^i$ and $\E\left[\Delta\w_r \right] = \frac{\tilde{\eta}}{nH}\sum_{i=1}^n\sum_{h=1}^H\E\left[\Deltah_{r,h}^i\right]$ because $\Deltah_{r,h}^i$ does not depend on the client sampling $\B_r$.
	\begin{align}\label{eq:starter_lcmoml}
		\E\left[F(\w_{r+1})\mid \F_r\right] &\leq F(\w_r)- \tilde{\eta}  \frac{1}{nH}\sum_{i=1}^n\sum_{h=1}^H \inner{\nabla F(\w_r)}{\E\left[  \Deltah_{r,h}^i \mid \F_r\right]} + \E\left[\Norm{\Delta \w_r}^2\mid \F_r\right].
	\end{align}
	The second term on the right hand side can be decomposed as
	\begin{align*}
		& - \tilde{\eta} \frac{1}{nH}\sum_{i=1}^n\sum_{h=1}^H \inner{\nabla F(\w_r)}{\E\left[  \Deltah_{r,h}^i \mid \F_r\right]} \\
		& =  - \tilde{\eta} \frac{1}{nH}\sum_{i=1}^n\sum_{h=1}^H \inner{\nabla F(\w_r)}{\E\left[  \Deltah_{r,h}^i  - \nabla F(\w_r)\mid \F_r\right]}  - \tilde{\eta}  \Norm{\nabla F(\w_r)}^2.
	\end{align*}
	Based on the definition of $\Deltah_{r,h}^i$, we have
	\begin{small}
		\begin{align*}
			&  - \tilde{\eta}  \frac{1}{nH}\sum_{i=1}^n \sum_{h=1}^H \inner{\nabla F(\w_r)}{\E\left[  \Deltah_{r,h}^i  - \nabla F(\w_r)\mid \F_r\right]} \\
			& = - \tilde{\eta}  \E\left[  \inner{\nabla F(\w_r)}{ \frac{1}{nH}\sum_{i=1}^n\sum_{h=1}^H \left((I - \alpha \nabla^2\L_i(\w_{r,h}^i)) \nabla\L_i(\u_{r,h}^i) - (I-\alpha \nabla^2 \L_i(\w_r))\nabla \L_i(\v_i(\w_r))\right)}\mid \F_r\right] \\
			& \leq \frac{\tilde{\eta}}{2}\Norm{\nabla F(\w_r)}^2 \\
			&\quad\quad + \frac{\tilde{\eta}}{2} \E\left[\Norm{\frac{1}{nH}\sum_{i=1}^n\sum_{h=1}^H \left((I - \alpha \nabla^2\L_i(\w_{r,h}^i)) \nabla\L_i(\u_{r,h}^i) - (I-\alpha \nabla^2 \L_i(\w_r))\nabla \L_i(\v_i(\w_r))\right)}^2\mid \F_r \right].
		\end{align*}
	\end{small}
	The second term on the right hand side can be upper bounded as
	\begin{small}
		\begin{align*}
			&	\E\left[\Norm{\frac{1}{nH}\sum_{i=1}^n\sum_{h=1}^H \left((I - \alpha \nabla^2\L_i(\w_{r,h}^i)) \nabla\L_i(\u_{r,h}^i) - (I-\alpha \nabla^2 \L_i(\w_r))\nabla\L_i(\v_i(\w_r))\right)}^2\mid \F_r \right]\\
			& \leq 2\E\left[\Norm{\frac{1}{nH}\sum_{i=1}^n\sum_{h=1}^H \left((I - \alpha \nabla^2\L_i(\w_{r,h}^i)) \nabla\L_i(\u_{r,h}^i) - (I-\alpha \nabla^2 \L_i(\w_r))\nabla \L_i(\u_{r,h}^i)\right)}^2\mid \F_r \right]\\
			& \quad + 2\E\left[\Norm{\frac{1}{nH}\sum_{i=1}^n\sum_{h=1}^H \left((I - \alpha \nabla^2\L_i(\w_r)) \nabla\L_i(\u_{r,h}^i) - (I-\alpha \nabla^2 \L_i(\w_r))\nabla \L_i(\v_i(\w_r))\right)}^2\mid \F_r \right]\\
			& \leq 2\alpha^2\frac{1}{nH}\sum_{i=1}^n\sum_{h=1}^H \E\left[\Norm{\nabla^2\L_i (\w_{r,h}^i) - \nabla^2\L_i (\w_r)}^2\Norm{\nabla \L_i(\u_{r,h}^i)}^2\mid \F_r\right]\\
			& \quad + 4 \frac{1}{nH }\sum_{i=1}^n\sum_{h=1}^H \E\left[\Norm{\nabla\L_i(\u_{r,h}^i) -\nabla \L_i( \v_i(\w_{r,h}^i))}^2 \Norm{I-\alpha \nabla \L_i(\w_r)}^2\mid \F_r\right]  \\
			&\quad + 4 \frac{1}{nH }\sum_{i=1}^n\sum_{h=1}^H \E\left[\Norm{\nabla\L_i(\v_i(\w_{r,h}^i)) - \nabla \L_i(\v_i(\w_r))}^2 \Norm{I-\alpha \nabla \L_i(\w_r)}^2\mid \F_r\right].
		\end{align*}
	\end{small}
	Based on Assumption~\ref{asm:smoothness},~\ref{asm:bounded_grad}, we have
	\begin{small}
		\begin{align*}
			&	\E\left[\Norm{\frac{1}{nH}\sum_{i=1}^n\sum_{h=1}^H \left((I - \alpha \nabla^2\L_i(\w_{r,h}^i)) \nabla\L_i(\u_{r,h}^i) - (I-\alpha \nabla^2 \L_i(\w_r))\nabla\L_i(\w_r)\right)}^2\mid \F_r \right]\\
			& \leq 2\alpha^2 G^2 \rho^2 \frac{1}{nH}\sum_{i=1}^n\sum_{h=1}^H \E\left[\Norm{\w_{r,h}^i - \w_r}^2\mid \F_r\right] + 4(1+\alpha L)^2\frac{1}{nH}\sum_{i=1}^n\sum_{h=1}^H\E\left[\Norm{\u_{r,h}^i - \v_i(\w_{r,h}^i)}^2\mid \F_r \right]\\
			& \quad\quad + 4(1+\alpha L)^2 L^2 \frac{1}{nH}\sum_{i=1}^n\sum_{h=1}^H \E\left[\Norm{\w_{r,h}^i - \w_r}^2\mid \F_r\right]\\
			& =C_{\rho,G,L} \frac{1}{nH}\sum_{i=1}^n\sum_{h=1}^H \E\left[\Norm{\w_{r,h}^i - \w_r}^2\mid \F_r\right] + 16 \frac{1}{nH}\sum_{i=1}^n\sum_{h=1}^H\E\left[\Norm{\u_{r,h}^i - \v_i(\w_{r,h}^i)}^2\mid \F_r \right],
		\end{align*}
	\end{small}
	where $C_{\rho,G,L}\coloneqq 2\alpha^2 G^2 \rho^2 + 4(1+\alpha L)^2 L^2\leq 2 G^2 \rho^2/L^2 +16 L^2 $ when $\alpha \leq 1/L$. Besides, the last term on the right hand side of \eqref{eq:starter_lcmoml} can be upper bounded as
	\begin{align}\nonumber
		& \E\left[\Norm{\frac{1}{BH}\sum_{i\in\B_r}\sum_{h=1}^H \Deltah_{r,h}^i}^2\mid \F_r\right] \\\nonumber
		&\leq  \E\left[\frac{1}{B^2H^2}\sum_{i\in\B_r}\sum_{h=1}^H\E_{\S_2^i,\S_3^i}\left[\Norm{ \Deltah_{r,h}^i - (I-\alpha\nabla^2 \L_i(\w_{r,h}^i))\nabla \L_i(\u_{r,h}^i) }^2\right]\mid \F_r\right]\\\label{eq:decompose_lcmoml_stoc_grad}
		& +  \E\left[\Norm{\frac{1}{BH}\sum_{i\in\B_r}\sum_{h=1}^H (I-\alpha\nabla^2 \L_i(\w_{r,h}^i))\nabla\L_i(\u_{r,h}^i)}^2\right],
	\end{align}
	which is due to $\Deltah_{r,h}^i = (I-\alpha \nablah_{\S_2^i}^2\L_i (\w_{r,h}^i))\nablah_{\S_3^i}\L_i(\u_{r,h}^i)$ and 
	\begin{align*}
		\E_{\S_2^i,\S_3^i}\left[ (I-\alpha \nablah_{\S_2^i}^2\L_i (\w_{r,h}^i))\nablah_{\S_3^i}\L_i(\u_{r,h}^i) - (I-\alpha\nabla^2 \L_i(\w_{r,h}^i))\nabla \L_i(\u_{r,h}^i) \right] = 0.
	\end{align*}
	Next, we can decompose $\Deltah_{r,h}^i - (I-\alpha \nabla^2 \L_i(\w_r))\nabla\L_i(\w_r)$ as
	\begin{align*}
		&(I-\alpha \nablah_{\S_2^i}^2\L_i (\w_{r,h}^i))\nablah_{\S_3^i}\L_i(\u_{r,h}^i) - (I-\alpha\nabla^2 \L_i(\w_{r,h}^i))\nabla \L_i(\u_{r,h}^i)  \\
		&= \left(\nablah_{\S_3^i}\L_i(\u_{r,h}^i) - \nabla \L_i(\u_{r,h}^i )\right) - \alpha \left(\nablah_{\S_2^i}^2 \L_i(\w_{r,h}^i) \nablah_{\S_3^i} \L_i (\u_{r,h}^i) -\nablah_{\S_2^i}^2 \L_i(\w_{r,h}^i) \nabla \L_i (\u_{r,h}^i)\right) \\
		& \quad\quad - \alpha \left(\nablah_{\S_2^i}^2 \L_i(\w_{r,h}^i) \nabla \L_i (\u_{r,h}^i)-\nabla^2 \L_i(\w_{r,h}^i) \nabla \L_i (\u_{r,h}^i)\right),
	\end{align*}
	Then, the term $\E_{\S_2^i,\S_3^i}\left[\Norm{ \Deltah_{r,h}^i - (I-\alpha\nabla^2 \L_i(\w_{r,h}^i))\nabla \L_i(\u_{r,h}^i) }^2\right]$ can be upper bounded as
	\begin{small}
		\begin{align*}
			& \E_{\S_2^i,\S_3^i}\left[\Norm{ \Deltah_{r,h}^i - (I-\alpha\nabla^2 \L_i(\w_{r,h}^i))\nabla \L_i(\u_{r,h}^i) }^2\right] \\
			& =\E_{\S_2^i,\S_3^i}\left[\left\| \nablah_{\S_3^i}\L_i(\u_{r,h}^i) - \nabla \L_i(\u_{r,h}^i)  +  \alpha (\nablah_{\S_2^i}^2 \L_i(\w_{r,h}^i) \nablah_{\S_3^i} \L_i (\u_{r,h}^i) -\nablah_{\S_2^i}^2 \L_i(\w_{r,h}^i) \nabla \L_i (\u_{r,h}^i))\right\|^2\right]  \\
			& \quad\quad + \alpha^2 \E_{\S_2^i}\left[\Norm{\left(\nabla^2\L_i(\w_{r,h}^i)- \nablah_{\S_2^i}^2 \L_i(\w_{r,h}^i)\right) \nabla\L_i(\u_{r,h}^i)}^2\right]\\
			&\leq  2\E_{\S_3^i}\left[\Norm{\nablah_{\S_3^i}\L_i(\u_{r,h}^i) - \nabla \L_i(\u_{r,h}^i )}^2\right] + \alpha^2 G^2 \E_{\S_2^i}\left[\Norm{\nabla^2\L_i(\w_{r,h}^i)- \nablah_{\S_2^i}^2 \L_i(\w_{r,h}^i)}^2\right]\\
			& \quad\quad + 2\alpha^2 \E_{\S_2^i}\left[\Norm{\nablah_{\S_2^i}^2 \L_i(\w_{r,h}^i)}^2\right]\E_{\S_3^i}\left[\Norm{\nablah_{\S_3^i} \L_i (\u_{r,h}^i) - \nabla \L_i (\u_{r,h}^i)}^2\right]\\
			& \leq \frac{2\sigma_G^2}{|\S_3^i|} + \frac{2\alpha^2 \sigma_G^2}{|\S_3^i|}\E_{\S_2^i}\left[\Norm{\nablah_{\S_2^i}^2\L_i(\w_{r,h}^i) - \nabla^2\L_i(\w_{r,h}^i)}^2 + L^2 \right]+ \frac{\alpha^2G^2\sigma_H^2}{|\S_2^i|}\\
			& \leq \frac{2\sigma_G^2}{|\S_3^i|} + \frac{2\alpha^2\sigma_G^2}{|\S_3^i|}\left(\frac{\sigma_H^2}{|\S_2^i|} + L^2\right)+ \frac{\alpha^2G^2\sigma_H^2}{|\S_2^i|} \coloneqq \hat{\sigma}^2.
		\end{align*}
	\end{small}
	Besides, the last term on the right hand side of  \eqref{eq:decompose_lcmoml_stoc_grad} can be upper bounded as
	\begin{align*}
		& \E\left[\Norm{\frac{1}{BH}\sum_{i\in\B_r}\sum_{h=1}^H (I-\alpha\nabla^2 \L_i(\w_{r,h}^i))\nabla\L_i(\u_{r,h}^i)}^2\right]\\
		&  \leq 2\E\left[\Norm{\frac{1}{BH}\sum_{i\in\B_r}\sum_{h=1}^H (I-\alpha\nabla^2 \L_i(\w_{r,h}^i)) \nabla \L_i(\v_i(\w_{r,h}^i))}^2\mid \F_r\right]\\
		& \quad\quad + 2\E\left[\Norm{\frac{1}{BH}\sum_{i\in\B_r}\sum_{h=1}^H (I-\alpha\nabla^2 \L_i(\w_{r,h}^i)) \left(\nabla \L_i(\v_i(\w_{r,h}^i)) - \nabla \L_i(\u_{r,h}^i)\right)}^2\mid \F_r\right].
	\end{align*}
	Apply Lemma~\ref{lem:tau_nice} to the first term on the right hand side:
	\begin{align*}
		& \E\left[\Norm{\frac{1}{BH}\sum_{i\in\B_r}\sum_{h=1}^H (I-\alpha\nabla^2 \L_i(\w_{r,h}^i)) \nabla \L_i(\v_i(\w_{r,h}^i))}^2\mid \F_r\right]\\
		&\leq  \frac{(n-B)}{B(n-1)}\frac{1}{nH}\sum_{i=1}^n\sum_{h=1}^H  \E\left[\Norm{(I-\alpha\nabla^2 \L_i(\w_{r,h}^i)) \nabla \L_i(\v_i(\w_{r,h}^i))}^2\mid \F_r \right]\\
		& + \E\left[\Norm{\frac{1}{n}\sum_{i=1}^n\sum_{h=1}^H (1-\alpha \nabla^2 \L_i(\w_{r,h}^i))\nabla \L_i(\u_{r,h}^i)}^2\mid \F_r \right].
	\end{align*}
	Then, we have
	\begin{align*}
		& \E\left[\Norm{\frac{1}{BH}\sum_{i\in\B_r}\sum_{h=1}^H (I-\alpha\nabla^2 \L_i(\w_{r,h}^i))\nabla\L_i(\u_{r,h}^i)}^2\right]\\
		& \leq \frac{2(n-B)}{B(n-1)}\frac{1}{nH}\sum_{i=1}^N\sum_{h=1}^H  \E\left[\Norm{(I-\alpha\nabla^2 \L_i(\w_{r,h}^i)) \nabla \L_i(\v_i(\w_{r,h}^i))}^2\mid \F_r \right]+4\Norm{\nabla F(\w_r)}^2 \\
		& \quad\quad  +4 E\left[\Norm{\frac{1}{nH}\sum_{i=1}^n\sum_{h=1}^H \left(\nabla F_i(\w_{r,h}^i) - \nabla F_i(\w_r)\right)}^2\mid \F_r \right] \\
		& \quad\quad + \frac{2(1+\alpha L)^2L^2}{nH}\sum_{i=1}^n\sum_{h=1}^H \E\left[\Norm{\v_i(\w_{r,h}^i) - \u_{r,h}^i}^2\mid \F_r\right].
	\end{align*}
	Based on Lemma~\ref{lem:gamma_F}, we have
	\begin{align*}
		& \frac{1}{nH} \sum_{i=1}^n\sum_{h=1}^H \mathbb{E}\left[\left\|\left(I-\alpha\nabla^2\mathcal{L}_i(\mathbf{w}_{r,h}^i)\right)\nabla \mathcal{L}_i(\mathbf{v}_i(\mathbf{w}_{r,h}^i))\right\|^2\right]  = \frac{1}{nH} \sum_{i=1}^n\sum_{h=1}^H \mathbb{E}\left[\|\nabla F_i (\mathbf{w}_{r,h}^i)\|^2\right]\\
		& \leq \frac{2}{nH} \sum_{i=1}^n\sum_{h=1}^H \mathbb{E}\left[\|\nabla F_i (\mathbf{w}_r)\|^2\right] +  \frac{2L_F^2}{nH} \sum_{i=1}^n\sum_{h=1}^H \mathbb{E}\left[\|\mathbf{w}_{r,h}^i - \mathbf{w}_r\|^2\right]\\
		& \leq 2\gamma_F^2 + 2\mathbb{E}\left[\|\nabla F(\mathbf{w}_r)\|^2\right]+\frac{2L_F^2}{nH} \sum_{i=1}^n\sum_{h=1}^H\mathbb{E}\left[\|\mathbf{w}_{r,h}^i - \mathbf{w}_r\|^2\right],
	\end{align*}
	Put them together and use the tower property of expectation on both sides. 
	\begin{align*}
		& \E\left[F(\w_{r+1})\right]  \\
		&\leq \E\left[F(\w_r)\right] - \frac{\tilde{\eta}}{2}\left(1-8\tilde{\eta} - \frac{8\tilde{\eta}(n-B)}{B(n-1)}\right)\E\left[\Norm{\nabla F(\w_r)}^2\right] \\
		& \quad\quad + \frac{\tilde{\eta}}{2}\left(C_{\rho,G,L} + 8\tilde{\eta} L_F^2 + \frac{8\tilde{\eta}(n-B)L_F^2}{B(n-1)}\right) \frac{1}{nH}\sum_{i=1}^n \sum_{h=1}^H \E\left[\Norm{\w_{r,h}^i - \w_r}^2\right]\\
		& \quad\quad + 8\tilde{\eta} \left(1+  L^2 \tilde{\eta}\right)\frac{1}{nH}\sum_{i=1}^n\sum_{h=1}^H\E\left[\Norm{\u_{r,h}^i - \v_i(\w_{r,h}^i)}^2\right] + \frac{4(n-B)}{B(n-1)}\tilde{\eta}^2 \gamma_F^2 + \frac{\tilde{\eta}^2\hat{\sigma}^2}{BH}.
	\end{align*}
\end{proof}	

\begin{lemma}\label{lem:local_drift}
	If $\eta \leq \frac{1}{2 H L_F}$, it is satisfied that:
	\begin{align}\nonumber
		&	\frac{1}{nH}\sum_{i=1}^n\sum_{h=1}^H\mathbb{E}\left[\|\mathbf{w}_{r,h}^i - \mathbf{w}_r\|^2\right]\\\nonumber
		& \leq 16\eta^2 H (H-1) \left(\hat{\sigma}^2 + 2\gamma_F^2\right) + 32\eta^2H^2L^2 \left(\frac{1}{nH}\sum_{i=1}^n\sum_{h=1}^H\mathbb{E}\left[\|\mathbf{u}_{r,h}^i - \mathbf{v}_i(\mathbf{w}_{r,h}^i)\|^2\right]\right) \\\label{eq:local_drift}
		&\quad\quad\quad  + 32\eta^2 H (H-1)  \mathbb{E}\left[\|\nabla F(\mathbf{w}_r)\|^2\right].
	\end{align}
\end{lemma}	

\begin{proof}
	For the stochastic estimator $\Deltah_{r,h}^i \coloneqq (I-\alpha \nablah_{\S_2^i}^2\L_i(\w_{r,h}^i))\nablah_{\S_3^i}\L_i(\u_{r,h}^i)$ of the meta-gradient and $\Delta_{r,h}^i \coloneqq (I-\alpha \nabla^2\L_i(\w_{r,h}^i))\nabla\L_i(\u_{r,h}^i)$, we have
	\begin{align*}
		\E\left[\Deltah_{r,h}^i - \Delta_{r,h}^i\right] = \E\left[\E_{\S_2^i,\S_3^i}\left[\Deltah_{r,h}^i - \Delta_{r,h}^i\right]\right] =0.
	\end{align*}
	Then, the local drift at iteration $h+1$, round $r$ can be upper bounded as
	\begin{align*}
		& \mathbb{E}\left[\|\mathbf{w}_{r,h+1}^i - \mathbf{w}_r\|^2\right] \\
		& \leq \left(1+\frac{1}{H}\right)\mathbb{E}\left[\|\mathbf{w}_{r,h}^i -\mathbf{w}_r\|^2\right] + \eta^2(1+H)\mathbb{E}\left[\|(I-\alpha \widehat{\nabla}_{\mathcal{S}_2}^2\mathcal{L}_i(\mathbf{w}_{r,h}^i))\widehat{\nabla}_{\mathcal{S}_3}\mathcal{L}_i(\mathbf{u}_{r,h}^i)\|^2\right]\\
		& = \left(1+\frac{1}{H}\right)\mathbb{E}\left[\|\mathbf{w}_{r,h}^i -\mathbf{w}_r\|^2\right] + \eta^2(1+H)\mathbb{E}\left[\|(I-\alpha \nabla^2\mathcal{L}_i(\mathbf{w}_{r,h}^i))\nabla\mathcal{L}_i(\mathbf{u}_{r,h}^i)\|^2\right]\\
		& \quad\quad + \eta^2(1+H)\mathbb{E}\left[\|(I-\alpha \widehat{\nabla}_{\mathcal{S}_2}^2\mathcal{L}_i(\mathbf{w}_{r,h}^i))\widehat{\nabla}_{\mathcal{S}_3}\mathcal{L}_i(\mathbf{u}_{r,h}^i)-(I-\alpha \nabla^2\mathcal{L}_i(\mathbf{w}_{r,h}^i))\nabla\mathcal{L}_i(\mathbf{u}_{r,h}^i)\|^2\right]\\
		& \leq \left(1+\frac{1}{H}\right)\mathbb{E}\left[\|\mathbf{w}_{r,h}^i -\mathbf{w}_r\|^2\right] + 2\eta^2H\mathbb{E}\left[\|(I-\alpha \nabla^2\mathcal{L}_i(\mathbf{w}_{r,h}^i))\nabla\mathcal{L}_i(\mathbf{u}_{r,h}^i)\|^2\right] + 2\eta^2H\hat{\sigma}^2.
	\end{align*}
	Next, we upper bound $\mathbb{E}\left[\|(I-\alpha \nabla^2\mathcal{L}_i(\mathbf{w}_{r,h}^i))\nabla\mathcal{L}_i(\mathbf{u}_{r,h}^i)\|^2\right]$. Note that we define $\mathbf{v}_i(\mathbf{w}_{r,h}^i) \coloneqq \mathbf{w}_{r,h}^i - \alpha \nabla \mathcal{L}_i(\mathbf{w}_{r,h}^i)$ and $\nabla F_i (\mathbf{w}_{r,h}^i) \coloneqq \left(I-\alpha\nabla^2\mathcal{L}_i(\mathbf{w}_{r,h}^i)\right)\nabla \mathcal{L}_i(\mathbf{v}_i(\mathbf{w}_{r,h}^i))$.
	\begin{align*}
		& \mathbb{E}\left[\|(I-\alpha \nabla^2\mathcal{L}_i(\mathbf{w}_{r,h}^i))\nabla\mathcal{L}_i(\mathbf{u}_{r,h}^i)\|^2\right] \\
		& \leq \mathbb{E}\left[\|(I-\alpha \nabla^2\mathcal{L}_i(\mathbf{w}_{r,h}^i))\left(\nabla\mathcal{L}_i(\mathbf{u}_{r,h}^i) -\nabla\mathcal{L}_i(\mathbf{v}_i(\mathbf{w}_{r,h}^i)) \right)\|^2\right] \\
		&\quad\quad\quad  + \mathbb{E}\left[\|(I-\alpha \nabla^2\mathcal{L}_i(\mathbf{w}_{r,h}^i))\nabla\mathcal{L}_i(\mathbf{v}_i(\mathbf{w}_{r,h}^i))\|^2\right]\\
		& \leq (1 + \alpha L)^2L^2 \mathbb{E}\left[\|\mathbf{u}_{r,h}^i - \mathbf{v}_i(\mathbf{w}_{r,h}^i)\|^2\right] + \mathbb{E}\left[\|\nabla F_i(\mathbf{w}_{r,h}^i)\|^2\right]\\
		& \leq (1 + \alpha L)^2L^2 \mathbb{E}\left[\|\mathbf{u}_{r,h}^i - \mathbf{v}_i(\mathbf{w}_{r,h}^i)\|^2\right] + 2L_F^2\mathbb{E}\left[\| \mathbf{w}_{r,h}^i - \mathbf{w}_r\|^2\right] + 2\mathbb{E}\left[\|\nabla F_i(\mathbf{w}_r)\|^2\right].
	\end{align*}
	Thus, we have the following equation if $\eta \leq \frac{1}{2H L_F}$
	\begin{align*}
		& \frac{1}{n}\sum_{i=1}^n \mathbb{E}\left[\Norm{\w_{r,h+1}^i -\w_r}^2\right]\\
		& \leq \left(1+\frac{2}{H}\right)\frac{1}{n}\sum_{i=1}^n\mathbb{E}\left[\|\mathbf{w}_{r,h}^i -\mathbf{w}_r\|^2\right] + 2\eta^2H\hat{\sigma}^2\\
		& \quad\quad + 2\eta^2H\left( \frac{(1 + \alpha L)^2L^2}{n}\sum_{i=1}^n\mathbb{E}\left[\|\mathbf{u}_{r,h}^i - \mathbf{v}_i(\mathbf{w}_{r,h}^i)\|^2\right] + \frac{2}{n}\sum_{i=1}^n \mathbb{E}\left[\|\nabla F_i(\mathbf{w}_r)\|^2\right]\right).
	\end{align*}
	Based on Lemma~\ref{lem:gamma_F}, we have
	\begin{align*}
		\frac{1}{n}\sum_{i=1}^n \mathbb{E}\left[\|\nabla F_i(\mathbf{w}_r)\|^2\right] & = \mathbb{E}\left[\|\nabla F(\mathbf{w}_r)\|^2\right] + \frac{1}{n}\sum_{i=1}^n \mathbb{E}\left[\|\nabla F_i(\mathbf{w}_r) - \nabla F(\mathbf{w}_r)\|^2\right] \\
		& \leq  \mathbb{E}\left[\|\nabla F(\mathbf{w}_r)\|^2\right]  + \gamma_F^2.
	\end{align*}
	Given that $\mathbf{w}_{r,1}^i = \mathbf{w}_r$ and $\alpha \leq 1/L$, and $H\geq 2$, we can obtain that
	\begin{align*}
		&	\frac{1}{nH}\sum_{i=1}^n\sum_{h=1}^H\mathbb{E}\left[\|\mathbf{w}_{r,h}^i - \mathbf{w}_r\|^2\right]\\
		& \leq 16\eta^2 H (H-1) \left(\hat{\sigma}^2 + 2\gamma_F^2\right) + 32\eta^2H^2L^2 \left(\frac{1}{nH}\sum_{i=1}^n\sum_{h=1}^H\mathbb{E}\left[\|\mathbf{u}_{r,h}^i - \mathbf{v}_i(\mathbf{w}_{r,h}^i)\|^2\right]\right) \\
		& \quad\quad\quad + 32\eta^2 H (H-1)  \mathbb{E}\left[\|\nabla F(\mathbf{w}_r)\|^2\right].
	\end{align*}
\end{proof}

\begin{lemma}\label{lem:fed_moml_per_mavg}
	If $\alpha \leq 1/L$ and $\eta \leq \min\left\{\frac{1}{4HL_F},\frac{\beta}{16L\mathbb{I}[\beta\in(0,1)]}\right\}$, we have
	\begin{align}\nonumber
		& \frac{1}{nHR}\sum_{i=1}^n\sum_{r=1}^R\sum_{h=1}^H\mathbb{E}\left[\|\mathbf{u}_{r,h}^i - \mathbf{v}_i(\mathbf{w}_{r,h}^i)\|^2\right] \\\nonumber
		& \leq \frac{2\sigma_G^2}{\beta \left(\mathbb{I}[B<n]H S_0 + \mathbb{I}[B=n]HR \right)} \mathbb{I}[\beta\in(0,1)] + \frac{96\eta^2(\hat{\sigma}^2 + 2\gamma_F^2)}{\beta^2} \mathbb{I}[\beta\in(0,1)]  + \frac{2\beta \alpha^2\sigma_G^2}{|\S_1^i|} \\\label{eq:fed_moml_per_mavg}
		& \quad\quad\quad +\frac{192\eta^2}{\beta^2R }\sum_{r=1}^R \mathbb{E}\left[\|\nabla F(\mathbf{w}_r)\|^2\right]\mathbb{I}[\beta\in(0,1)].
	\end{align}
\end{lemma}
\begin{proof}
	We define that $\vhat_{r,h}^i \coloneqq \w_{r,h}^i - \alpha \nablah_{\S_1^i} \L_i(\w_{r,h})$. 
	\begin{align*}
		& \E\left[\Norm{\u_{r,h+1}^i - \v_i(\w_{r,h+1}^i)}^2\right] \\
		&= 	\E\left[\Norm{\v_i(\w_{r,h+1}^i)- (1-\beta)\u_{r,h}^i - \beta\vhat_{r,h+1}^i}^2\right]\\
		& = \E\left[\Norm{(1-\beta) \left(\v_i(\w_{r,h}^i) - \u_{r,h}^i + \v_i(\w_{r,h+1}^i) -\v_i(\w_{r,h}^i)\right) + \beta(\v_i(\w_{r,h+1}^i) -\vhat_{r,h+1}^i )}^2\right]
	\end{align*}
	Since $\E\left[\v_i(\w_{r,h+1}^i) -\vhat_{r,h+1}^i \right] = 0$ and $\E\left[\Norm{\v_i(\w_{r,h+1}^i) -\vhat_{r,h+1}^i }^2\right]\leq \frac{\alpha^2 \sigma_G^2}{|\S_1^i|}$, we have
	\begin{align*}
		&\E\left[\Norm{\u_{r,h+1}^i - \v_i(\w_{r,h+1}^i)}^2\right] \\
		& \leq (1-\beta)^2\E\left[\Norm{\v_i(\w_{r,h}^i) - \u_{r,h}^i + \v_i(\w_{r,h+1}^i) -\v_i(\w_{r,h}^i)}^2\right] + \beta^2 \frac{\alpha^2 \sigma_G^2}{|\S_1^i|}.\\
		& \leq (1-\beta)^2(1+\beta) \E\left[\Norm{\v_i(\w_{r,h}^i) - \u_{r,h}^i}^2\right] \\
		& \quad\quad + 4(1-\beta)^2(1+1/\beta)(1+\alpha L)^2\E\left[\Norm{\w_{r,h+1}^i - \w_{r,h}^i}^2\right]\mathbb{I}[\beta\in(0,1)]  + \beta^2 \frac{\alpha^2 \sigma_G^2}{|\S_1^i|}\\
		& \leq (1-\beta)\E\left[\Norm{\v_i(\w_{r,h}^i) - \u_{r,h}^i}^2\right] + \frac{4\eta^2(1+\alpha L)^2}{\beta}\E\left[\Norm{\Deltah_{r,h}^i}^2\right]\mathbb{I}[\beta\in(0,1)]  + \beta^2 \frac{\alpha^2 \sigma_G^2}{|\S_1^i|}.
	\end{align*}
	As done in the proof of Lemma~\ref{lem:local_drift}, the term $\E\left[\Norm{\Deltah_{r,h}^i}^2\right]$ can be upper bounded as
	\begin{align*}
		& \mathbb{E}\left[\|\widehat{\Delta}_{r,h}^i\|^2\right]  \\
		& \leq \hat{\sigma}^2 + (1+\alpha L)^2 L^2 \mathbb{E}\left[\|\mathbf{u}_{r,h}^i - \mathbf{v}_i(\mathbf{w}_{r,h}^i) \|^2\right] + 2 L_F^2\mathbb{E}\left[\|\mathbf{w}_{r,h}^i - \mathbf{w}_r\|^2\right] + 2\mathbb{E}\left[\|\nabla F_i(\mathbf{w}_r)\|^2\right].
	\end{align*}
	If $\alpha \leq 1/L$, we have
	\begin{align*}
		& \frac{1}{n}\sum_{i=1}^n\mathbb{E}\left[\|\mathbf{u}_{r,h+1}^i - \mathbf{v}_i(\mathbf{w}_{r,h+1}^i)\|^2\right]\\
		& \leq \left(1-\beta+\frac{64\eta^2L^2}{\beta}\mathbb{I}[\beta\in(0,1)] \right) \frac{1}{n}\sum_{i=1}^n\mathbb{E}\left[\|\mathbf{u}_{r,h}^i - \mathbf{v}_i(\mathbf{w}_{r,h}^i)\|^2\right]  + \beta^2\frac{\alpha^2\sigma_G^2}{|\S_1^i|} + \frac{16\eta^2}{\beta}\hat{\sigma}^2\mathbb{I}[\beta\in(0,1)] \\
		&\quad\quad + \frac{32 L_F^2\eta^2}{\beta}\frac{1}{n}\sum_{i=1}^n\mathbb{E}\left[\|\mathbf{w}_{r,h}^i - \mathbf{w}_r\|^2\right]\mathbb{I}[\beta\in(0,1)]  + \frac{32\eta^2}{\beta}\left(\mathbb{E}\left[\|\nabla F(\mathbf{w}_r)\|^2\right]+\gamma_F^2\right) \mathbb{I}[\beta\in(0,1)].
	\end{align*}
	Telescope it from $h = 1$ to $H$ and divide both sides by $H$. 
	\begin{small}
		\begin{align*}
			& \frac{1}{nH}\sum_{i=1}^n\sum_{h=1}^H\mathbb{E}\left[\|\mathbf{u}_{r,h+1}^i - \mathbf{v}_i(\mathbf{w}_{r,h+1}^i)\|^2\right]\\
			& \leq 	\left(1-\beta+\frac{64\eta^2L^2}{\beta}\mathbb{I}[\beta\in(0,1)] \right) \frac{1}{nH}\sum_{i=1}^n\sum_{h=1}^H\mathbb{E}\left[\|\mathbf{u}_{r,h}^i - \mathbf{v}_i(\mathbf{w}_{r,h}^i)\|^2\right]+ \beta^2\frac{\alpha^2\sigma_G^2}{|\S_1^i|} + \frac{16\eta^2}{\beta}\hat{\sigma}^2\mathbb{I}[\beta\in(0,1)] \\
			&\quad\quad + \frac{32 L_F^2\eta^2}{\beta}\frac{1}{nH}\sum_{i=1}^n\sum_{h=1}^H\mathbb{E}\left[\|\mathbf{w}_{r,h}^i - \mathbf{w}_r\|^2\right]\mathbb{I}[\beta\in(0,1)]  + \frac{32\eta^2}{\beta}\left(\mathbb{E}\left[\|\nabla F(\mathbf{w}_r)\|^2\right]+\gamma_F^2\right) \mathbb{I}[\beta\in(0,1)].
		\end{align*}
	\end{small}
	Applying Lemma~\ref{lem:local_drift} and setting $\eta \leq \min\left\{\frac{1}{4HL_F},\frac{\beta}{16L\mathbb{I}[\beta\in(0,1)]}\right\}$ leads to
	\begin{align*}
		& \frac{1}{nH}\sum_{i=1}^n\sum_{h=1}^H\mathbb{E}\left[\|\mathbf{u}_{r,h+1}^i - \mathbf{v}_i(\mathbf{w}_{r,h+1}^i)\|^2\right]\\
		& \leq \left(1-\beta+\frac{64\eta^2L^2}{\beta}(1+16\eta^2H^2L_F^2)\mathbb{I}[\beta\in(0,1)] \right) \frac{1}{nH}\sum_{i=1}^n\sum_{h=1}^H\mathbb{E}\left[\|\mathbf{u}_{r,h}^i - \mathbf{v}_i(\mathbf{w}_{r,h}^i)\|^2\right]\\
		& \quad\quad\quad+ \frac{16\eta^2}{\beta}(1+32L_F^2\eta^2H(H-1))\left(\hat{\sigma}^2 + 2\gamma_F^2\right)\mathbb{I}[\beta\in(0,1)]+ \beta^2\frac{\alpha^2\sigma_G^2}{|\S_1^i|}\\
		& \quad\quad\quad+ \frac{32\eta^2}{\beta}\left(1+32L_F^2\eta^2 H(H-1)\right)\E\left[\Norm{\nabla F(\w_r)}^2\right]\mathbb{I}[\beta\in(0,1)]\\
		& \leq \left(1-\frac{\beta}{2}\right) \frac{1}{nH}\sum_{i=1}^n\sum_{h=1}^H\mathbb{E}\left[\|\mathbf{u}_{r,h}^i - \mathbf{v}_i(\mathbf{w}_{r,h}^i)\|^2\right]\mathbb{I}[\beta\in(0,1)]\\
		& \quad\quad\quad + \frac{48\eta^2\left(\hat{\sigma}^2 + 2\gamma_F^2\right)}{\beta}\mathbb{I}[\beta\in(0,1)] + \frac{\beta^2\alpha^2\sigma_G^2}{|\S_1^i|} + \frac{96\eta^2}{\beta}\E\left[\Norm{\nabla F(\w_r)}^2\right]\mathbb{I}[\beta\in(0,1)].\\
	\end{align*}
	Re-arrange the terms, telescope it from $r=1$ to $R$, and divide both sides by $\beta R/2$
	\begin{align*}
		& \frac{1}{nHR}\sum_{i=1}^n\sum_{r=1}^R\sum_{h=1}^H\mathbb{E}\left[\|\mathbf{u}_{r,h}^i - \mathbf{v}_i(\mathbf{w}_{r,h}^i)\|^2\right] \leq \frac{2\sigma_G^2}{\beta \left(\mathbb{I}[B<n]H S_0 + \mathbb{I}[B=n]HR \right)} \mathbb{I}[\beta\in(0,1)]  \\
		& \quad\quad\quad + \frac{96\eta^2(\hat{\sigma}^2 + 2\gamma_F^2)}{\beta^2} \mathbb{I}[\beta\in(0,1)]  + \frac{2\beta \alpha^2\sigma_G^2}{|\S_1^i|}  +\frac{192\eta^2}{\beta^2R }\sum_{r=1}^R \mathbb{E}\left[\|\nabla F(\mathbf{w}_r)\|^2\right]\mathbb{I}[\beta\in(0,1)].
	\end{align*}
\end{proof}

\begin{proof}[Proof of Theorem~\ref{thm:fed_moml}]
	Based on \eqref{eq:one_round} and \eqref{eq:local_drift}, we have
	\begin{small}
		\begin{align*}
			& \E\left[F(\w_{r+1})\right] \leq \E\left[F(\w_r)\right] - \frac{\tilde{\eta}}{2}\left(1-8\tilde{\eta} - \frac{8\tilde{\eta}(n-B)}{B(n-1)}\right)\E\left[\Norm{\nabla F(\w_r)}^2\right]\\
			& + \frac{\tilde{\eta}}{2}\left(C_{\rho,G,L} + 8\tilde{\eta}L_F^2 + \frac{8\tilde{\eta}(n-B)L_F^2}{B(n-1)}\right) \frac{1}{nH}\sum_{i=1}^{n}\sum_{h=1}^H\E\left[\Norm{\w_{r,h}^i - \w_r}^2\right]\\
			& + 8\tilde{\eta}(1+L^2\tilde{\eta})\frac{1}{nH}\sum_{i=1}^n\sum_{h=1}^H \E\left[\Norm{\u_{r,h}^i - \v_i(\w_{r,h}^i)}^2\right] + \frac{4(n-B)}{B(n-1)}\tilde{\eta}^2\gamma_F^2 + \frac{\tilde{\eta}^2\hat{\sigma}^2}{BH}\\
			& \leq \E\left[F(\w_r)\right]+ 8\tilde{\eta}\eta^2 H(H-1)\left(C_{\rho,G,L} + 8\tilde{\eta}L_F^2 + \frac{8\tilde{\eta}(n-B)L_F^2}{B(n-1)}\right)(\hat{\sigma}^2 + 2\gamma_F^2) + \frac{4(n-B)}{B(n-1)}\tilde{\eta}^2\gamma_F^2  \\
			& - \frac{\tilde{\eta}}{2}\left(1-8\tilde{\eta} - \frac{8\tilde{\eta}(n-B)}{B(n-1)}- 32\tilde{\eta}^2\left(C_{\rho,G,L}+ 8\tilde{\eta}L_F^2 + \frac{8\tilde{\eta}(n-B)L_F^2}{B(n-1)}\right)\right)\E\left[\Norm{\nabla F(\w_r)}^2\right]+ \frac{\tilde{\eta}^2\hat{\sigma}^2}{BH}\\
			& +8\tilde{\eta}\left(1+L^2\tilde{\eta} + 2\tilde{\eta}^2L^2\left(C_{\rho,G,L}+ 8\tilde{\eta}L_F^2 + \frac{8\tilde{\eta}(n-B)L_F^2}{B(n-1)}\right) \right)\frac{1}{nH}\sum_{i=1}^n\sum_{h=1}^H \E\left[\Norm{\u_{r,h}^i - \v_i(\w_{r,h}^i)}^2\right]\\
			& \leq \E\left[F(\w_r)\right] + 8\tilde{\eta}\eta^2H(H-1)C_1(\hat{\sigma}^2 + 2\gamma_F^2) + \frac{4(n-B)}{B(n-1)}\tilde{\eta}^2\gamma_F^2 + \frac{\tilde{\eta}^2\hat{\sigma}^2}{BH} - \frac{\tilde{\eta}}{2}\left(1-8C_2\tilde{\eta}\right)\E\left[\Norm{\nabla F(\w_r)}^2\right]\\
			& + 8\tilde{\eta}C_3 \frac{1}{nH}\sum_{i=1}^n\sum_{h=1}^H \E\left[\Norm{\u_{r,h}^i - \v_i(\w_{r,h}^i)}^2\right].
		\end{align*}
	\end{small}
	where we set $\eta \leq \frac{1}{4HL_F}$ and define $C_1\coloneqq C_{\rho,G,L} + 2L_F + \frac{2(n-B)L_F}{B(n-1)}$, $C_2 \coloneqq 1 + \frac{n-B}{B(n-1)} + \frac{C_1}{4L_F^2}$, and $C_3 \coloneqq 1 + \frac{L^2}{4L_F} + \frac{L^2C_1}{8L_F^2}$. Telescoping the equation above from round $r=1$ to $R$ and dividing both sides by $R$ leads to
	\begin{align*}
		\frac{1}{R}\sum_{r=1}^R \E\left[F(\w_{r+1})\right] & \leq 	\frac{1}{R}\sum_{r=1}^R \E\left[F(\w_r)\right] + 8\tilde{\eta}\eta^2H(H-1)C_1(\hat{\sigma}^2 + 2\gamma_F^2)+ \frac{4(n-B)}{B(n-1)}\tilde{\eta}^2\gamma_F^2\\
		&  + \frac{\tilde{\eta}^2\hat{\sigma}^2}{BH} - \frac{\tilde{\eta}}{2}\left(1-8C_2\tilde{\eta}\right)\frac{1}{R}\sum_{r=1}^R\E\left[\Norm{\nabla F(\w_r)}^2\right] \\
		& + 8\tilde{\eta}C_3 \frac{1}{nHR}\sum_{i=1}^n\sum_{r=1}^R\sum_{h=1}^H \E\left[\Norm{\u_{r,h}^i - \v_i(\w_{r,h}^i)}^2\right].
	\end{align*}
	The last term above can be upper bounded by Lemma~\ref{lem:fed_moml_per_mavg}.
	\begin{align*}
		& \frac{1}{R}\sum_{r=1}^R \E\left[F(\w_{r+1})\right] \\
		& \leq 	\frac{1}{R}\sum_{r=1}^R \E\left[F(\w_r)\right] + 8\tilde{\eta}\eta^2H(H-1)C_1(\hat{\sigma}^2 + 2\gamma_F^2)+ \frac{4(n-B)}{B(n-1)}\tilde{\eta}^2\gamma_F^2\\
		&  + \frac{\tilde{\eta}^2\hat{\sigma}^2}{BH} - \frac{\tilde{\eta}}{2}\left(1-8C_2\tilde{\eta}- \frac{6144\eta^2C_3}{\beta^2}\mathbb{I}[\beta\in(0,1)]\right)\frac{1}{R}\sum_{r=1}^R\E\left[\Norm{\nabla F(\w_r)}^2\right] + \frac{16\tilde{\eta}C_3\beta \alpha^2\sigma_G^2}{|\S_1^i|}\\
		& + \frac{16\tilde{\eta}C_3\sigma_G^2}{\beta \left(\mathbb{I}[B<n]H S_0 + \mathbb{I}[B=n]HR \right)} \mathbb{I}[\beta\in(0,1)]  + \frac{768\tilde{\eta}C_3\eta^2(\hat{\sigma}^2 + 2\gamma_F^2)}{\beta^2} \mathbb{I}[\beta\in(0,1)].
	\end{align*}
	If we set $\tilde{\eta}=\eta H \leq \frac{1}{32C_2}$ and $\eta\leq \frac{\beta}{\sqrt{24576C_3\mathbb{I}[\beta\in(0,1)]}}$, we have
	\begin{align*}
		& \frac{1}{R}\sum_{r=1}^R \E\left[\Norm{\nabla F(\w_r)}^2\right] \\
		& \leq \frac{4F(\w_1)}{\eta T} + 32\eta^2 H(H-1)C_1(\hat{\sigma}^2 + 2\gamma_F^2) + \frac{4\eta}{B}\left(\hat{\sigma}^2 + \frac{4(n-B)}{(n-1)}H\gamma_F^2\right) + \frac{64C_3\beta\alpha^2\sigma_G^2}{|\S_1^i|}\\
		& + \frac{64C_3\sigma_G^2}{\beta \left(\mathbb{I}[B<n]H S_0 + \mathbb{I}[B=n]HR \right)} \mathbb{I}[\beta\in(0,1)] + \frac{3072C_3\eta^2(\hat{\sigma}^2 + 2\gamma_F^2)}{\beta^2} \mathbb{I}[\beta\in(0,1)],
	\end{align*}
	where we use the fact $\tilde{\eta}R = \eta HR = \eta T$. Then, we can define $C_4\coloneqq \frac{1}{4\max\{L_F,8C_2\}}$, $C_5\coloneqq \frac{1}{16\max\{L, \sqrt{96C_3}\}}$, and set $	\eta \leq \min\left\{\frac{C_4}{H},\frac{C_5\beta}{\mathbb{I}[\beta\in(0,1)]}\right\}$.
\end{proof}

\section{Additional Theoretical Results}

In this section, we provide some other theoretical results that we obtained. 

\subsection{MAML, BSGD, and BSpiderBoost in the finite \#tasks case}\label{sec:finite_ext}

In original papers of MAML~\citep{fallah2020convergence} and BSGD/BSpiderBoost~\citep{hu2020biased}, the convergence rates are only available for the case that the number of tasks $n$ is infinite (e.g., the tasks are online). However, it is easy to extend their results to the finite $n$ case. For example, Equation (97) of \cite{fallah2020convergence} uses the fact 
\begin{align*}
	\E\left[\frac{1}{B}\sum_{i\in\B}X_{\xi_i}\right]=X,\quad \E\left[\Norm{\frac{1}{B}\sum_{i\in\B} X_{\xi_i}}^2\right] \leq \frac{1}{B}\E_\xi\left[\Norm{X_\xi}^2\right] + \Norm{X}^2,
\end{align*}  
for $\E\left[X_{\xi_i}\right]=X$. This further leads to requirement on batch size $B=\O(1/\epsilon^2)$ in the final result to ensure the convergence. In the finite $n$ case, we instead use the finite $n$ counterpart in Lemma~\ref{lem:tau_nice}. Then, we can conclude that MAML needs $B=n$ to ensure convergence if we follow the rest part of the analysis of \cite{fallah2020convergence}. 

\subsection{Comparison of convergence rates in the infinite $n$ case}

As demonstrated in the main paper, our MOML only works when the number of tasks is finite. However, our LocalMOML works for both finite and infinite $n$ if it is implemented on a single machine. In Table~\ref{tab:comparision_online}, we compare it with existing results. 

\subsection{Problematic proof in \cite{fallah2020personalized} of Per-FedAvg}\label{sec:wrong_fed_maml}

The convergence analysis in \cite{fallah2020personalized} is problematic starting from Equation (99) in their paper (We follow their notations below):
\begin{align*}
	\E\left[\left(\frac{1}{\tau n}\sum_{i\in\A_k} \nabla F_i(\bar{w}_{k+1,t})\right)\mid \F_{k+1}^t\right] = \frac{1}{n}\sum_{i=1}^n\nabla F_i(\bar{w}_{k+1,t}).
\end{align*}
The equality above is wrong unless $|\A_k| = n$ (full client participation) because $\bar{w}_{k+1,t} = \frac{1}{\tau n}\sum_{i\in\A_k} w_{k+1,t}^i$ also depends on the randomly sampled batch of clients $\A_k$. This fault makes their proof cannot proceed. In this paper, we corrected this issue.

\section{Additional Experimental Results of Sinewave Regression}\label{sec:fitted_curves}

We also provide the fitted sinusoid curves on unseen tasks. The curves when $K=1$ on 5 unseen tasks can be found in Figure~\ref{fig:fitted_sinwave_1shots} and those of $K=3$ can be found in Figure~\ref{fig:fitted_sinwave_5shots}.

\begin{figure*}[ht]
	\subfigure[Task 1]{
		\centering
		\includegraphics[width=0.300\linewidth]{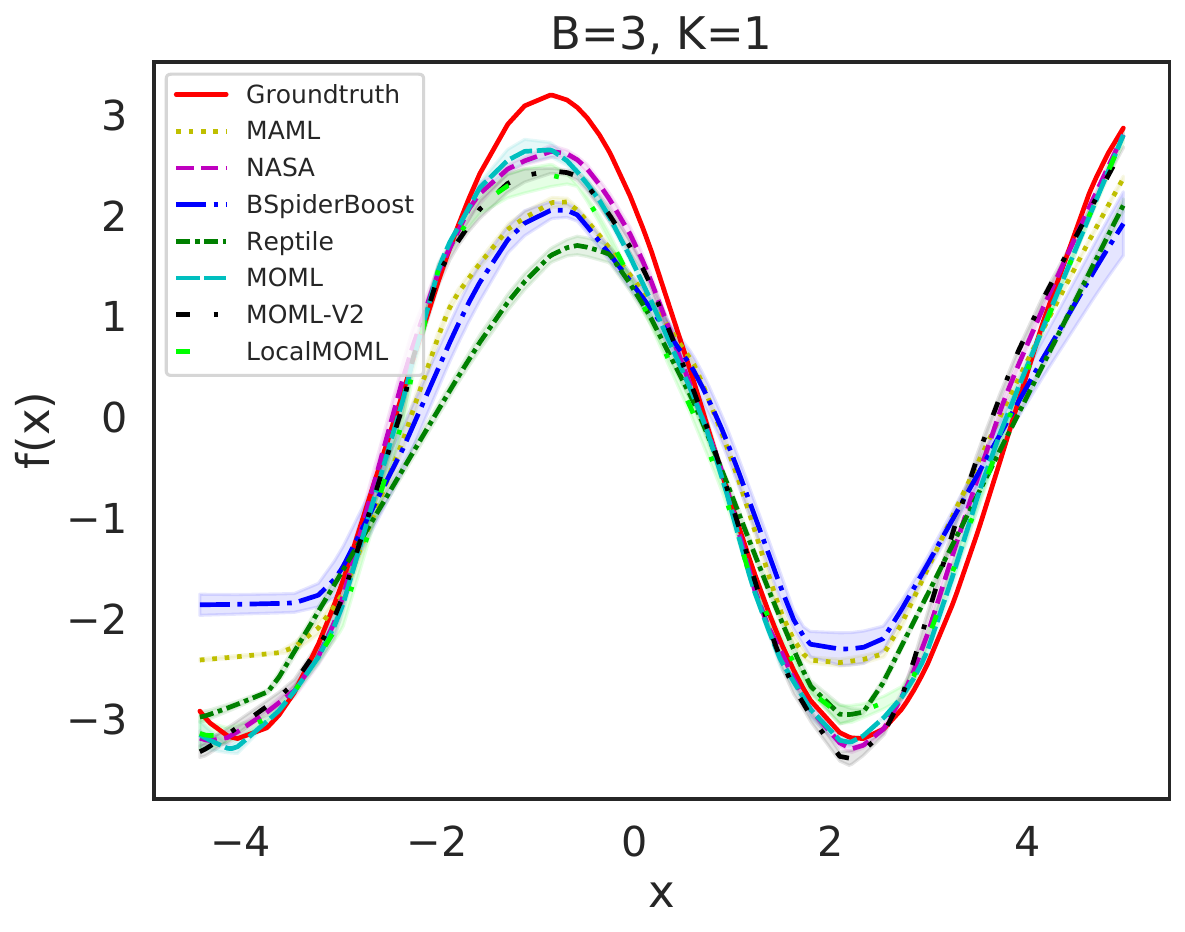}
	}
	\hfill
		\subfigure[Task 2]{
		\centering
		\includegraphics[width=0.300\linewidth]{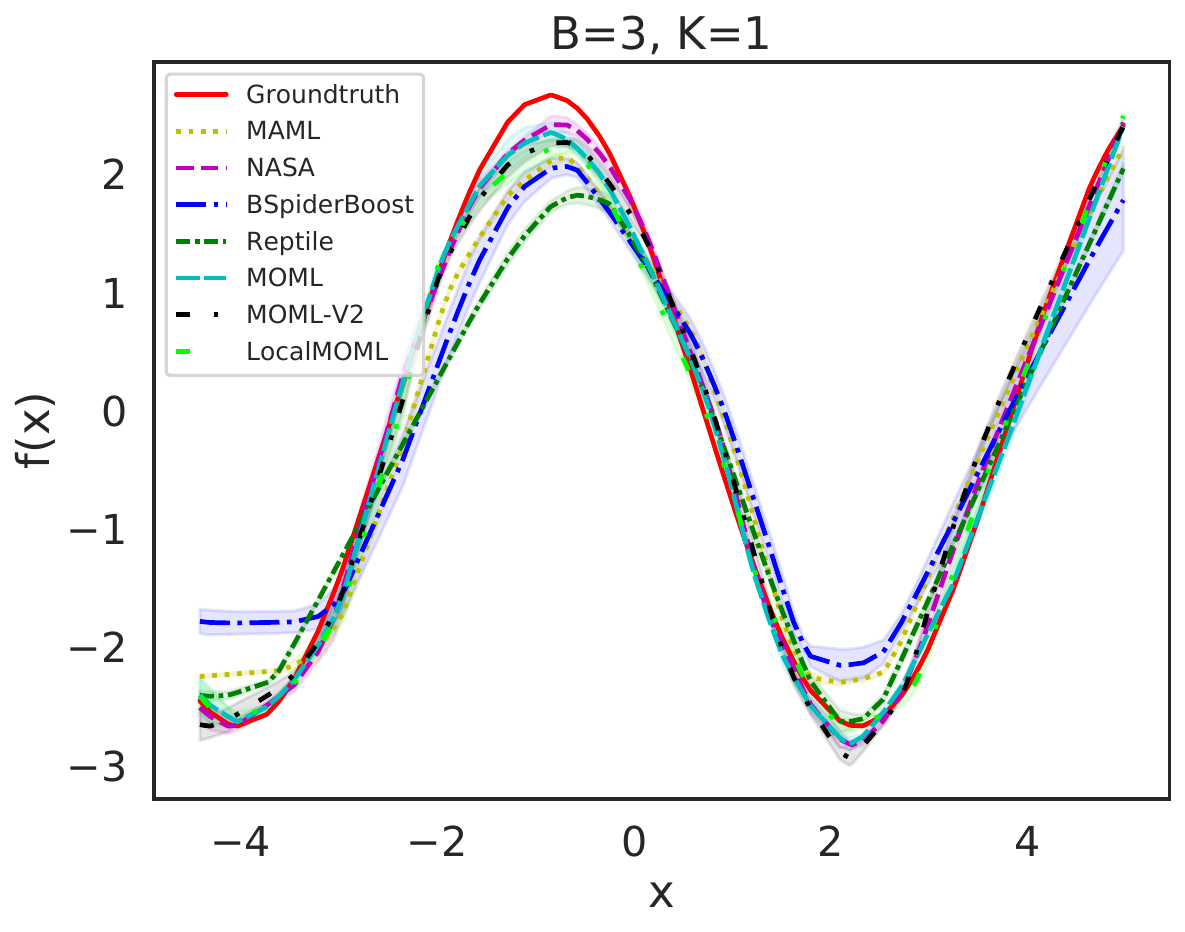}
	}
	\hfill
		\subfigure[Task 3]{
		\centering
		\includegraphics[width=0.300\linewidth]{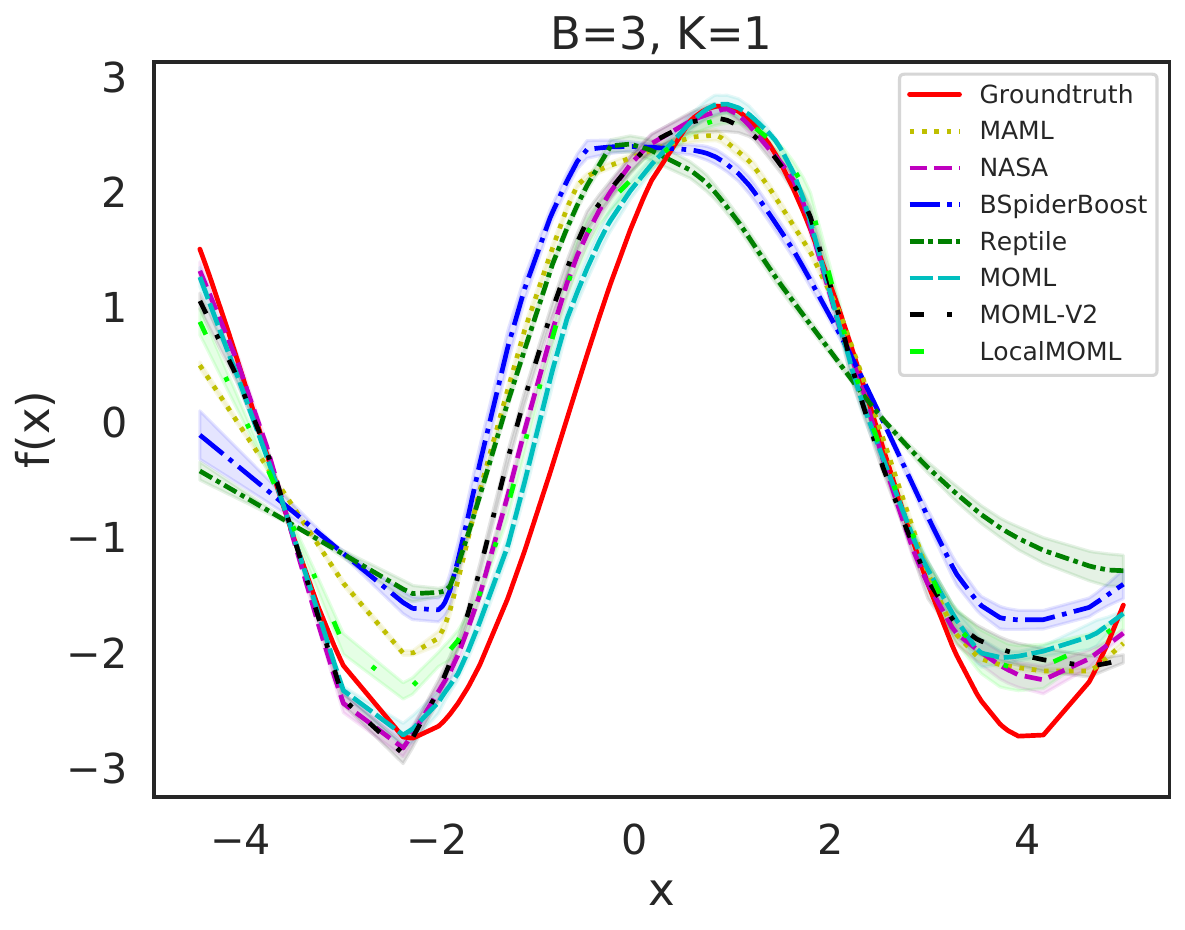}
	}
	\hfill
		\subfigure[Task 4]{
		\centering
		\includegraphics[width=0.300\linewidth]{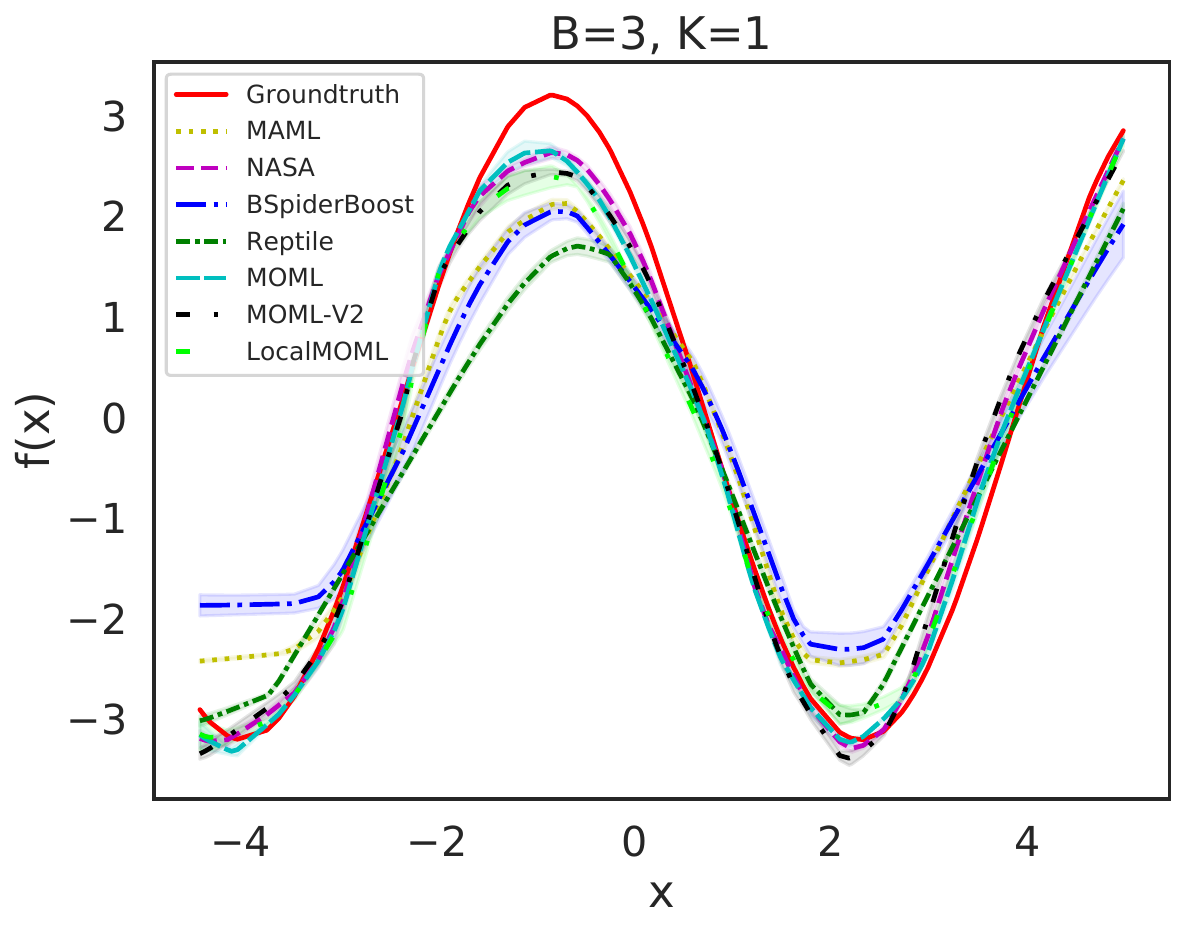}
	}
	\hfill
		\subfigure[Task 5]{
		\centering
		\includegraphics[width=0.300\linewidth]{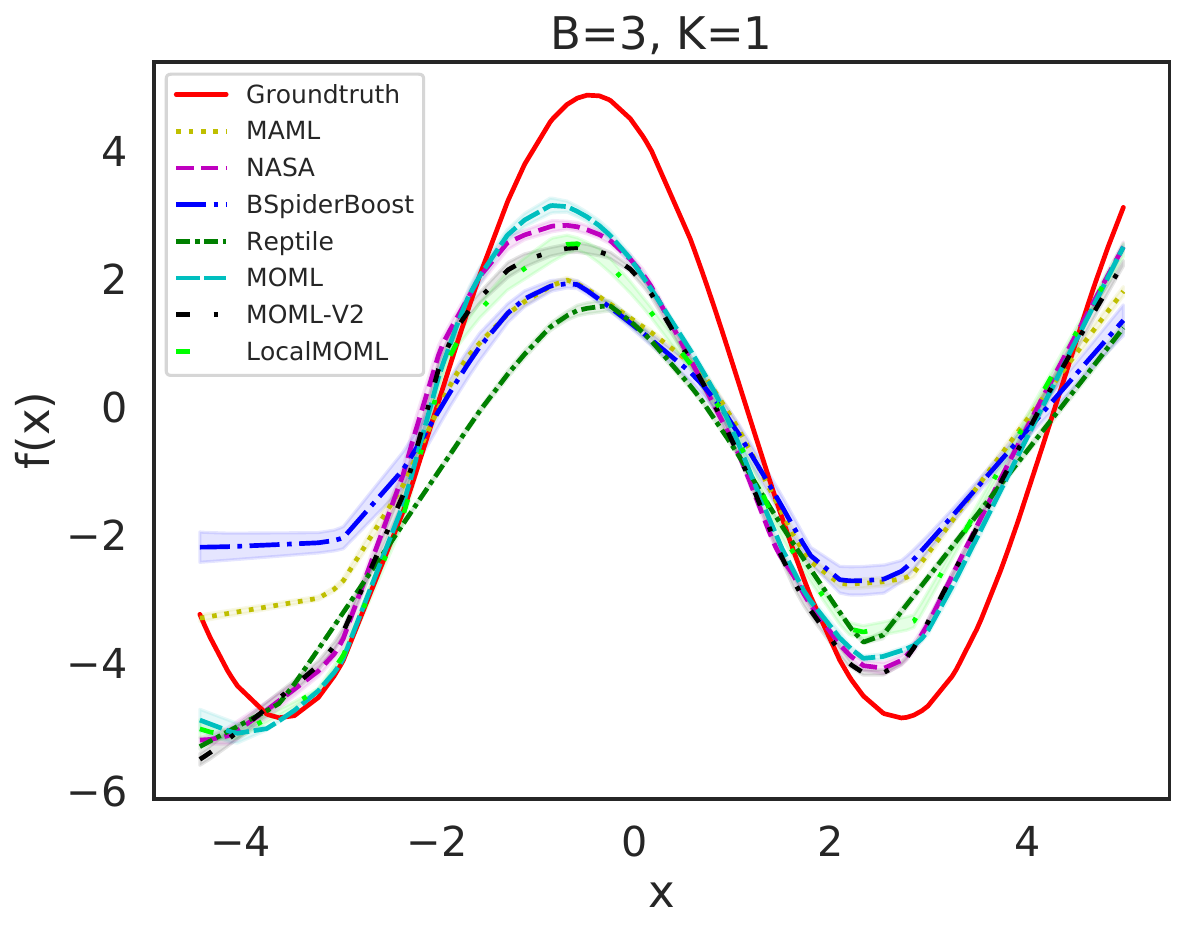}
	}
	\hfill
	\caption{Fitted Sinusoid Curves on five unseen tasks when $K=1$.}
\label{fig:fitted_sinwave_1shots}
\end{figure*}

\begin{figure*}[ht]
	\subfigure[Task 1]{
		\centering
		\includegraphics[width=0.300\linewidth]{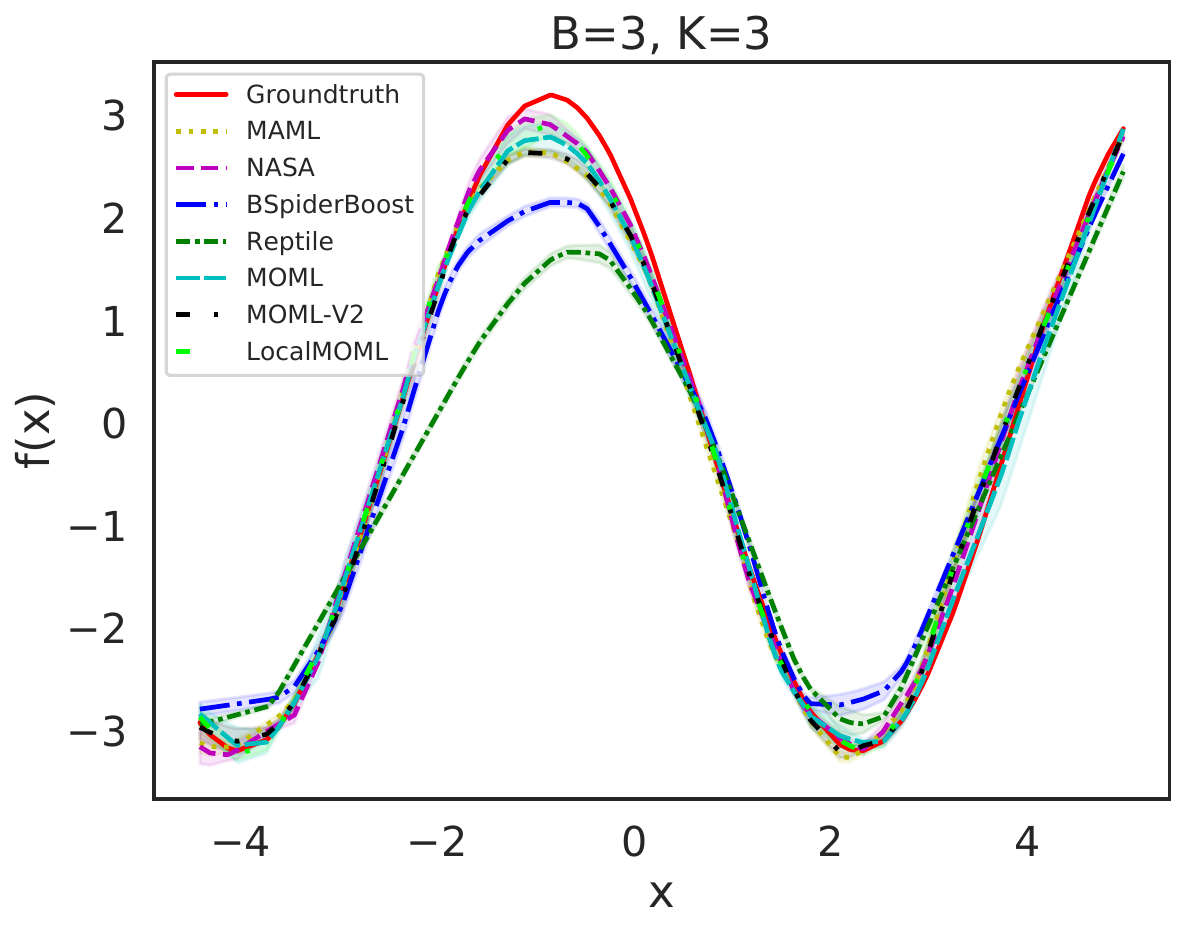}
	}
	\hfill
	\subfigure[Task 2]{
		\centering
		\includegraphics[width=0.300\linewidth]{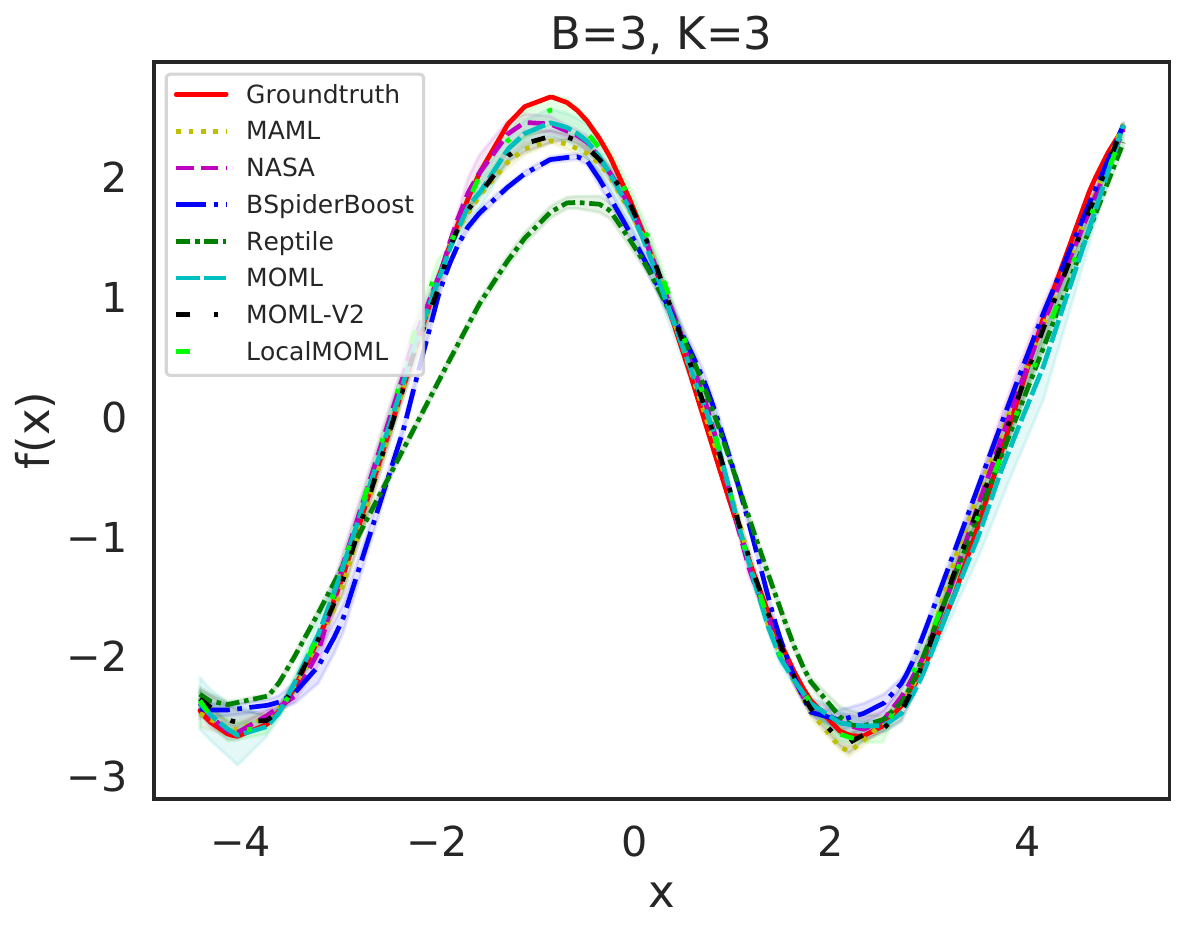}
	}
	\hfill
	\subfigure[Task 3]{
		\centering
		\includegraphics[width=0.300\linewidth]{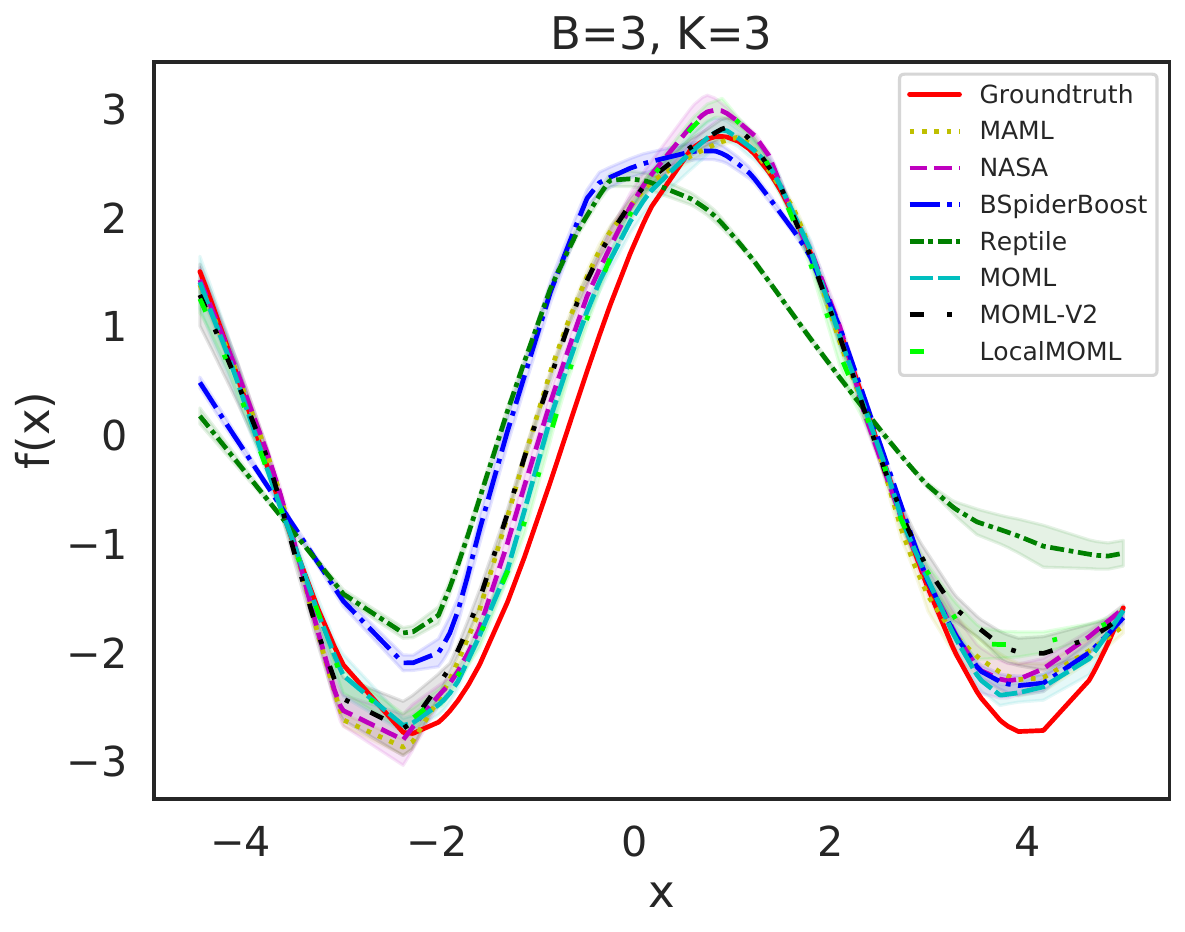}
	}
	\hfill
	\subfigure[Task 4]{
		\centering
		\includegraphics[width=0.300\linewidth]{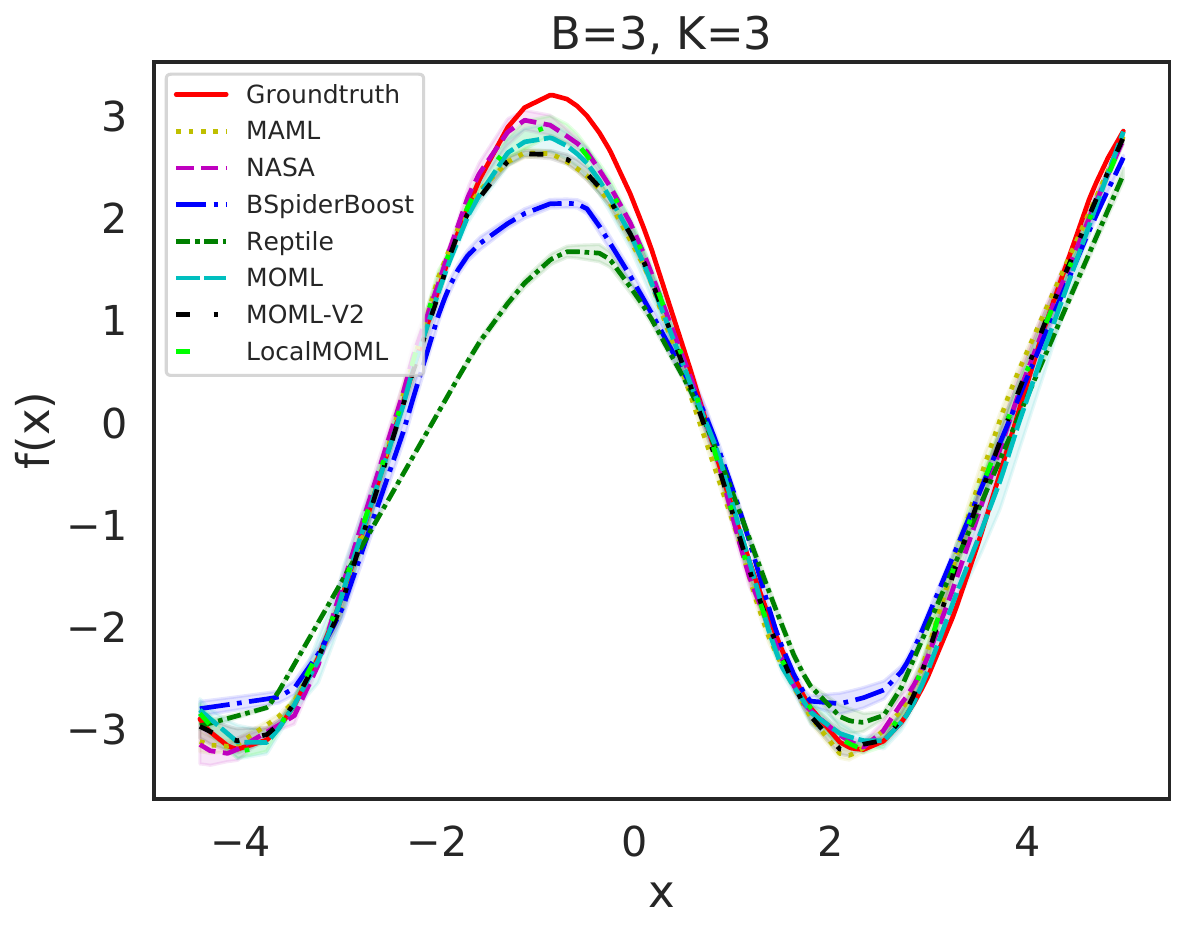}
	}
	\hfill
	\subfigure[Task 5]{
		\centering
		\includegraphics[width=0.300\linewidth]{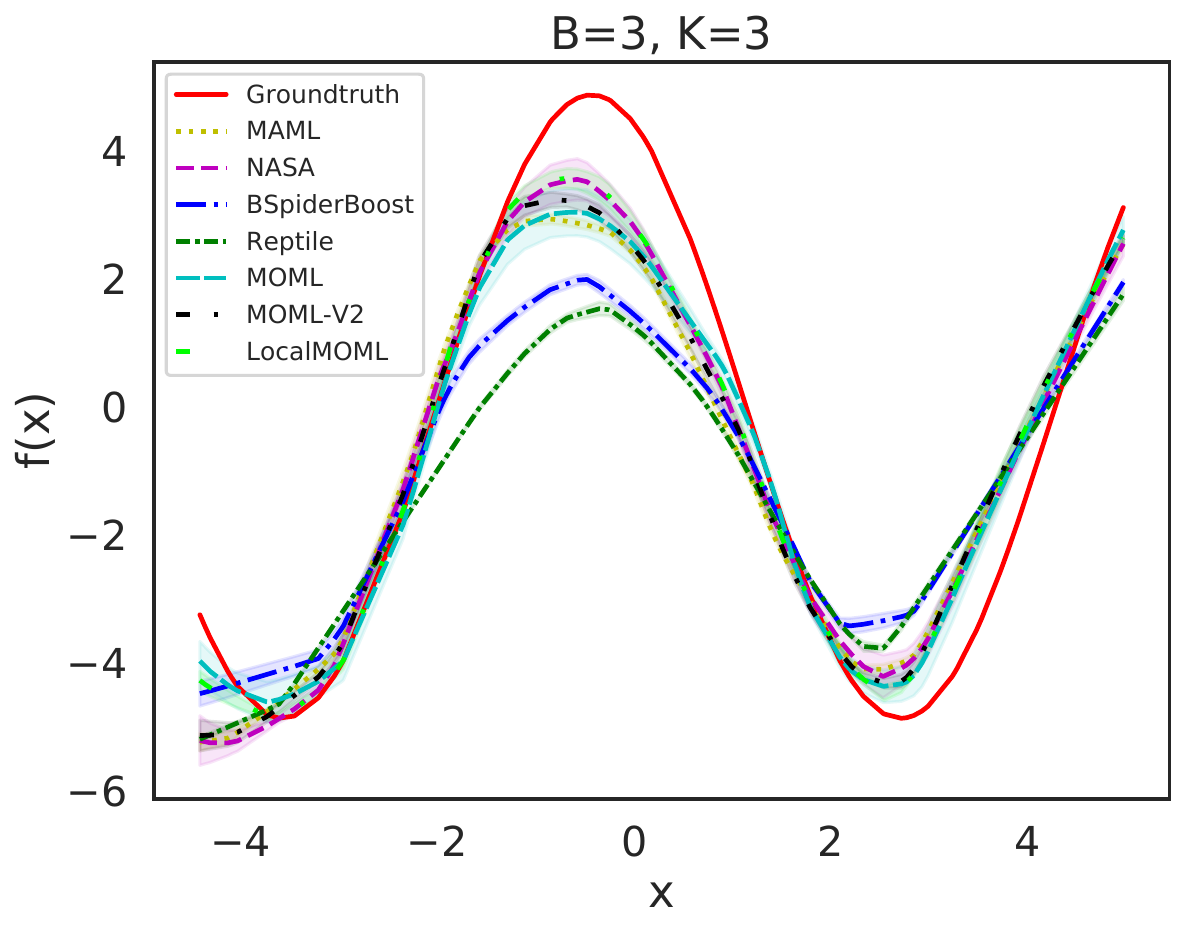}
	}
	\hfill
	\caption{Fitted Sinusoid Curves on five unseen tasks when $K=5$.}
	\label{fig:fitted_sinwave_5shots}
\end{figure*}

\end{document}